\onecolumn

\documentclass{article}

\usepackage{microtype}
\usepackage{graphicx}
\usepackage{subfigure}
\usepackage{booktabs} 

\usepackage{natbib}

\usepackage[utf8]{inputenc} 
\usepackage[T1]{fontenc}    
\usepackage[colorlinks,citecolor=blue,urlcolor=black,linkcolor=blue,pdfborder={0 0 0}]{hyperref}
\usepackage{url}            
\usepackage{booktabs}       
\usepackage{amsfonts}       
\usepackage{nicefrac}       
\usepackage{microtype}      
\usepackage{arydshln}

\usepackage{amsmath,amsbsy,amsfonts,amssymb,amsthm,physics,bm,verbatim,bbm,xspace}
\usepackage{mathtools}
\usepackage{graphicx}
\usepackage{algorithm, algorithmic}
\usepackage{color}
\usepackage{float}						
\usepackage{enumitem}
\usepackage{thmtools}
\usepackage{thm-restate}
\usepackage[capitalize]{cleveref}  
\crefname{nlem}{Lemma}{Lemmas}
\crefname{nprop}{Proposition}{Propositions}
\crefname{ncor}{Corollary}{Corollaries}
\crefname{nthm}{Theorem}{Theorems}
\crefname{exa}{Example}{Examples}
\crefname{assumption}{Assumption}{Assumptions}
\crefname{equation}{}{}

\usepackage{autonum} 
\usepackage{todonotes}

\renewcommand{\algorithmicrequire}{\textbf{Input:}}
\renewcommand{\algorithmicensure}{\textbf{Output:}}

\theoremstyle{plain}
\newtheorem{theorem}{Theorem}

\newtheorem{corollary}{Corollary}
\newtheorem{assumption}{Assumption}
\newtheorem{lemma}{Lemma}
\newtheorem{proposition}[theorem]{Proposition}

\theoremstyle{definition}
\newtheorem{condition}{Condition}
\newtheorem{definition}{Definition}

\newtheorem{remark}{Remark}
\allowdisplaybreaks

\DeclareMathOperator*{\argmin}{argmin}

\newcommand{\Perp}{\perp \! \! \! \perp}

\newcommand{\cond}{C_{\textrm{cond}}}
\newcommand{\Cmax}{C_{\max}}
\newcommand{\Cmin}{C_{\min}}

\newcommand{\mysec}[1]{Section~\ref{sec:#1}}
\newcommand{\myapp}[1]{Appendix~\ref{sec:#1}}
\newcommand{\eq}[1]{Eq.~(\ref{eq:#1})}
\newcommand{\myfig}[1]{Figure~\ref{fig:#1}}

\newcommand{\sG}{\text{sG}}
\newcommand{\sE}{\text{sE}}

\newcommand{\bX}{\mathbf{X}}

\newcommand{\bx}{\mathbf{x}}

\newcommand{\by}{\mathbf{y}}

\newcommand{\be}{\mathbf{e}}

\newcommand{\bbf}{\mathbf{f}}
\newcommand{\bbtf}{\hat{\mathbf{f}}}
\newcommand{\bg}{\mathbf{g}}
\newcommand{\btg}{\hat{\mathbf{g}}}
\newcommand{\bq}{\mathbf{q}}
\newcommand{\btq}{\hat{\mathbf{q}}}

\newcommand{\bz}{\mathbf{z}}

\newcommand{\bZ}{\mathbf{Z}}
\newcommand{\bw}{\mathbf{w}}

\newcommand{\bV}{\mathbf{V}}
\newcommand{\bv}{\mathbf{v}}

\newcommand{\bD}{\mathbf{D}}
\newcommand{\bI}{\mathbf{I}}

\newcommand{\bQ}{\mathbf{Q}}
\newcommand{\bu}{\mathbf{u}}
\newcommand{\bU}{\mathbf{U}}

\newcommand{\bR}{\mathbf{R}}

\newcommand{\bS}{\mathbf{S}}
\newcommand{\bepsilon}{\bm{\epsilon}}

\newcommand{\bbeta}{\bm{\beta}}
\newcommand{\hbbeta}{\hat{\bm{\beta}}}

\newcommand{\bSigma}{\bm{\Sigma}}
\newcommand{\bPi}{\bm{\Pi}_{\lambda}}
\newcommand{\hbSigma}{\hat{\bm{\Sigma}}_n}
\newcommand{\bOmega}{\bm{\Omega}}
\newcommand{\bDelta}{\bm{\Delta}}

\newcommand{\mE}{\mathcal{E}}
\newcommand{\mF}{\mathcal{F}}

\newcommand{\mI}{\mathcal{I}}

\newcommand{\mN}{\mathcal{N}}

\newcommand{\mR}{\mathbb{R}}
\newcommand{\mS}{\mathcal{S}}

\newcommand{\mJ}{\mathcal{J}}

\newcommand{\E}{\mathbb{E}}
\newcommand{\mone}{\mathbbm{1}}

\newcommand{\normt}[1]{\norm{#1}_2}

\newcommand{\normo}[1]{\norm{#1}_1}

\newcommand{\normi}[1]{\norm{#1}_{\infty}}

\newcommand{\diag}{\operatorname{diag}}

\newcommand{\iid}{\textrm{i.i.d.}\xspace}
\newcommand{\dist}{\sim}
\newcommand{\distiid}{\overset{\textrm{\tiny\iid}}{\dist}}
\newcommand{\distind}{\overset{\textrm{\tiny\textrm{indep}}}{\dist}}
\newcommand{\eqdist}{\stackrel{d}{=}}

\newcommand{\trimmednorm}[2]{\norm{#1}_{(#2)}}

\newcommand{\maxseig}[1]{\Lambda_{s}[#1]}

\newcommand{\grate}{r_g}
\newcommand{\gratet}{r_{\bg,2}}

\newcommand{\fratet}{r_{f,2}}
\newcommand{\bratet}{r_{\bbeta,2}}
\newcommand{\brateo}{r_{\bbeta,1}}

\newcommand{\xstar}{\bx_{\star}}
\newcommand{\ystar}{y_{\star}}
\newcommand{\vstar}{\mathbf{v}_{\star}}
\newcommand{\Pstar}{\mathcal{P}^{\star}}

\newcommand{\supp}{\text{supp}}

\newcommand{\hblambda}{\hat{\bbeta}_{R}(\lambda)}
\newcommand{\hblambdastar}{\hat{\bbeta}_{R}(\lambda_{*})}

\newcommand{\bLambda}{\mathbf{\Lambda}}
\newcommand{\snr}{\textsc{snr}}

\newcommand{\hblaslambda}{\hat{\bbeta}_{L}(\lambda)}

\newcommand{\yjm}{\hat{y}_{\textsc{jm}}}
\newcommand{\yom}{\hat{y}_{\textsc{om}}}
\newcommand{\hy}{\hat{y}}

\newcommand{\Xo}{\bX^{(1)}}
\newcommand{\yo}{\by^{(1)}}
\newcommand{\Xt}{\bX^{(2)}}
\newcommand{\yt}{\by^{(2)}}
\newcommand{\tone}{\mathbf{t}^{(1)}}
\newcommand{\zo}{\bz^{(1)}}

\newcommand{\qtext}[1]{\quad\text{#1}\quad} 
\def\balign#1\ealign{\begin{align}#1\end{align}}
\def\baligns#1\ealigns{\begin{align*}#1\end{align*}}
\def\balignat#1\ealign{\begin{alignat}#1\end{alignat}}
\def\balignats#1\ealigns{\begin{alignat*}#1\end{alignat*}}
\def\bitemize#1\eitemize{\begin{itemize}#1\end{itemize}}
\def\benumerate#1\eenumerate{\begin{enumerate}#1\end{enumerate}}

\newenvironment{talign*}
 {\csname align*\endcsname}
 {\endalign}
\newenvironment{talign}
 {\csname align\endcsname}
 {\endalign}

\def\balignst#1\ealignst{\begin{talign*}#1\end{talign*}}
\def\balignt#1\ealignt{\begin{talign}#1\end{talign}}%
\usepackage[nohyperref,accepted]{icml2020}

\icmltitlerunning{Single Point Transductive Prediction}

\begin{document}

\twocolumn[
\icmltitle{Single Point Transductive Prediction}



\icmlsetsymbol{equal}{*}

\begin{icmlauthorlist}
\icmlauthor{Nilesh Tripuraneni}{berkeley}
\icmlauthor{Lester Mackey}{msr}
\end{icmlauthorlist}

\icmlaffiliation{berkeley}{Department of EECS, University of California, Berkeley}
\icmlaffiliation{msr}{Microsoft Research, New England}

\icmlcorrespondingauthor{Nilesh Tripuraneni}{nilesh\_tripuraneni@berkeley.edu}

\icmlkeywords{Single point, Transduction, Bias, Linear prediction, Lasso, Ridge, Semiparametric, Orthogonal moments, Orthogonal machine learning, Double / debiased machine learning}

\vskip 0.3in
]



\printAffiliationsAndNotice{}  

\setcounter{footnote}{0}

\begin{abstract}
    Standard methods in supervised learning separate training and prediction: the model is fit independently of any test points it may encounter. However, can knowledge of the next test point $\mathbf{x}_{\star}$ be exploited to improve prediction accuracy? We address this question in the context of linear prediction, showing how  techniques from semi-parametric inference can be used transductively to combat regularization bias. We first lower bound the $\mathbf{x}_{\star}$ prediction error of ridge regression and the Lasso, showing that they must incur significant bias in certain test directions. We then provide non-asymptotic upper bounds on the $\mathbf{x}_{\star}$ prediction error of two transductive prediction rules. We conclude by showing the efficacy of our methods on both synthetic and real data, highlighting the improvements single point transductive
    prediction can provide in settings with distribution shift.
\end{abstract}
\section{Introduction}\label{sec:intro}
We consider the task of prediction given independent datapoints $((y_i, \bx_i))_{i=1}^n$ from a linear model,
\begin{align}
    y_i = \bx_i^\top \bbeta_0 + \epsilon_i, \quad \E[\epsilon_i] = 0, \quad \epsilon_i \Perp \bx_i \label{eq:model}
\end{align}
in which our observed targets $\by = (y_1,\dots, y_n) \in \mR^n$ and covariates $\bX = [\bx_1, \dots, \bx_n]^\top \in \mR^{n \times p}$ are related by an unobserved parameter vector $\bbeta_0 \in \mR^p$ and noise vector $\bepsilon = (\epsilon_1, \dots, \epsilon_n) \in \mR^n$.

Most approaches to linear model prediction are \emph{inductive},  divorcing the steps of training and prediction; for example, regularized least squares methods like ridge regression \cite{hoerl1970ridge} and the Lasso \cite{tibshirani1996regression} are fit independently of any knowledge of the next target test point $\xstar$. This suggests a tantalizing \emph{transductive} question:
\textbf{can knowledge of a single test point $\bx_\star$ be leveraged to improve prediction for $\bx_\star$?}
In the random design linear model setting \cref{eq:model}, we answer this question in the affirmative.

Specifically, in \mysec{lb} we establish \textit{out-of-sample} prediction lower bounds for the popular ridge and Lasso estimators, highlighting the significant dimension-dependent bias introduced by regularization. In \mysec{ub} we demonstrate how this bias can be mitigated by presenting two classes of transductive estimators that exploit explicit knowledge of the test point $\xstar$. We provide non-asymptotic risk bounds for these estimators in the random design setting, proving that they achieve dimension-free $O(\frac{1}{n})$ $\bx_\star$-prediction risk for $n$ sufficiently large.
In \cref{sec:experiments}, we first validate our theory in simulation, demonstrating that transduction improves the prediction accuracy of the Lasso with fixed regularization even when $\bx_\star$ is drawn from the training distribution.
We then demonstrate that under distribution shift, our transductive methods outperform even the popular cross-validated Lasso, cross-validated ridge, and cross-validated elastic net estimators (which attempt to find an optimal data-dependent trade-off between bias and variance) on both synthetic data and a suite of five real datasets.
\vspace{-.3cm}
\subsection{Related Work}

Our work is inspired by two approaches to semiparametric inference: the debiased Lasso approach introduced by \citep{zhang2014confidence, van2014asymptotically, javanmard2014confidence} and the orthogonal machine learning approach of \citet{chernozhukov2017double}.
The works \citep{zhang2014confidence, van2014asymptotically, javanmard2014confidence}
obtain small-width and asympotically-valid confidence intervals (CIs) for individual model parameters $(\bbeta_{0})_j = \langle{\bbeta_0,\be_j\rangle}$ by debiasing an initial Lasso estimator \cite{tibshirani1996regression}.
The works \citep{chao2014high,cai2017confidence,athey2018approximate} each consider a more closely related problem of obtaining prediction confidence intervals using a generalization of the debiased Lasso estimator of \citet{javanmard2014confidence}.
The work of \citet{chernozhukov2017double} describes a general-purpose procedure for extracting $\sqrt{n}$-consistent and asymptotically normal target parameter estimates in the presence of nuisance parameters.
Specifically, \citet{chernozhukov2017double} construct a two-stage estimator where one initially fits first-stage estimates of nuisance parameters using arbitrary ML estimators on a first-stage data sample. In the second-stage, these first-stage estimators are used to provide estimates of the relevant model parameters using an orthogonalized method-of-moments. \citet{wager2016high} also uses generic ML procedures as regression adjustments to form efficient confidence intervals (CIs) for treatment effects. 

These pioneering works all focus on improved CI construction.
Here we show that the semiparametric techniques developed for hypothesis testing can be adapted to provide practical improvements in mean-squared prediction error.  Our resulting mean-squared error bounds complement the in-probability bounds of the aforementioned literature by controlling prediction performance across all events.

While past work on transductive regression has demonstrated both empirical and theoretical benefits over induction when many unlabeled test points are simultaneously available \cite{belkin2006manifold,alquier2012transductive,bellec2018prediction,chapelle2000transductive,cortes2007transductive,cortes2008stability}, none of these works have demonstrated a significant benefit, either empirical or theoretical, from transduction given access to only a single test point. For example, the works \cite{belkin2006manifold,chapelle2000transductive}, while theoretically motivated, provide no formal guarantees on transductive predictive performance and only show empirical benefits for large unlabeled test sets. The transductive Lasso analyses of \citet{alquier2012transductive,bellec2018prediction} provide prediction error bounds identical to those of the inductive Lasso, where only the restricted-eigenvalue constant is potentially improved by transduction. Neither analysis improves the dimension dependence of Lasso prediction in the SP setting to provide $O(1/n)$ rates. The formal analysis of  \citet{cortes2007transductive,cortes2008stability} only guarantees small error when the number of unlabeled test points is large. Our aim is to develop single point transductive prediction procedures that improve upon the standard inductive approaches both in theory and in practice.

Our approach also bears some resemblance to semi-supervised learning (SSL) -- improving the predictive power of an inductive learner by observing additional unlabelled examples \citep[see, e.g.,][]{zhu2005semi,bellec2018prediction}. Conventionally, SSL benefits from access to a large pool of unlabeled points drawn from the same distribution as the training data.
In contrast, our procedures receive access to only a single arbitrary test point $\xstar$ (we make no assumption about its distribution), and our aim is accurate prediction for that point.
We are unaware of SSL results that benefit significantly from access to single unlabeled point $\xstar$. 
\subsection{Problem Setup}
Our principal aim in this work is to understand the \emph{$\xstar$ prediction risk},
\begin{align}
    \hspace{-.2mm}\mathcal{R}(\xstar, \hy) = \E[(\ystar-\hy)^2] - \sigma_{\epsilon}^2 =  \E[(\hy-\langle \xstar, \bbeta_0 \rangle)^2]
    \label{eq:prediction-risk}
\end{align}
of an estimator $\hat{y}$ of the unobserved test response $\ystar = \xstar^\top \bbeta_0+ \epsilon_\star$. 
Here, $\epsilon_\star$ is independent of $\xstar$ with variance $\sigma_{\epsilon}^2$.
We exclude the additive noise $\sigma^2_{\epsilon}$ from our risk definition, as it is irreducible for any estimator. 
Importantly, to accommodate non-stationary learning settings, we consider $\xstar$ to be fixed and arbitrary; in particular,  $\xstar$ need not be drawn from the training distribution. 
Hereafter, we will make use of several assumptions which are standard in the random design linear regression literature.
\begin{assumption}[Well-specified Model]
  The data $(\bX, \by)$ is generated from the model \cref{eq:model}.
  \label{assump:well_spec}
\end{assumption}

\begin{assumption}[Bounded Covariance]
The covariate vectors have common covariance $\bSigma = \E[\bx_i\bx_i^\top]$ with $\bSigma_{ii} \leq 1/2$, $\sigma_{\max}(\bSigma) \leq C_{\max}$ and $\sigma_{\min}(\bSigma) \geq C_{\min}$. We further define the precision matrix $\bOmega = \bSigma^{-1}$ and condition number $\cond = \Cmax/\Cmin$. \label{assump:cov}
\end{assumption}

\begin{assumption}[Sub-Gaussian Design]
  Each covariate vector $\bSigma^{-1/2} \bx_i$ is sub-Gaussian with parameter $\kappa \geq 1$, in the sense that, $
      \E[\exp(\bv^\top \bx_i)] \leq \exp \left( {\kappa^2 \Vert \bSigma^{1/2} \bv \Vert^2}{/2} \right)
  $
  \label{assump:design}.
\end{assumption}

\begin{assumption}[Sub-Gaussian Noise]
  The noise $\epsilon_i$ is sub-Gaussian with variance parameter $\sigma_{\epsilon}^2$.
  \label{assump:noise1}
\end{assumption}
Throughout, we use bold lower-case letters (e.g., $\bx$) to refer to vectors and bold upper-case letters to refer to matrices (e.g., $\bX$). We define $[p] = \{1, \hdots, p \}$ and $p \vee n = \max(p,n)$. Vectors or matrices subscripted with an index set $S$ indicate the subvector or submatrix supported on $S$. The expression $s_{\bbeta_0}$ indicates the number of non-zero elements in $\bbeta_0$,  $\supp(\bbeta_0) = \{ j : (\bbeta_0)_j \neq 0 \}$ and $\mathbb{B}_0(s)$ refers to the set of $s$-sparse vectors in $\mR^p$.
We use $\gtrsim$, $\lesssim$, and $\asymp$ to denote greater than, less than, and equal to up to a constant that is independent of $p$ and $n$. 

\section{Lower Bounds for Regularized Prediction}\label{sec:lb}

We begin by providing lower bounds on the $\bx_\star$ prediction risk of 
Lasso and ridge regression; the corresponding predictions take the form $\hat{y} = \langle \bx_\star, \hat{\bbeta}\rangle$ for a regularized  estimate $\hat{\bbeta}$ of the unknown vector $\bbeta_0$.

\subsection{Lower Bounds for Ridge Regression Prediction}\label{sec:lb_ridge}
We first consider the $\bx_\star$ prediction risk of the ridge estimator $\hblambda \triangleq \argmin_{\bbeta} \normt{\by-\bX \bbeta}^2 + \lambda \normt{\bbeta}^2$ with regularization parameter $\lambda > 0$. In the asymptotic high-dimensional limit (with $n, p \to \infty$) and assuming the training distribution equals the test distribution, \citet{dobriban2018high} compute the predictive risk of the ridge estimator in
 a dense random effects model. By contrast, we provide a non-asymptotic lower bound which does not impose any distributional assumptions on $\xstar$ or on the underlying parameter vector $\bbeta_0$. \cref{thm:ridge_lb}, proved in \cref{sec:app_ridge_lb}, isolates the error in the ridge estimator due to bias for any choice of regularizer $\lambda$.
\begin{theorem}
\label{thm:ridge_lb}
    Under \cref{assump:well_spec},
    suppose $\bx_i \distiid \mN(0, \bI_p)$ with independent noise $\bepsilon \sim \mN(0, \bI_n \sigma_{\epsilon}^2)$. If $n \geq p \geq 20$, 
    
\vspace{-2em}
    \begin{talign}
        \E&[\langle \xstar, \hblambda - \bbeta_0 \rangle^2] \geq  \\ &\frac{\normt{\bbeta_0}^2}{\sigma_{\epsilon}^2}  \cdot \frac{n}{4}  \left(\frac{\lambda/n}{\lambda/n+7} \right)^2 \cdot \normt{\xstar}^2 \cdot \frac{ \sigma_{\epsilon}^2}{n} \cdot \cos(\xstar, \bbeta_0)^2.
    \end{talign}
\vspace{-2em}
\end{theorem}
Notably, the dimension-free term $\normt{\xstar}^2 \cdot \frac{\sigma_{\epsilon}^2}{n}$ in this bound coincides with the $\bx_\star$ risk of the ordinary least squares (OLS) estimator in this setting. The remaining multiplicative factor indicates that the ridge risk can be substantially larger if the regularization strength $\lambda$ is too large.
In fact, our next result shows that, surprisingly, over-regularization can result even when $\lambda$ is tuned to minimize  held-out prediction error over the training population.
The same undesirable outcome results when $\lambda$ is selected to minimize $\ell_2$ estimation error; the proof can be found in \cref{sec:app_ridge_opt_lb}.
\begin{corollary}
    \label{cor:ridge_opt_lb}
    Under the conditions of \cref{thm:ridge_lb}, if $\tilde{\bx} \eqdist \bx_1$ and $\tilde{\bx}$ is independent of  $(\bX, \by)$, then for $\snr \triangleq \normt{\bbeta_0}^2/\sigma_{\epsilon}^2$,
    \begin{talign}
    &\lambda_* \triangleq \argmin_{\lambda} \E[\langle \tilde{\bx}, \hblambda - \bbeta_0 \rangle^2] = \\
    & \argmin_{\lambda} \E[\Vert \hblambda-\bbeta_0 \Vert_2^2] 
    = \frac{p}{\snr}, \ \text{and, for $n \geq \frac{1}{6} \frac{p}{\snr}$,}\\
        &\E[\langle \xstar, \hblambdastar - \bbeta_0 \rangle^2] \geq \frac{p^2}{n \snr} \cdot \normt{\xstar}^2 \cdot\frac{\sigma_{\epsilon}^2}{n} \cdot\frac{\cos(\xstar, \bbeta_0)^2}{784}. 
    \end{talign}
    \vspace{-2em}
\end{corollary}
Several insights can be gathered from the previous results.
First, the expression $\E[\langle \tilde{\bx}, \hblambda - \bbeta_0 \rangle^2]$ minimized in \cref{cor:ridge_opt_lb} is the expected prediction risk $\E [(\tilde{\by}-\tilde{\bx}^\top \hblambda )^2]- \sigma_\epsilon^2$ for a new datapoint $(\tilde{\bx},\tilde{\by})$ drawn from the training distribution. This is the population  analog of held-out validation error or cross-validation error that is often minimized to select $\lambda$ in practice. Second, in the setting of \cref{cor:ridge_opt_lb}, 
taking $\snr=\frac{1}{6} \frac{p}{n}$ yields
\begin{talign}
\E[\langle \xstar, \hblambdastar - \bbeta_0 \rangle^2] \geq p \cdot \normt{\xstar}^2 \cdot\frac{\sigma_{\epsilon}^2}{n} \cdot\frac{3\cos(\xstar, \bbeta_0)^2}{392} .
\end{talign}
More generally,
if we take $\cos(\xstar, \bbeta_0)^2 = \Theta(1)$, $\snr = o(\frac{p^2}{n})$ and $\snr \geq \frac{1}{6} \frac{p}{n}$ then, 
\begin{talign}
    \E[\langle \xstar, \hblambdastar - \bbeta_0 \rangle^2] \geq \omega( \normt{\xstar}^2 \cdot\frac{\sigma_{\epsilon}^2}{n} ).
\end{talign}
If $\lambda$ is optimized for estimation error or for prediction error with respect to the training distribution, 
the ridge estimator must incur much larger test error then the OLS estimator in some test directions. Such behavior can be viewed as a symptom of over-regularization -- the choice $\lambda_*$ is optimized for the training distribution and cannot be targeted to provide uniformly good performance over all $\xstar$. In \cref{sec:ub} we show how transductive techniques can improve prediction in this regime.

The chief difficulty in lower-bounding the $\xstar$ prediction risk in \cref{thm:ridge_lb} lies in controlling the expectation over the design $\bX$, which enters nonlinearly into the prediction risk. Our proof circumvents this difficulty in two steps. First, the isotropy and independence properties of Wishart matrices are used to reduce the computation to that of a 1-dimensional expectation with respect to the unordered eigenvalues of $\bX$. Second, in the regime $n \geq p$, the sharp concentration of Gaussian random matrices in spectral norm is exploited to essentially approximate $\frac{1}{n} \bX^\top \bX \approx \bI_p$. 
\subsection{Lower Bounds for Lasso Prediction}\label{sec:lb_lasso}
We next provide a strong lower bound on the out-of-sample prediction error of the Lasso estimator $\hblaslambda \triangleq \argmin_{\bbeta} \frac{1}{2n} \normt{\by - \bX \bbeta}^2 + \lambda \normo{\bbeta}$ with regularization parameter $\lambda > 0$.
There has been extensive work \citep[see, e.g.,][]{raskutti2011minimax} establishing minimax lower bounds for the in-sample prediction error and parameter estimation error of any procedure given data from a sparse linear model. However, our focus is on out-of-sample prediction risk for a specific procedure, the Lasso. The point $\xstar$ need not be one of the training points (in-sample) nor even be drawn from the same distribution as the covariates. 
\cref{thm:lasso_lb}, proved in \cref{sec:app_lb_lasso}, establishes that a well-regularized Lasso program suffers significant biases even in a simple problem setting with i.i.d.\ Gaussian covariates and noise.\footnote{A yet tighter lower bound is available if, instead of being fixed, $\xstar$ follows an arbitrary distribution, and the expectation is taken over $\xstar$ as well. See the proof for details.}

\begin{theorem}  \label{thm:lasso_lb}
    Under \cref{assump:well_spec}, fix  $s \geq 0$, and let $\bx_i \distiid \mN(0, \bI_p)$ with independent noise $\bepsilon \sim \mN(0, \bI_n \sigma_{\epsilon}^2)$. If 
    $\lambda \geq (8+2\sqrt{2}) \sigma_{\epsilon} \sqrt{\log (2ep)/n}$ and $p \geq 20$,\footnote{The cutoff at $20$ is arbitrary and can be decreased.} then there exist 
    universal constants $c_{1:3}$ such that for all $n \geq c_1 s^2 \log (2ep)$,
    \begin{align}
    &c_3 \lambda^2 \trimmednorm{\xstar}{s}^2 \geq 
    \sup_{\bbeta_0 \in \mathbb{B}_0(s)} \E[\langle \xstar, \hblaslambda - \bbeta_0\rangle^2] \\
    &\geq  
    \sup_{\bbeta_0 \in \mathbb{B}_0(s), \normi{\bbeta_0} 
    \leq \lambda} \E[\langle \xstar, \hblaslambda - \bbeta_0\rangle^2]
    \geq c_2 \lambda^2 \trimmednorm{\xstar}{s}^2 
     \end{align}
     where
     the \emph{trimmed norm} $\trimmednorm{\xstar}{s}$ is the sum of the magnitudes of the $s$ largest magnitude entries of $\xstar$. 
\end{theorem}

In practice we will always be interested in a known $\xstar$ direction, but the next result clarifies the dependence of our Lasso lower bound on sparsity for  worst-case test directions $\xstar$ (see \cref{sec:app_lb_lqball} for the proof):
\begin{corollary}
    \label{cor:lb_lqball}
    In the setting of Theorem \ref{thm:lasso_lb},
    for $q \in [1,\infty]$, 
    \begin{align}
     & \sup_{\norm{\xstar}_q=1} \sup_{\bbeta_0 \in \mathbb{B}_0(s)} 
     \E[\langle \xstar, \hblaslambda - \bbeta_0\rangle^2] 
     \geq c_2 \lambda^2 s^{2-2/q}. 
    \end{align}
\end{corollary}
We make several comments regarding these results.
First, \cref{thm:lasso_lb} yields an $\xstar$-specific lower bound -- showing that given any potential direction $\xstar$ there will exist an underlying $s$-sparse parameter $\bbeta_0$ for which the Lasso performs poorly. Morever, the magnitude of error suffered by the Lasso scales both with the regularization strength $\lambda$ and the norm of $\bx_\star$ along its top $s$ coordinates. Second, the constraint on the regularization parameter in Theorem \ref{thm:lasso_lb}, $\lambda \gtrsim \sigma_\epsilon\sqrt{\log p/n}$, is a sufficient and standard choice to obtain consistent estimates with the Lasso (see \citet[Ch. 7]{wainwright2017highdim} for example). Third, simplifying to the case of $q=2$, we see that \cref{cor:lb_lqball} implies the Lasso must incur worst-case $\xstar$ prediction error $\gtrsim \frac{\sigma_\epsilon^2 s \log p}{n}$, matching upper bounds for Lasso prediction error \citep[Example 7.14]{wainwright2017highdim}. In particular such a bound is not dimension-free, possessing a dependence on $s \log p$, even though the Lasso is only required to predict well along a \textit{single} direction.

The proof of Theorem \ref{thm:lasso_lb} uses two key ideas.
First, in this benign setting, we can show that $\hblaslambda$ has support strictly contained in the support of $\bbeta_0$ with at least constant probability. We then adapt ideas from the study of debiased lasso estimation in \citep{javanmard2014confidence} to sharply characterize the coordinate-wise bias of the Lasso estimator along the support of $\bbeta_0$; in particular we show that a worst-case $\bbeta_0$ can match the signs of the $s$ largest elements of $\xstar$ and have magnitude $\lambda$ on each non-zero coordinate.
Thus the bias induced by regularization can coherently sum across the $s$ coordinates in the support of $\bbeta_0$.  A similar lower bound follows by choosing $\bbeta_0$ to match the signs of $\bx_\star$ on any subset of size $s$.
This sign alignment between $\bx_\star$ and $\bbeta_0$ is also explored in the independent and concurrent work of \citep[Thm. 2.2]{bellec2019biasing}.

\section{Upper Bounds for Transductive Prediction} \label{sec:ub}
Having established that regularization can lead to excessive prediction bias, we now introduce two classes of estimators which can mitigate this bias using knowledge of the single test direction $\xstar$.
While our presentation focuses on the prediction risk \cref{eq:prediction-risk}, which features an expectation over $\hat{y}$, our proofs in the appendix 
also provide identical high probability upper bounds on $(\hy-\langle \xstar, \bbeta_0 \rangle)^2$. Throughout this section, the $O(\cdot)$ masks constants depending only on $\kappa, \Cmin, \Cmax, \cond$.
\subsection{Javanmard-Montanari (JM)-style Estimator}
\label{sec:jm_ub}
Our first approach to single point transductive prediction is inspired by the debiased Lasso estimator of \citet{javanmard2014confidence} which was to designed to construct confidence intervals for individual model parameters $(\bbeta_{0})_j$.
For prediction in the $\xstar$ direction,
we will consider the following generalization of the Javanmard-Montanari (JM) debiasing construction\footnote{In the event the constraints are not feasible we define $\bw=0$.}:
\begin{talign}
    \yjm
    &= \langle \xstar, \hat{\bbeta} \rangle + \frac{1}{n} \bw^\top \bX^\top(\by-\bX \hat{\bbeta})\qtext{for}
    \label{eq:debias_lasso}\\
    \bw &= \argmin_{\tilde{\bw}} \tilde{\bw}^\top \bSigma_n \tilde{\bw} \ \text{s.t.} \ \normi{\bSigma_n \tilde{\bw} -\xstar} \leq \lambda_{\bw}. \label{eq:jm_program}
\end{talign}
Here, $\hat{\bbeta}$ is any (ideally $\ell_1$-consistent) initial pilot estimate of $\bbeta_0$, like the estimate $\hat{\bbeta}_L(\lambda)$ returned by the Lasso.  When $\xstar = \be_j$ the estimator \cref{eq:debias_lasso} reduces exactly to the program in \citep{javanmard2014confidence}, and equivalent generalizations have been used in \citep{chao2014high, athey2018approximate, cai2017confidence} to construct prediction intervals and to estimate treatment effects.
Intuitively, $\bw$ approximately inverts the population covariance matrix along the direction defined by $\bx_\star$ (i.e., $\bw \approx \bOmega\bx_\star$). The second term in \cref{eq:debias_lasso} can be thought of as a high-dimensional one-step correction designed to remove bias from the initial prediction $\langle \xstar, \hat{\bbeta} \rangle$; see \citep{javanmard2014confidence} for more intuition on this construction. We can now state our primary guarantee for the \emph{JM-style estimator} \cref{eq:debias_lasso}; the proof is given in \cref{sec:app_ub_jm}.

\begin{theorem} \label{thm:ub_jm}
    Suppose \cref{assump:well_spec,,assump:cov,,assump:design,,assump:noise1} hold and that the transductive estimator $\yjm$ of \cref{eq:debias_lasso} is fit with regularization parameter $\lambda_{\bw} = 8a \sqrt{\cond} \kappa^2 \normt{\xstar} \sqrt{\frac{\log ( p \vee n)}{n}}$ for some $a > 0$. Then there is a universal constant $c_1$ such that if $n \geq c_1 a^2 \log(2e (p \vee n)) $,
    \begin{talign}\label{eq:jm-risk-bound}
        &\E[(\yjm - \langle \bbeta_0, \xstar \rangle)^2] \leq \\&O \left( \frac{\sigma_{\epsilon}^2 \xstar \bOmega \xstar}{n} + \brateo^2  (\lambda_{\bw}^2 + \normi{\xstar}^2 \frac{1}{(n \vee p)^{c_3}}) \right).
    \end{talign}
    for $c_3 = \frac{a^2}{4}-\frac{1}{2}$ and
    $\brateo = (\E[\Vert \hat{\bbeta}-\bbeta_0 \Vert_1^4])^{1/4}$, 
    the $\ell_1$ error of the initial estimate. Moreover, if $\lambda_{\bw} \geq \normi{\xstar}$, then $\E[(\yjm - \langle \bbeta_0, \xstar \rangle)^2] = \E[\langle\xstar, \hbbeta-\hbbeta_0\rangle^2]$. 
\end{theorem}
Intuitively, the first term in our bound \cref{eq:jm-risk-bound} can be viewed as the variance of the estimator's prediction along the direction of $\xstar$ while the second term can be thought of as the (reduced) bias of the estimator. We consider the third term to be of higher order since $a$ (and in turn $c_3$) can be chosen as a large constant. Finally, when $\lambda_{\bw} \geq \normi{\xstar}$ the error of the transductive procedure reduces to that of the pilot regression procedure.
When the Lasso is used as the pilot regression procedure we can derive the following corollary to \cref{thm:ub_jm}, also proved in \cref{sec:app_ub_jm_lasso}.

\begin{corollary} 
    \label{cor:ub_jm_lasso}
    Under the conditions of Theorem \ref{thm:ub_jm}, consider the JM-style estimator \cref{eq:debias_lasso} with pilot estimate  $\hat{\bbeta} = \hblaslambda$ with $\lambda \geq 80 \sigma_{\epsilon} \sqrt{\frac{\log(2ep/s_{\bbeta_0})}{n}}$. If $p \geq 20$, then there exist universal constants $c_1$, $c_2$ such that if $\normi{\bbeta_0}/\sigma_{\epsilon} = o(e^{c_1 n})$ and $n \geq c_2 \max \{ \frac{s_{\bbeta_0} \kappa^4}{C_{\min}}, a^2 \} \log(2e(p \vee n))$,
    \begin{talign}
        & \E[(\yjm\hspace{-.05cm} -\hspace{-.05cm} \langle \bbeta_0, \xstar \rangle)^2]\hspace{-.05cm} \leq\hspace{-.05cm} O ( \frac{\sigma_{\epsilon}^2 \xstar \bOmega \xstar}{n} \hspace{-.05cm}+\hspace{-.05cm} \lambda^2 s_{\bbeta_0}^2  (\lambda_{\bw}^2\hspace{-.05cm} + \hspace{-.05cm} \frac{\normi{\xstar}^2}{(n \vee p)^{c_3}})).
    \end{talign}
\end{corollary}

We make several remarks to further interpret this result.
First, to simplify the presentation of the results (and match the lower bound setting of \cref{thm:lasso_lb}) consider the setting in \cref{cor:ub_jm_lasso} with  $a \asymp 1$, $\lambda \asymp \sigma_{\epsilon} \sqrt{\log p/n}$, and $n \gtrsim s_{\bbeta_0}^2 \log p \log (p \vee n)$. Then the upper bound in \cref{thm:ub_jm} can be succinctly stated as $
      O(\frac{\sigma_{\epsilon}^2 \normt{\xstar}^2}{n}).
    $
    In short, the transductive estimator attains a dimension-free rate for sufficiently large $n$. Under the same conditions the Lasso estimator suffers a prediction error of $\Omega(\trimmednorm{\xstar}{s}^2  \frac{\sigma_{\epsilon}^2 \log p}{n})$ as \cref{thm:lasso_lb} and \cref{cor:lb_lqball} establish. 
    Thus transduction guarantees improvement over the Lasso lower bound whenever $\xstar$ satisfies the soft sparsity condition $\frac{\normt{\xstar}}{\trimmednorm{\xstar}{s}} \lesssim \sqrt{\log p}$.  Since $\xstar$ is observable, one can selectively deploy transduction based on the soft sparsity level $\frac{\normt{\xstar}}{\trimmednorm{\xstar}{s}}$ or on bounds thereof. 
    
    Second, the estimator described in  \cref{eq:debias_lasso} and \cref{eq:jm_program} is transductive in that it is tailored to an individual test-point $\xstar$. The corresponding guarantees in \cref{thm:ub_jm} and \cref{cor:ub_jm_lasso} embody a computational-statistical tradeoff. In our setting, the detrimental effects of regularization can be mitigated at the cost of extra computation: the convex program in \cref{eq:jm_program} must be solved for each new $\xstar$. Third, the condition $\normi{\bbeta_0}/\sigma_{\epsilon} = o(e^{c_1 n})$ is not used for our high-probability error bound and is only used to control prediction risk \cref{eq:prediction-risk} on the low-probability event that the (random) design matrix $\mathbf{X}$ does not satisfy a restricted eigenvalue-like condition. For comparison, note that our \cref{thm:lasso_lb} lower bound establishes substantial excess Lasso bias even when $\normi{\bbeta_0} = \lambda = o(1)$.

Finally, we highlight that \citet{cai2017confidence} have shown that the JM-style estimator with a scaled lasso base procedure and $\lambda_{\bw} \asymp \sqrt{\frac{\log p}{n}}$ produce CIs for $\xstar^\top \bbeta_0$ with minimax rate optimal length when $\xstar$ is sparsely loaded. 
Although our primary focus is in improving the mean-square prediction risk \cref{eq:prediction-risk}, we conclude this section by showing that a different setting of $\lambda_{\bw}$ yields minimax rate optimal CIs for dense $\xstar$ and simultaneously minimax rate optimal CIs for sparse and dense $\xstar$ when $\bbeta_0$ is sufficiently sparse:
\begin{proposition}
\label{prop:CI}
Under the conditions of Theorem \ref{thm:ub_jm} with $\sigma_{\epsilon}=1$, consider the JM-style estimator \cref{eq:debias_lasso} with pilot estimate  $\hat{\bbeta} = \hblaslambda$ and $\lambda = 80 \sqrt{\frac{\log(2p)}{n}}$. Fix any $C_1, C_2, C_3 > 0$, and instate the assumptions of \citet{cai2017confidence}, namely that the vector $\xstar$ satisfies $\frac{\max_j \abs{(\xstar)_j}}{\min_j \abs{(\xstar)_j}} \leq C_1$ and  $s_{\bbeta_0} \asymp p^{\gamma}$ for $0 \leq \gamma < \frac{1}{2}$. Then  for $n \gtrsim s_{\bbeta_0} \log p $ the estimator $\yjm$ \cref{eq:debias_lasso} with $\lambda_{\bw} = 8 \sqrt{\cond} \kappa^2 \frac{1}{s_{\bbeta_0} \sqrt{\log p}}\normt{\xstar}$ yields (minimax rate optimal) $1-\alpha$ confidence intervals for $\xstar^\top \bbeta_0$ of expected length
    \begin{itemize}[leftmargin=.4cm]
        \item $O(\norm{\xstar}_{\infty} \cdot s_{\bbeta_0} \sqrt{\frac{\log p}{n}})$ in the dense $\xstar$ regime where $\norm{\xstar}_0=C_3 p^{\gamma_q}$ with $2 \gamma < \gamma_q < 1$ (matching the result of \citep[Thm.~4]{cai2017confidence}). 
        \item $O(\norm{\xstar}_2 \cdot \frac{1}{\sqrt{n}})$  in the sparse $\xstar$ regime of \citep[Thm.~1]{cai2017confidence} where  $\norm{\xstar}_{0} \leq C_2 s_{\bbeta_0}$ if $n \gtrsim s_{\bbeta_0}^2 (\log p)^2$.
    \end{itemize}
    Here the $O(\cdot)$ masks constants depending only on $\kappa, C_1, C_2, C_3, \Cmin, \Cmax, \cond$.
\end{proposition}
The proof can be found in \cref{sec:app_CI}.

\subsection{Orthogonal Moment (OM) Estimators}
\label{sec:om_ub}
Our second approach to single point transductive  prediction is inspired by orthogonal moment (OM) estimation~\citep{chernozhukov2017double}. OM estimators are commonly used to estimate single parameters of interest (like a treatment effect) in the presence of high-dimensional or nonparametric nuisance.
To connect our problem to this semiparametric world, we first frame the task of prediction in the $\xstar$ direction as one of estimating a single parameter, $\theta_0 = \xstar^\top \bbeta_0 $. 
Consider the linear model equation \cref{eq:model}
\begin{talign}
    & y_i = \bx_i^\top \bbeta_0 + \epsilon_i = ((\bU^{-1})^\top \bx_i)^\top \bU \bbeta_0 + \epsilon_i 
\end{talign}
with a data reparametrization defined by the matrix $
\bU = \Vert \xstar \Vert_2 \cdot \begin{bmatrix} \bu_1 \\
\bR
\end{bmatrix}
$ for $\frac{\xstar}{\Vert \xstar \Vert_2} = \bu_1$ so that  $\be_1^\top \bU \bbeta_0 = \xstar^\top \bbeta_0 = \theta_0$. 
Here, the matrix $\bR \in \mR^{(p-1) \times p}$ has orthonormal rows which span the subspace orthogonal to $\bu_1$ -- these are obtained as the non-$\bu_1$ eigenvectors of the projector matrix $\bI_p - \bu_1 \bu_1^\top$. This induces the  data reparametrization $\bx' = [t, \bz] =  (\bU^{-1})^\top \bx$. In the reparametrized basis, the linear model becomes, 
\begin{talign}
    & y_i = \theta_0 t_i + \bz_i^\top \bbf_0 + \epsilon_i, \quad \quad t_i = \bg_0(\bz_i) + \eta_i, \\
    \quad\quad
    & \bq_0(\bz_i) \triangleq \theta_0\bg_0(\bz_i) + \bz_i^\top \bbf_0
    \label{eq:oml_model}
\end{talign}
where we have introduced convenient auxiliary equations in terms of $\bg_0(\bz_i)\triangleq \E[t_i \mid \bz_i]$.

To estimate $\theta_0 = \xstar^\top \bbeta_0$ in the presence of the unknown nuisance parameters $\bbf_0, \bg_0, \bq_0$, 
we introduce a thresholded-variant of the two-stage method of moments estimator proposed in \citep{chernozhukov2017double}. 
The method of moments takes as input a moment function $m$ of both data and parameters that uniquely identifies the target parameter of interest.
Our reparameterized model form \cref{eq:oml_model} gives us access to two different \emph{Neyman orthogonal} moment functions described \citep{chernozhukov2017double}:
\begin{talign}
\textbf{$\bbf$ moments: }
& m(t_i, y_i, \theta, \bz_i^\top \bbf, \bg(\bz_i)) = \\ & (y_i-t_i\theta - \bz_i^\top \bbf)(t_i-\bg(\bz_i)) \label{eq:first_f_moment} \\
\textbf{$\bq$ moments: }
& m(t_i, y_i, \theta, \bq(\bz_i), \bg(\bz_i)) = \\ 
& (y_i-\bq(\bz_i) - \theta(t_i-\bg(\bz_i)))(t_i- \bg(\bz_i)). \label{eq:first_q_moment}
\end{talign}
These orthogonal moment equations enable the accurate estimation of a target parameter $\theta_0$ in the presence of high-dimensional or nonparametric nuisance parameters (in this case $\bbf_0$ and $\bg_0$).
We focus our theoretical analysis and present description on the set of $\bbf$ moments since the analysis is similar for the $\bq$, although we investigate the practical utility of both in \cref{sec:experiments}. 

Our OM proposal to estimate $\theta_0$ now proceeds as follows.  We first split our original dataset of $n$ points into two\footnote{In practice, we use $K$-fold cross-fitting to increase the sample-efficiency of the scheme as in \citep{chernozhukov2017double}; for simplicity of presentation, we defer the description of this slight modification to \cref{sec:expts_oml}.} disjoint, equal-sized folds $(\Xo, \yo) = \{ (\bx_i, y_i) : i \in \{ 1, \hdots , \frac{n}{2} \} \}$ and $(\Xt, \yt) = \{ (\bx_i, y_i) : i \in \{ \frac{n}{2}+1, \hdots , n \} \}$.
Then,
\begin{itemize}[leftmargin=.5cm]
    \vspace{-0.2cm}
    \item The first fold $(\Xo, \yo)$ is used to run two \textit{first-stage} regressions. We estimate
    $\bbeta_0$ by linearly regressing $\yo$ onto $\Xo$ to produce $\hat{\bbeta}$; this provides an estimator of $\bbf_0$ as $\be_{-1}^\top \bU \hat{\bbeta}=\bbtf$.
    Second we estimate $\bg_0$ by regressing $\tone$ onto $\zo$ to produce a regression model $\btg(\cdot) : \mR^{p-1} \to \mR$. 
    Any arbitrary linear or non-linear regression procedure can be used to fit $\btg(\cdot)$.
    \vspace{-0.1cm}
\item Then, we estimate $\E[\eta_1^2]$ as $\mu_2 = \frac{1}{n/2} \sum_{i=\frac{n}{2}+1}^{n} t_i(t_i-\btg(\bz_i))$ where the sum is taken over the second fold of data in $(\Xt, \yt)$; crucially $(t_i, \bz_i)$ are independent of $\btg(\cdot)$ in this expression. \vspace{-0.1cm}
\item If $\mu_2 \leq \tau$ for a threshold $\tau$ we simply output $\yom = \xstar^\top \hat{\bbeta}$. If $\mu_2 \geq \tau$ we estimate $\theta_0$ by solving the empirical moment equation:
\begin{talign}
    & \sum_{i=\frac{n}{2}+1}^{n} m(t_i, y_i, \yom, \bz_i^\top \bbtf, \btg(\bz_i)) =  0 \implies \\
    & \yom = \frac{\frac{1}{n/2} \sum_{i=\frac{n}{2}+1}^{n} (y_i - \bz^\top_i \bbtf)(t_i-\btg(\bz_i))}{\mu_2} \label{eq:om-estimator}
\end{talign}
where the sum is taken over the second fold of data in $(\Xt, \yt)$ and $m$ is defined in \cref{eq:first_f_moment}.
\vspace{-0.3cm}
\end{itemize}
If we had oracle access to the underlying $\bbf_0$ and $\bg_0$, solving the population moment condition  $\E_{t_1, y_1, \bz_1}[m(t_1, y_1, \theta, \bz_1^\top \bbf_0, \bg_0(\bz_1))] = 0$ for $\theta$ would exactly yield $\theta_0=\xstar^\top \bbeta_0$. In practice, we first construct estimates $\bbtf$ and $\btg$ of the unknown nuisance parameters to serve as surrogates for $\bbf_0$ and $\bg_0$ and then solve an empirical version of the aforementioned moment condition to extract $\yom$. A key property of the moments in \cref{eq:first_f_moment} is their Neyman orthogonality: they satisfy $\E [\nabla_{\bz_1^\top \bbf} m(t_1, y_1, \theta_0, \bz_1^\top \bbf_0, \bg_0(\bz_1))] = 0$ and $\E [\nabla_{\bg(\bz_1)} [ m(t_1, y_1, \theta_0, \bz_1^\top \bbf_0, \bg_0(\bz_1))] = 0$. Thus the solution of the empirical moment equations is first-order insensitive to errors arising from using $\bbtf, \btg$ in place of $\bbf_0$ and $\bg_0$. Data splitting is further used to create independence across the two stages of the procedure. 
In the context of testing linearly-constrained hypotheses of the parameter $\bbeta_0$, \citet{zhu2018linear} propose a two-stage OM test statistic based on the transformed $f$ moments introduced above; they do not use cross-fitting and specifically employ adaptive Dantzig-like selectors to estimate $\bbf_0$ and $\bg_0$.
Finally, the thresholding step allows us to control the variance increase that might arise from $\mu_2$ being too small and thereby enables our non-asymptotic prediction risk bounds. Before presenting the analysis of the OM estimator \cref{eq:om-estimator} we introduce another condition\footnote{This assumption is not essential to our result and could be replaced by assuming $\eta_i$ satisfies $\E[\eta_i | \bz_i] = 0$ and is almost surely (w.r.t. to $\bz_i$) sub-Gaussian with a uniformly (w.r.t. to $\bz_i$) bounded variance parameter.}:
\begin{assumption}
    The noise $\eta_i$ is independent of $\bz_i$.
    \label{assump:noise2}
    \vspace{-.2cm}
\end{assumption}
Recall $\btg$ is evaluated on the (independent) second fold data $\bz$. We now obtain our central guarantee for the OM estimator (proved in \cref{sec:app_oml_ub}).
\begin{theorem} \label{thm:oml_ub}
    Let \cref{assump:well_spec,,assump:cov,,assump:design,,assump:noise1,,assump:noise2} hold, and assume that $\bg_0(\bz_i) = \bg_0^\top \bz_i$ in \cref{eq:oml_model} for $\bg_0 = \argmin_{\bg} \E[(t_1-\bz_1^\top \bg)^2]$. 
    Then the thresholded orthogonal ML estimator $\yom$ of \cref{eq:om-estimator} with $\tau = \frac{1}{4}\sigma_{\eta}^2$ satisfies
    \begin{talign}
        & \E[(\yom-\xstar^\top \bbeta_0)^2] \leq \\
        &\normt{\xstar}^2 \left[ O(\frac{ \sigma_\epsilon^2}{\sigma_{\eta}^2 n}) +
        O(\frac{\bratet^2 \gratet^2}{(\sigma_{\eta}^2)^2})  + O(\frac{\bratet^2 \sigma_\eta^2 +   \gratet^2 \sigma_{\epsilon}^2}{(\sigma_{\eta}^2)^2 n})  \right] \label{eq:om_main_guarantee}
    \end{talign}
    where $\bratet =  (\E[\Vert 
\hat{\bbeta}-\bbeta_0 \Vert_2^4])^{1/4}$ and $\gratet = (\E [ (\btg(\bz_n)-\bg_0(\bz_n))^4 ])^{1/4}$  denote the expected prediction errors of the first-stage estimators. 
\end{theorem}
Since we are interested in the case where $\hat{\bbeta}$ and $\btg(\cdot)$ have small error (i.e., $\bratet = \gratet = o(1)$), the first term in \cref{eq:om_main_guarantee} can be interpreted as the variance of the estimator's prediction along the direction of $\xstar$, while the remaining terms represent the reduced bias of the estimator. We first instantiate this result in the setting where both $\bbeta_0$ and $\bg_0$ are estimated using ridge regression (see \cref{sec:app_oml_ub_ridge} for the corresponding proof).
\begin{corollary}[OM Ridge]
\label{cor:oml_ub_ridge}
    Assume $\normi{\bbeta_0}/\sigma_{\epsilon} = O(1)$. In the setting of \cref{thm:oml_ub}, suppose $\hat{\bbeta}$ and $\hat{\bg}(\bz_i) = \hat\bg^\top \bz_i$ are fit with the ridge estimator with regularization parameters $\lambda_{\bbeta}$ and $\lambda_{\bg}$ respectively. Then there exist universal constants $c_{1:5}$ such that if $ p \geq 20$, $c_1 \frac{n^2 \Cmin}{p \cond} e^{-n c_2/\kappa^4 \cond^2} \leq \lambda_{\bbeta} \leq c_3 \left(\cond \Cmax n \right)^{1/3}$, and $c_4 \frac{n^2 \Cmin}{p \cond} e^{-n c_2/\kappa^4 \cond^2} \leq \lambda_{\bg} \leq p \left(\frac{\Cmax \normt{\xstar}^2}{\cond} \frac{n}{p}  \sigma_{\eta}^4 \right)^{1/3}$ for $n \geq c_5 \kappa^4 \cond^2 p$,
    \begin{talign}
         & \E[(\yom-\xstar^\top \bbeta_0)^2] \\
         &\leq   \normt{\xstar}^2 \left[ O(\frac{ \sigma_\epsilon^2}{\sigma_{\eta}^2 n}) +  O(\frac{p^2}{(\sigma_{\eta}^2)^2 n^2}) + O(\frac{p (\sigma_{\eta}^2 +  \sigma_{\epsilon}^2)}{(\sigma_{\eta}^2)^2 n^2}) \right].
    \end{talign}
\end{corollary}
Similarly, when $\bbeta_0$ and $\bg_0$ are estimated using the Lasso we conclude the following (proved in \cref{sec:app_oml_ub_ridge}).
\begin{corollary}[OM Lasso]
\label{cor:oml_ub_lasso}
 In the setting of \cref{thm:oml_ub}, suppose $\hat{\bbeta}$ and $\hat{\bg}(\bz_i) = \hat\bg^\top \bz_i$ are fit with the Lasso with regularization parameters $\lambda_{\bbeta} \geq 80 \sigma_{\epsilon} \sqrt{\log(2ep/s_{\bbeta_0})/n}$ and $\lambda_{\bg} \geq 80 \sigma_{\eta} \sqrt{\log(2ep/s_{\bg})/n}$ respectively. If $p \geq 20$, $s_{\bbeta_0}=\norm{\bbeta_0}_0$, and $s_{\bg_0}=\norm{\bg_0}_0$, then there exist universal constants $c_1, c_2$ such that if $\normi{\bbeta_0}/\sigma_{\epsilon}=o(e^{c_1 n})$, then for $n \geq \frac{c_1 \kappa^4}{C_{\min}} \max \{ s_{\bbeta_0}, s_{\bg} \} \log(2ep)$,
    \begin{talign}
         & \E[(\yom-\xstar^\top \bbeta_0)^2] \leq  \\
         & \normt{\xstar}^2 \left[ O(\frac{ \sigma_\epsilon^2}{\sigma_{\eta}^2 n}) +  O(\frac{\lambda_{\bbeta}^2 \lambda_{\bg}^2 s_{\bbeta_0} s_{\bg_0}}{(\sigma_{\eta}^2)^2}) + O(\frac{\lambda_{\bbeta}^2 s_{\bbeta_0} \sigma_{\eta}^2 + \lambda_{\bg}^2 s_{\bg_0} \sigma_{\epsilon}^2}{(\sigma_{\eta}^2)^2 n}) \right].
    \end{talign}
\end{corollary}
We make several comments regarding the aforementioned results.
First, \cref{thm:oml_ub} possesses a double-robustness property. In order for the dominant bias term $O(\bratet^2 \gratet^2)$ to be small, it is sufficient for \textit{either} $\bbeta_0$ or $\bg_0$ to be estimated at a fast rate or \textit{both} to be estimated at a slow rate. As before, the estimator is transductive and adapted to predicting along the direction $\xstar$. Second, in the case of ridge regression, to match the lower bound of \cref{cor:ridge_opt_lb}, consider the setting where $n = \Omega( p^2)$, $\snr = o(\frac{p^2}{n})$, $\cos(\xstar, \bbeta_0)^2=\Theta(1)$ and $\snr \gtrsim  \frac{p}{n}$. Then, the upper bound\footnote{Note that in this regime, $\sqrt{\snr} = \normt{\bbeta_0}/\sigma_{\epsilon}=o(1)$ and hence the condition $\normi{\bbeta_0}/\sigma_{\epsilon}=O(1)$ in \cref{cor:oml_ub_ridge} is satisfied.} can be simplified to
    $
        O(\normt{\xstar}^2 \frac{\sigma_{\epsilon}^2}{n})
    $.
    By contrast, \cref{cor:ridge_opt_lb} shows the error of the optimally-tuned ridge estimator is lower bounded by $\omega(\normt{\xstar}^2 \frac{\sigma_{\epsilon}^2}{n})$; for example, the error is $\Omega(p \normt{\xstar}^2 \frac{\sigma_{\epsilon}^2}{n})$ when $\snr = 
\frac{1}{6}\frac{p}{n}$. Hence, the performance of the ridge estimator can be significantly worse then its transductive counterpart. Third, if we consider the setting of \cref{cor:oml_ub_lasso} where $n \gtrsim s_{\bbeta_0} s_{\bg_0} (\log p)^2$ while we take $\lambda_{\bbeta} \asymp \sigma_{\epsilon} \sqrt{\log p/n}$ and $\lambda_{\bg} \asymp \sigma_{\eta} \sqrt{\log p/n}$, the error of the OML estimator 
attains the fast, dimension-free $O(\normt{\xstar}^2 \frac{\sigma_{\epsilon}^2}{n})$ rate. On the other hand, \cref{cor:lb_lqball} shows the Lasso suffers prediction error $\Omega(\trimmednorm{\xstar}{s}^2  \frac{\sigma_{\epsilon}^2 \log p}{n})$, and hence again strict improvement is possible over the baseline when $\frac{\norm{\xstar}_2}{\trimmednorm{\xstar}{s}} \lesssim \sqrt{\log p}$. Finally, although \cref{thm:oml_ub} makes stronger assumptions on the design of $\bX$ than the JM-style estimator introduced in \cref{eq:jm_program} and \cref{eq:debias_lasso}, one of the primary benefits of the OM framework is its flexibility. All that is required for the algorithm are ``black-box'' estimates of $\bg_0$ and $\bbeta_0$ which can be obtained from more general ML procedures than the Lasso.

\section{Experiments}\label{sec:experiments}
We complement our theoretical analysis with a series of numerical experiments highlighting the failure modes of standard inductive prediction. In \cref{sec:excess_bias,sec:dist_shift}, error bars represent $\pm 1$ standard error of the mean computed over 20 independent problem instances. 
We provide complete experimental set-up details in \cref{sec:expts} and code replicating all experiments at \url{https://github.com/nileshtrip/SPTransducPredCode}.

\subsection{Excess Lasso Bias without Distribution Shift}
\label{sec:excess_bias}
We construct problem instances for Lasso estimation by independently generating $\bx_i \sim \mN(0, \bI_p)$,  $\epsilon_i \sim \mN(0,1)$, and $(\bbeta_0)_j \sim \mN(0,1)$ for $j$ less then the desired sparsity level $s_{\bbeta_0}$ while $(\bbeta_0)_j=0$ otherwise. 
We fit the Lasso estimator, JM-style estimator with Lasso pilot, and the OM $f$-moment estimator with Lasso first-stage estimators.
We set all hyperparameters to their theoretically-motivated values.
\begin{figure}[!ht]
\centering
\begin{minipage}[c]{.49\linewidth}
\includegraphics[width=\linewidth]{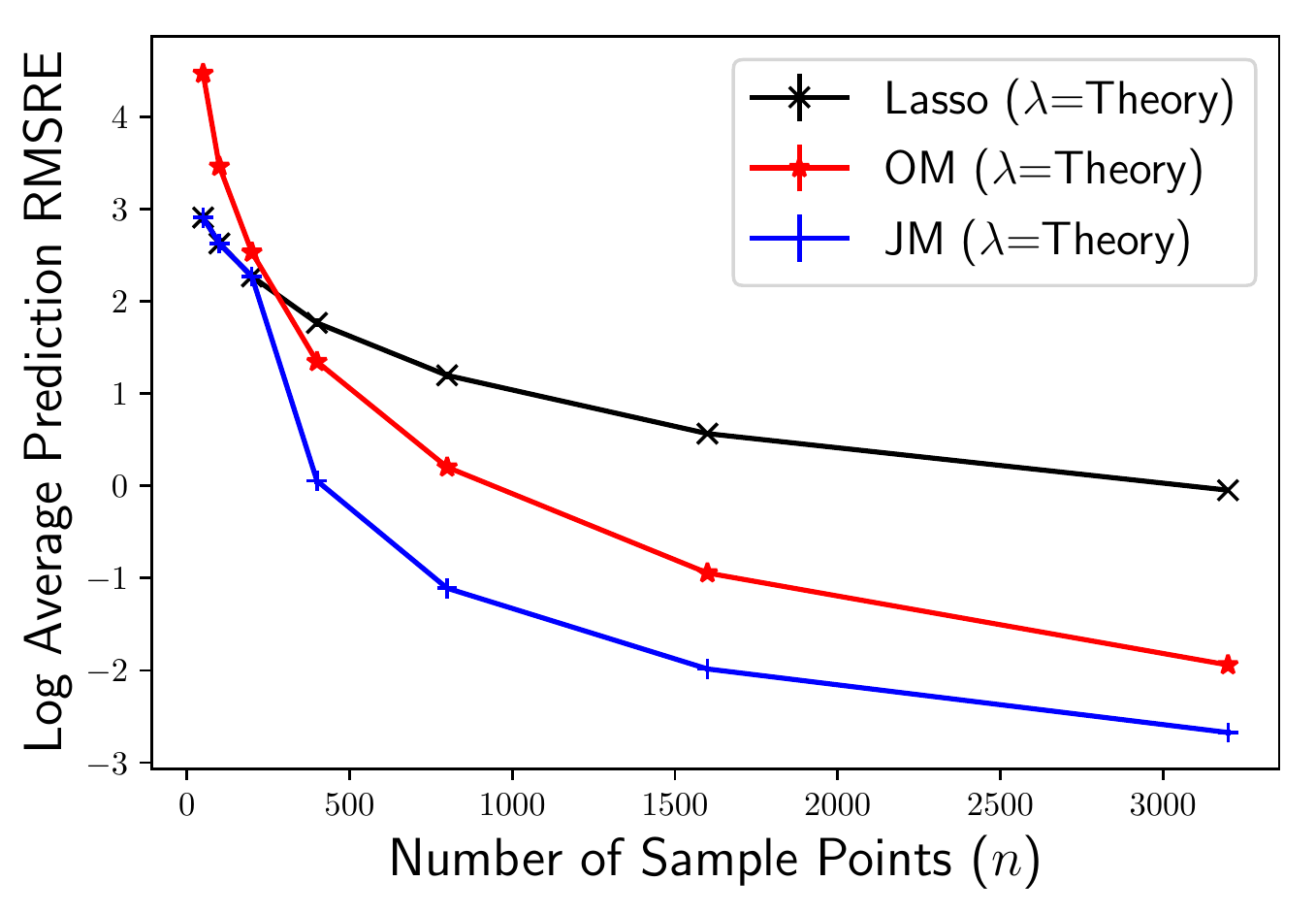}
\end{minipage}
\begin{minipage}[c]{.49\linewidth}
\includegraphics[width=\linewidth]{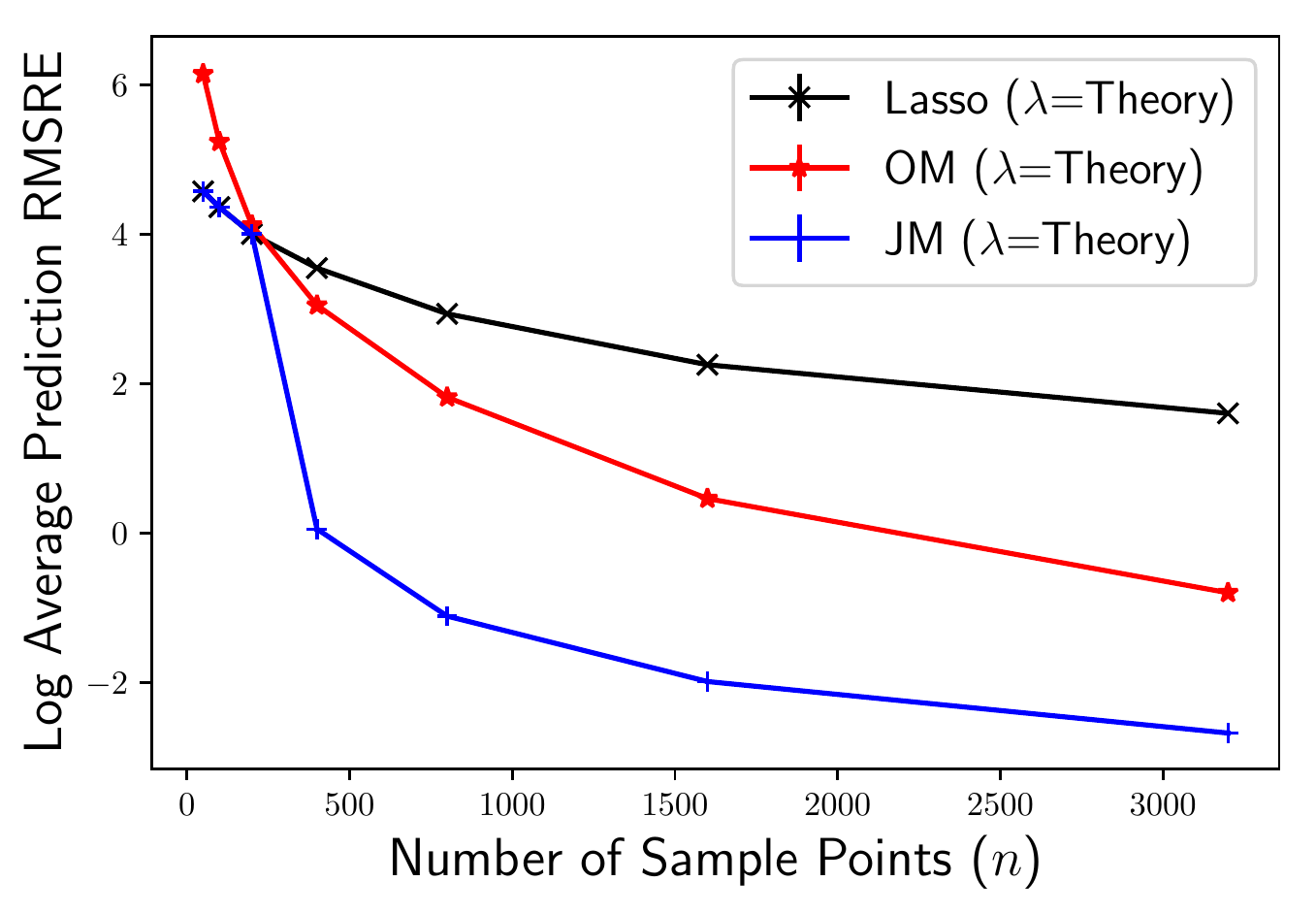}
\end{minipage}
\caption{Lasso vs. OM and JM Lasso prediction without distribution shift.
Hyperparameters are set according to theory (see \cref{sec:excess_bias}). Left: $p=200, s_{\bbeta_0}=20$. Right: $p=200, s_{\bbeta_0}=100$. 
}
\label{fig:synthetic_theory}
\end{figure}
As \myfig{synthetic_theory} demonstrates, both transductive methods significantly reduce the prediction risk of the Lasso estimator when the hyperparameters are calibrated to their theoretical values, even for a dense $\bbeta_0$ (where $\frac{p}{s_{\bbeta_0}} = 2$).

\subsection{Benefits of Transduction under Distribution Shift}\label{sec:dist_shift}
The no distribution shift simulations of \cref{sec:excess_bias} corroborate the theoretical results of \cref{cor:ub_jm_lasso,,cor:oml_ub_lasso}. However, since our transductive estimators are tailored to each individual test point $\xstar$, we expect these methods to provide an even greater gain when the test distribution deviates from the training distribution.

In \myfig{synthetic_distribution_shift_ridge1}, we consider two cases where the test distribution is either mean-shifted or covariance-shifted from the training distribution and evaluate the ridge estimator with the optimal regularization parameter for the training distribution, $\lambda_* = \frac{p \sigma_{\epsilon}^2}{\normt{\bbeta_0}^2}$. 
We independently generated $\bx_i \sim \mN(0, \bI_p)$,  $\epsilon_i \sim \mN(0,1)$, and $\bbeta_0 \sim \mN(0,\frac{1}{\sqrt{p}} \bI_p)$.
In the case with a mean-shifted test distribution, we generated $\xstar \sim \mN(10 \bbeta_0, \bI_p)$ for each problem instance while the covariance-shifted test distribution was generated by taking $\xstar \sim \mN(0, 100 \bbeta_0 \bbeta_0^\top)$. The plots in \myfig{synthetic_distribution_shift_ridge1} show the OM estimator with $\lambda_*$-ridge pilot provides significant gains over the baseline  $\lambda_*$-ridge estimator. 

\begin{figure}[!ht]
\centering
\begin{minipage}[c]{.49\linewidth}
\includegraphics[width=\linewidth]{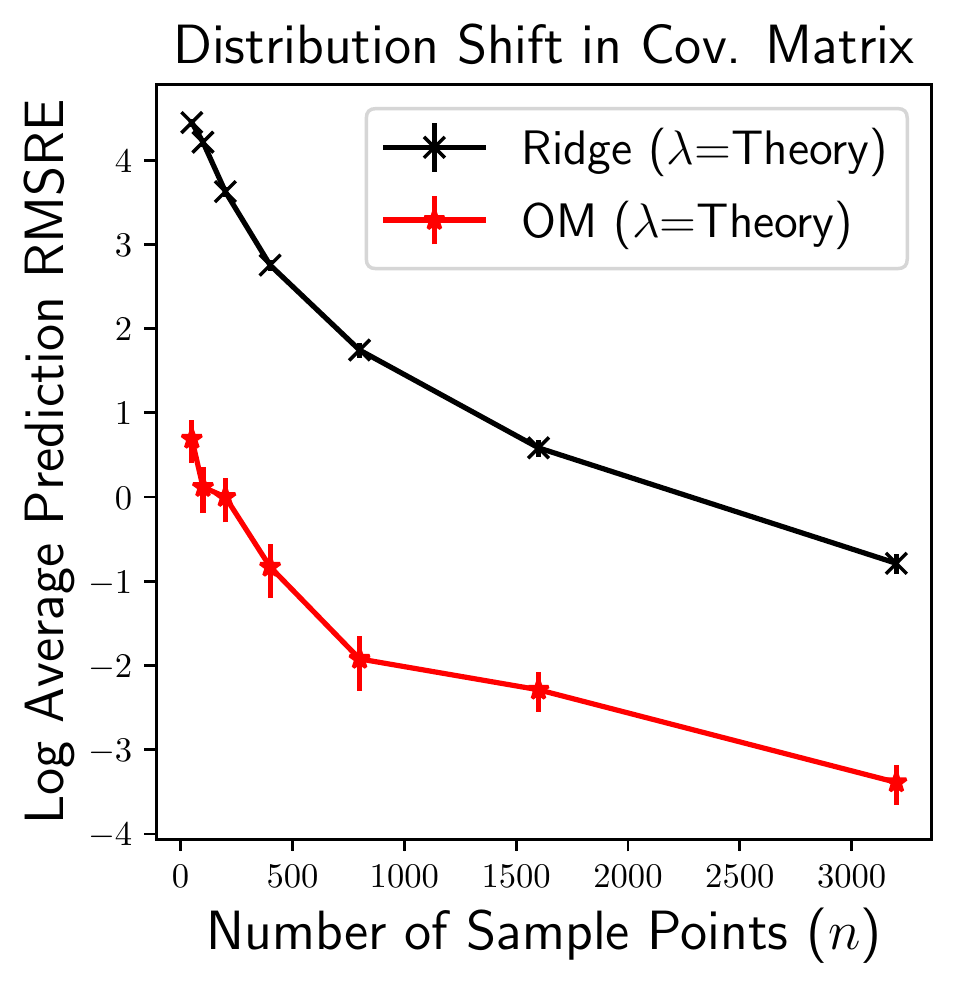}
\end{minipage}
\begin{minipage}[c]{.49\linewidth}
\includegraphics[width=\linewidth]{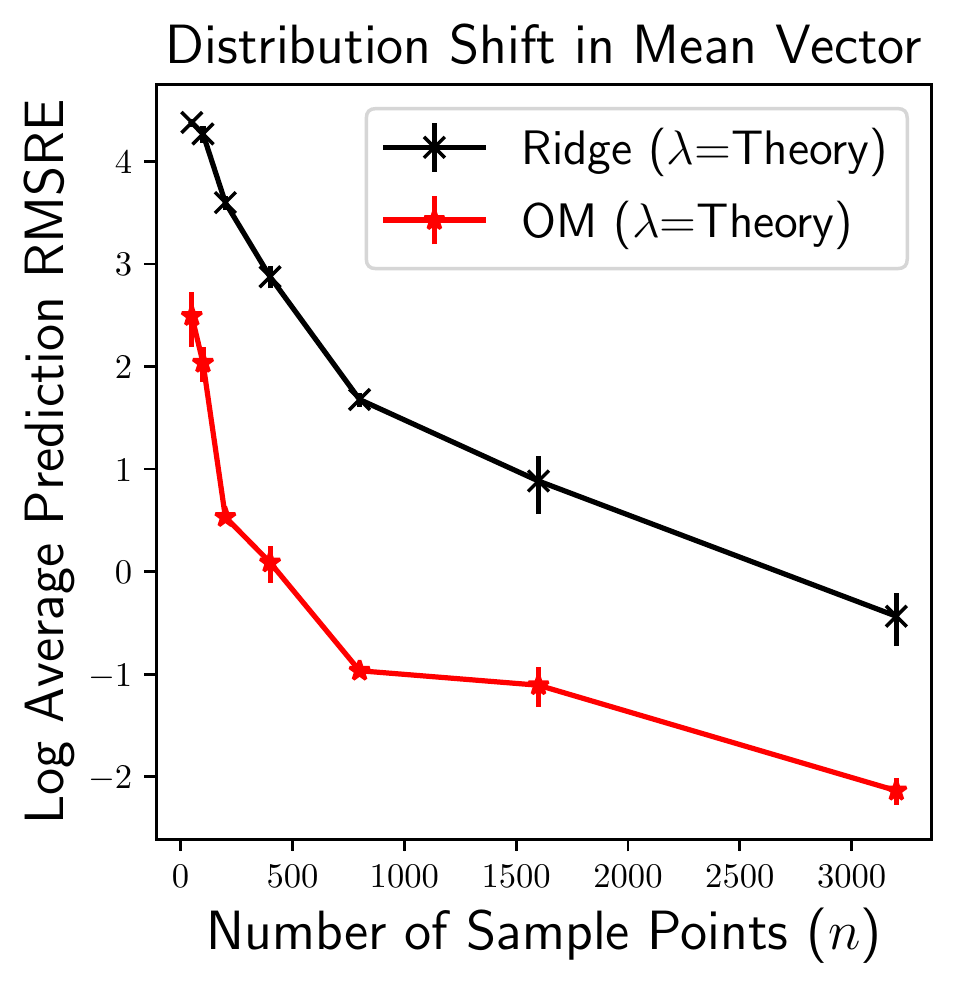}
\end{minipage}
\caption{Ridge vs. OM ridge prediction ($p=200$) under train-test distribution shift. 
Hyperparameters are set according to theory. 
}
\label{fig:synthetic_distribution_shift_ridge1}
\end{figure}

\begin{figure}[!ht]
\centering
\begin{minipage}[c]{.49\linewidth}
\includegraphics[width=\linewidth]{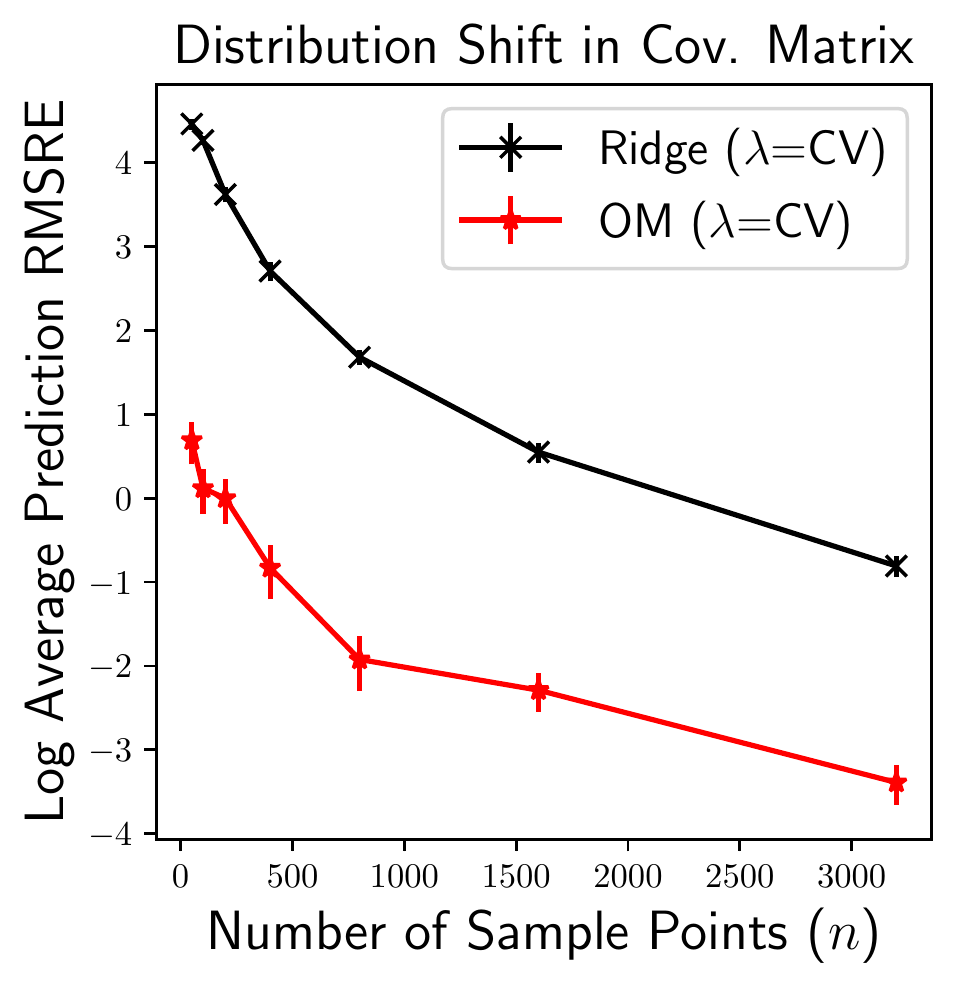}
\end{minipage}
\begin{minipage}[c]{.49\linewidth}
\includegraphics[width=\linewidth]{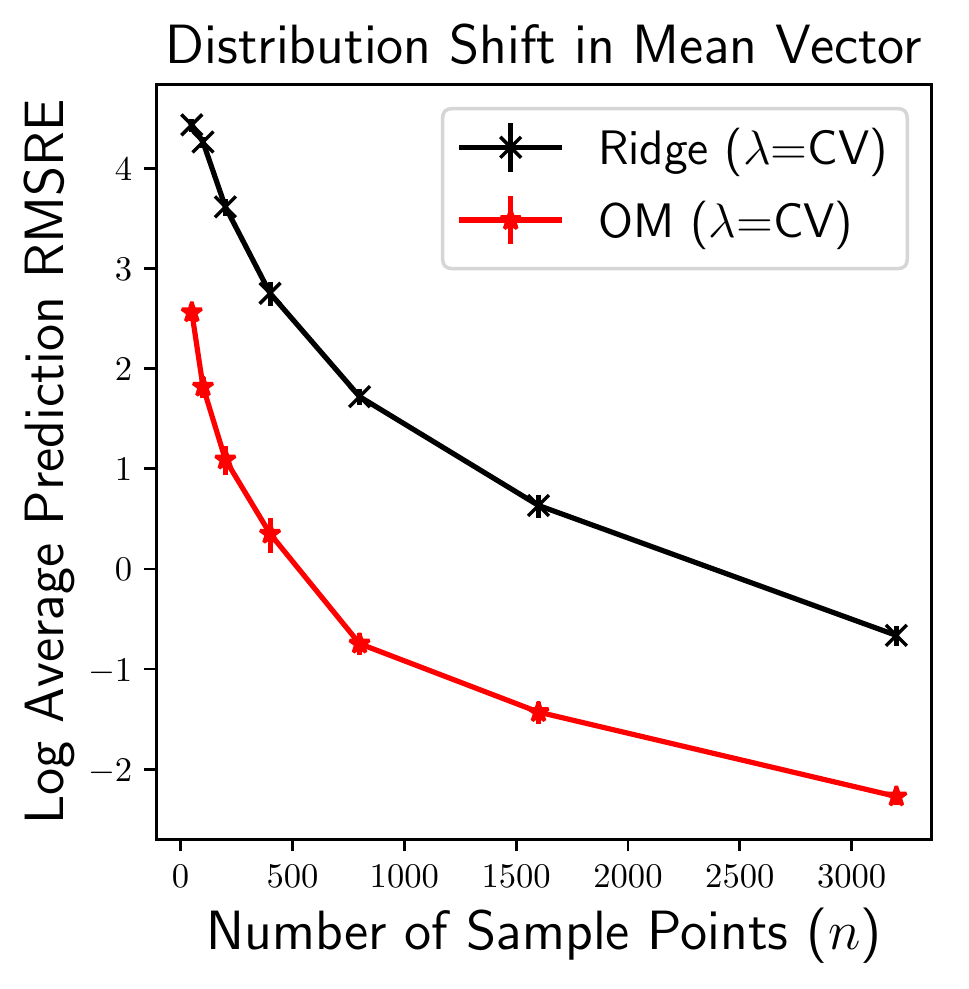}
\end{minipage}
\caption{Ridge vs. OM ridge prediction ($p=200$) under train-test distribution shift. 
Hyperparameters are set according to CV. 
}
\label{fig:synthetic_distribution_shift_ridge2}
\end{figure}

\begin{figure}[!ht]
\centering
\begin{minipage}[c]{.49\linewidth}
\includegraphics[width=\linewidth]{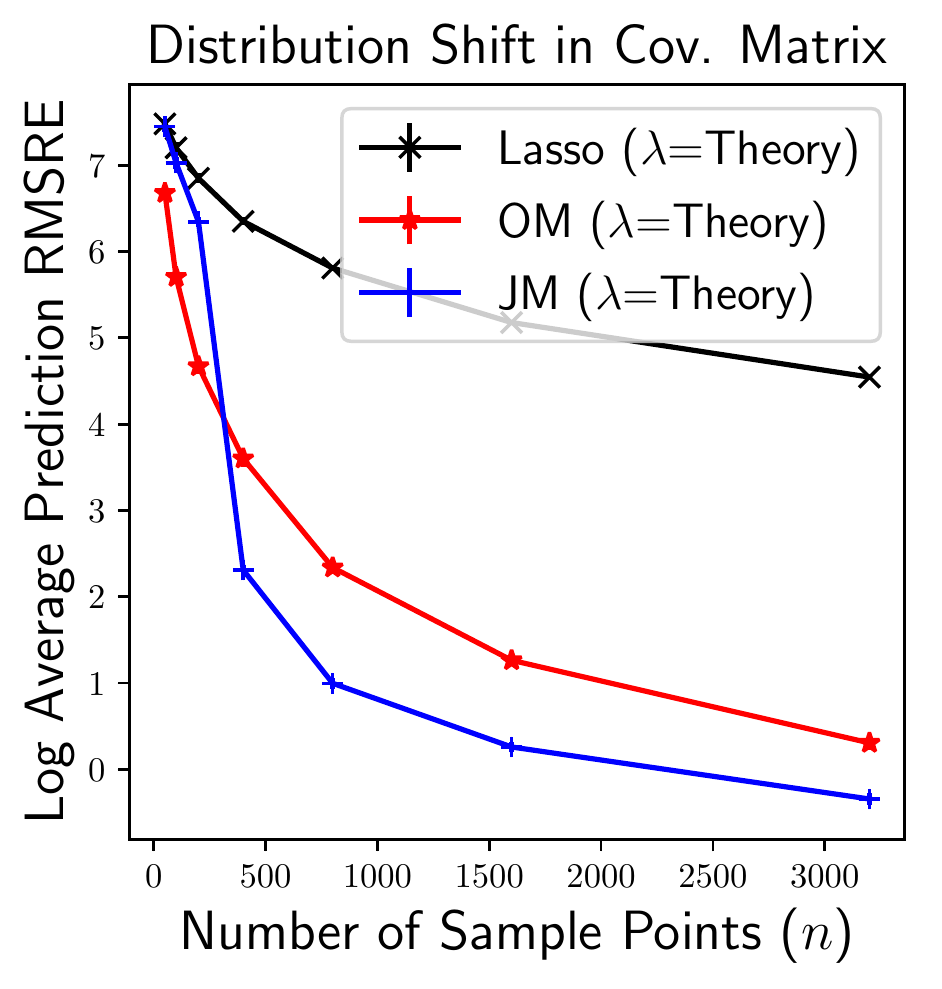}
\end{minipage}
\begin{minipage}[c]{.49\linewidth}
\includegraphics[width=\linewidth]{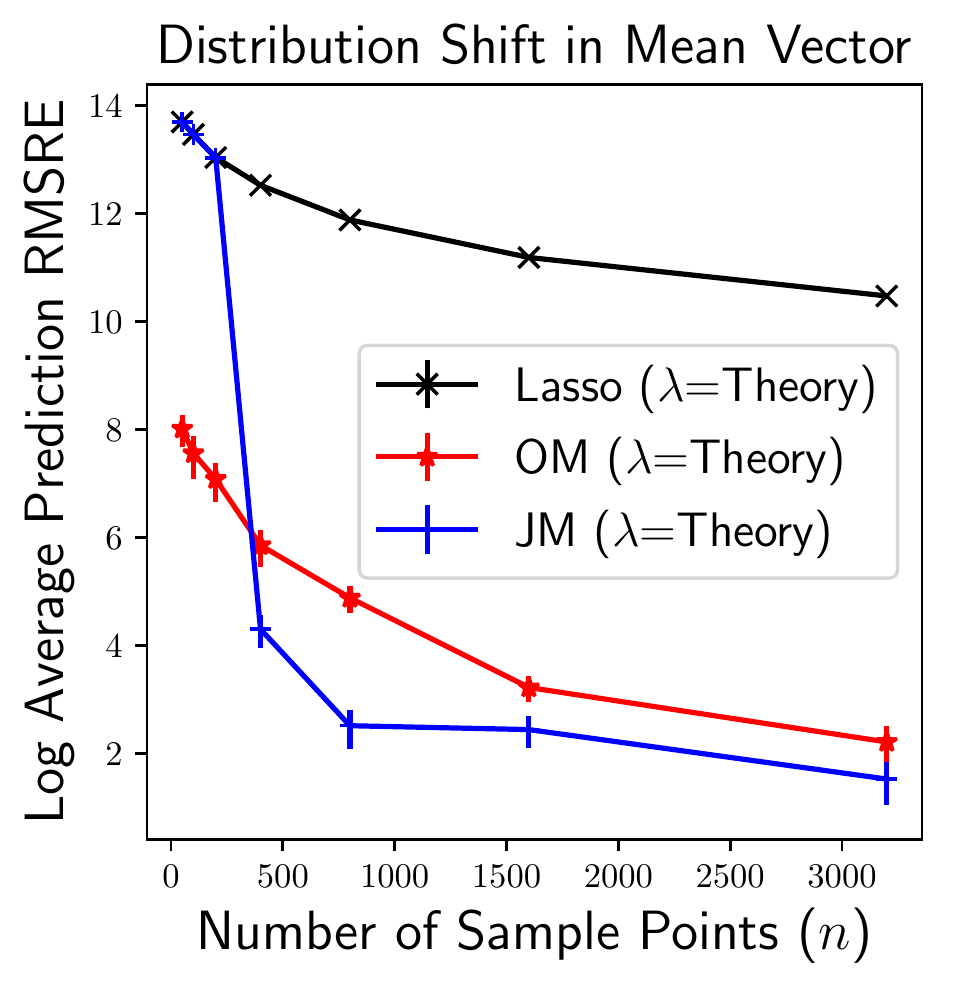}
\end{minipage}
\caption{Lasso vs. OM and JM Lasso prediction ($p=200$) under mean ($s_{\bbeta_0}=100$) or covariance ($s_{\bbeta_0}=20$) train-test distribution shifts. Hyperparameters are set according to theory. 
}
\label{fig:synthetic_distribution_shift1}
\end{figure}

\begin{figure}[!ht]
\centering
\begin{minipage}[c]{.49\linewidth}
\includegraphics[width=\linewidth]{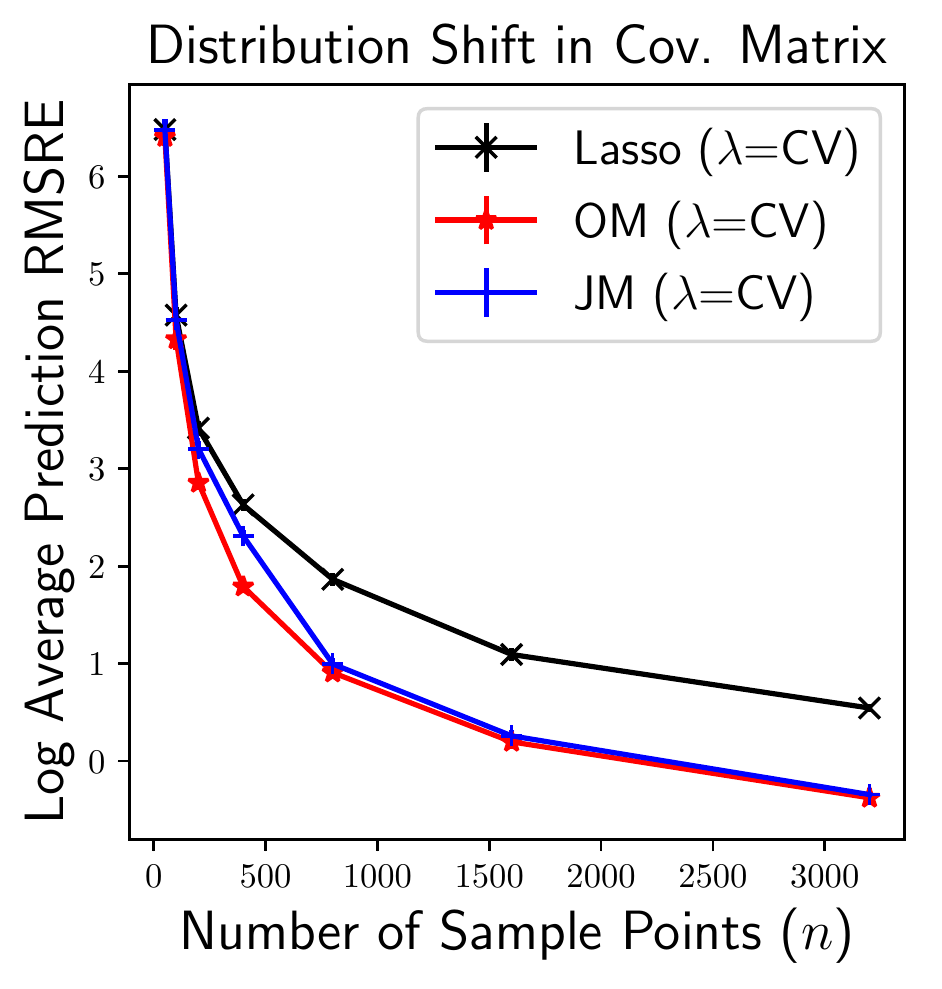}
\end{minipage}
\begin{minipage}[c]{.49\linewidth}
\includegraphics[width=\linewidth]{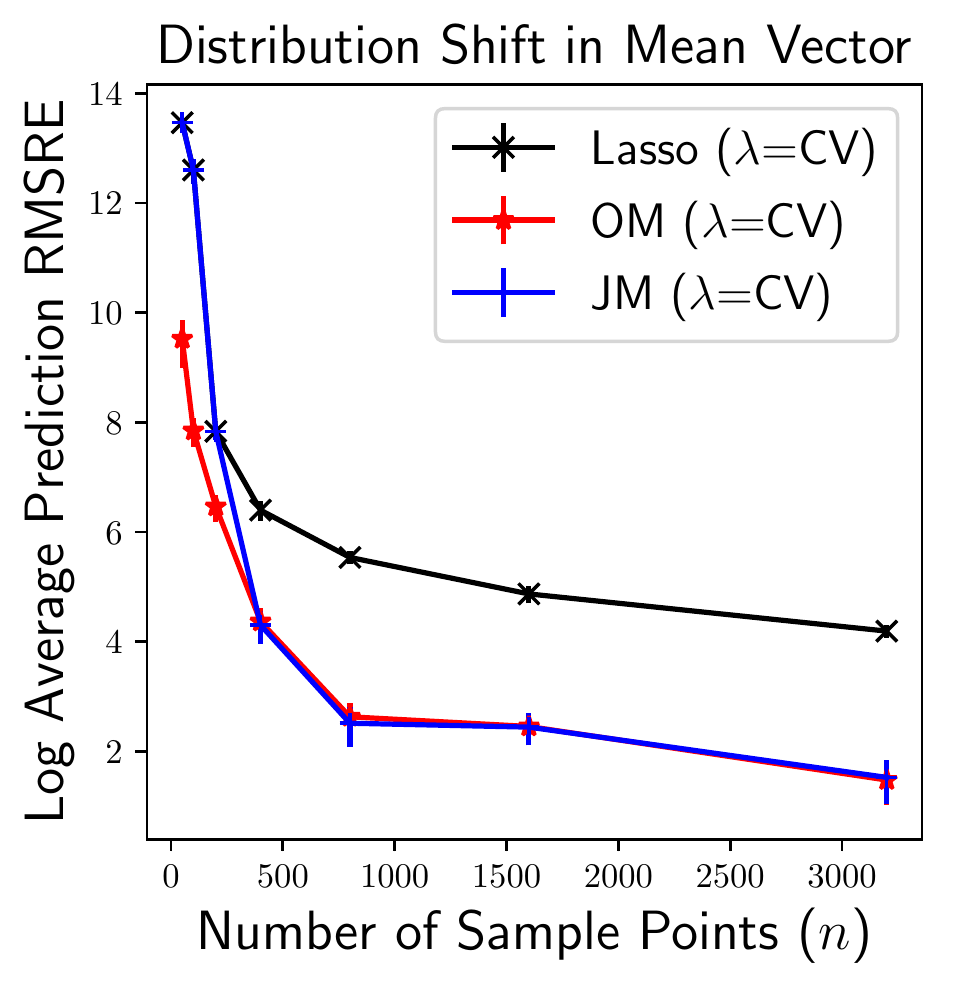}
\end{minipage}
\caption{Lasso vs. OM and JM Lasso prediction ($p=200$) under mean ($s_{\bbeta_0}=100$) or covariance ($s_{\bbeta_0}=20$) train-test distribution shifts. Hyperparameters are set according to CV. 
}
\label{fig:synthetic_distribution_shift2}
\end{figure}

In \myfig{synthetic_distribution_shift1} we also consider two cases where the test distribution is shifted for Lasso estimation but otherwise identical to the previous set-up in \cref{sec:excess_bias}. For covariance shifting, we generated $(\xstar)_i \distind \mN(0, 100)$ for $i \in \supp(\bbeta_0)$ and $(\xstar)_i = 0$ otherwise for each problem instance. For mean shifting, we generated $\xstar \sim \mN(10 \bbeta_0, \bI_p)$ for each problem instance. The first and second plots in \myfig{synthetic_distribution_shift1} show the transductive effect of the OM and JM estimators improves prediction risk with respect to the Lasso when the regularization hyperparameters are selected via theory. 

We also note that \myfig{synthetic_distribution_shift_ridge2} and \myfig{synthetic_distribution_shift2} compares CV-tuned ridge or Lasso to OM and JM with CV-tuned base procedures---showing the benefit of transduction in this practical setting where regularization hyperparameters are chosen by CV. As the first and second plots in \myfig{synthetic_distribution_shift_ridge2} show, selecting $\lambda$ via CV leads to over-regularization of the ridge estimator, and the transductive methods provide substantial gains over the base ridge estimator. In the case of the Lasso, the first and second plots in \myfig{synthetic_distribution_shift2} show the residual bias of the CV Lasso also causes it to incur significant error in its test predictions, while the transductive methods provide substantial gains by adapting to each $\bx_\star$.
\subsection{Improving Cross-validated Prediction}\label{sec:real_data}
\begin{table*}[!bth]
\centering
\caption{Test set RMSE of OLS; CV-tuned ridge, Lasso, and elastic net; OM and JM transductive CV-tuned ridge, Lasso, and elastic net; and prior transductive approaches (TD Lasso, Ridge, and KNN) on real-world datasets. All hyperparameters are set via CV. 
Error bars represent a delta method interval based on $\pm$1 standard error of the mean squared error over the test set.}
\resizebox{1.0\textwidth}{!}{
\begin{tabular}{@{}llllll@{}}
\toprule
\textbf{Method}           & \textbf{Wine}         & \textbf{Parkinson}               & \textbf{Fire}         & \textbf{Fertility}              & \textbf{Triazines} (no shift)        \\ \toprule
OLS              & 1.0118$\pm$0.0156 & 12.7916$\pm$0.1486  & 82.7147$\pm$35.5141 & 0.3988$\pm$0.0657 & 0.1716$\pm$0.037  \\
\midrule
Ridge            & 0.9936$\pm$0.0155  & 12.5267$\pm$0.1448 & 82.3462$\pm$35.5955 & 0.399$\pm$0.0665 & 0.1469$\pm$0.0285 \\
OM $f$ (Ridge)   & 0.9883$\pm$0.0154  & 12.4686$\pm$0.1439 & 82.3522$\pm$35.5519 & 0.3987$\pm$0.0655 & \textbf{0.1446$\pm$0.029} \\
OM $q$ (Ridge)   & \textbf{0.7696$\pm$0.0145} & \textbf{12.0891$\pm$0.1366} & \textbf{81.9794$\pm$35.7872} & \textbf{0.3977$\pm$0.0653} & 0.1507$\pm$0.0242 \\
\midrule
Lasso            & 0.9812$\pm$0.0155 & 12.2535$\pm$0.1356  & 82.0656$\pm$36.0321 & 0.4092$\pm$0.0716 & 0.1482$\pm$0.0237 \\
 JM (Lasso)       & 1.0118$\pm$0.0156 & 12.7916$\pm$0.1486 & 82.7147$\pm$35.5141 & 0.3988$\pm$0.0657 & 0.173$\pm$0.0367  \\
OM $f$ (Lasso)   & 0.9473$\pm$0.0152  & \textbf{11.869$\pm$0.1339} & \textbf{81.794$\pm$35.5699}  & 0.398$\pm$0.0665 & \textbf{0.1444$\pm$0.0239} \\
OM $q$ (Lasso)   & \textbf{0.7691$\pm$0.0144} & 11.8692$\pm$0.1339  & 81.811$\pm$35.5637 & \textbf{0.3976$\pm$0.0656} & 0.1479$\pm$0.0226 \\
\midrule
Elastic          & 0.9652$\pm$0.0154 & 12.2535$\pm$0.1356 & 81.8428$\pm$35.8333 & 0.4092$\pm$0.0716 & 0.1495$\pm$0.0238 \\
OM $f$ (Elastic) & 0.9507$\pm$0.0152 & \textbf{11.8369$\pm$0.1338} & \textbf{81.7719$\pm$35.6166}  & 0.398$\pm$0.0655 & \textbf{0.1445$\pm$0.024}  \\
OM $q$ (Elastic) & \textbf{0.7693$\pm$0.0145}  & 11.8658$\pm$0.1341  & 81.803$\pm$35.6485 & \textbf{0.3976$\pm$0.0657} & 0.147$\pm$0.0228  \\ \bottomrule
TD Lasso \citep{alquier2012transductive}          & 0.9813$\pm$0.0154 & 12.2535$\pm$0.1358 & 82.0657$\pm$36.0320 & 0.4092$\pm$0.0716 & 0.1483$\pm$0.0237 \\
TD Ridge \citep{chapelle2000transductive} & 0.8411$\pm$0.0004 & 12.2534$\pm$0.0021 & 82.0664$\pm$2.567  & 0.4089$\pm$0.0128 & 0.1735$\pm$0.0004  \\
TD KNN \citep{cortes2007transductive} & 0.8345$\pm$0.0153  & 12.3326$\pm$0.1447  & 81.9467$\pm$35.8340 & \textbf{0.3845$\pm$0.0760} & 0.1510$\pm$0.0240  \\ \bottomrule
\end{tabular}
}
\label{table:1}
\end{table*}
Motivated by our findings on synthetic data, we next report the performance of our methods on 5 real datasets with and without distribution shift. 
We also include the popular elastic net estimator as a base regression procedure alongside ridge and the Lasso. All hyperparameters are selected by CV.
For the OM estimators we exploited the flexibility of the framework by including a suite of methods for the auxiliary $\bg$ regressions: Lasso estimation, random forest regression, and a $\bg=0$ baseline. Amongst these, we select the method with the least estimated asymptotic variance, which can be done in a data-dependent way \textit{without} introducing any extra hyperparameters into the implementation.  
The $\bbf$ and $\bq$ regressions were always fit with Lasso, ridge,  or elastic net estimation.
See \myapp{expts} for further details on the methodology and datasets from the UCI dataset repository \citep{Dua:2019}. 

In \cref{table:1} we see that the OM estimators generically provide gains over the CV Lasso, CV ridge, and CV elastic net on datasets with intrinsic distribution shift and perform comparably on a dataset without explicit distribution shift.
On Wine, we see a substantial performance gain from 0.96-0.99 RMSE without transduction to 0.77  with OM $q$ transduction.
The gains on other datasets are smaller but notable as they represent consistent improvements over the de facto standard of CV prediction.

We also report the performance of ordinary least squares (OLS) which produces an unbiased estimate of the entire parameter vector $\bbeta_0$. OLS fares worse than most methods on each dataset due to an increase in variance.  In contrast, our proposed transductive procedures limit the variance introduced by targeting a single parameter of interest, $\langle \bx_\star,\bbeta_0\rangle$.

Finally, we evaluated three existing transductive prediction methods---the transductive Lasso (TD Lasso) of \cite{alquier2012transductive,bellec2018prediction}, transductive ridge regression (TD Ridge) \cite{chapelle2000transductive}, and transductive ridge regression with local (kernel) neighbor labelling (TD KNN) \cite{cortes2007transductive}---on each dataset, tuning all hyperparameters via CV. 
TD Lasso does not significantly improve upon the Lasso baseline on any dataset. TD Ridge only improves upon the baselines on Wine but is outperformed by OM $q$. TD KNN also underperforms OM $q$ on every dataset except Fertility.  

\section{Discussion and Future Work}\label{sec:concl}
We presented two single point transductive prediction procedures that, given advanced knowledge of a test point, can significantly improve the prediction error of an inductive learner.
We provided theoretical guarantees for these procedures and demonstrated their practical utility, especially under distribution shift, on synthetic and  real data. Promising directions for future work include improving our OM debiasing techniques using higher-order orthogonal moments \citep{mackey2017orthogonal} and exploring the utility of these debiasing techniques for other regularizers (e.g., group Lasso \citep{yuan2006model} penalties) and models such as generalized linear models and kernel machines.

\bibliography{betterdebiasing.bib}
\bibliographystyle{icml2020}

\newpage
\onecolumn
\appendix
\section{Notation}
We first establish several useful pieces of notation used throughout the Appendices. We will say that a mean-zero random variable $x$ is sub-gaussian, $x \sim \sG(\kappa)$, if $\E[\exp(\lambda x))] \leq \exp(\frac{\kappa^2 \lambda^2}{2})$ for all $\lambda$. We will say that a mean-zero random variable $x$ is sub-exponential, $x \sim \sE(\nu, \alpha)$, if $\E[\exp(\lambda x)] \leq \exp(\frac{\nu^2 \lambda^2}{2})$ for all $\abs{\lambda} \leq \frac{1}{\alpha}$. We will say that a mean-zero random vector is sub-gaussian, $\bx \sim \sG(\kappa)$, if $\forall \bv \in \mR^p$, $\E[\exp(\bv^\top \bx)] \leq \exp(\frac{\kappa^2 \Vert \bv \Vert_2^2}{2})$. Moreover a standard Chernoff argument shows if $x \sim \sE(\nu, \alpha)$ then $\Pr[\abs{x} \geq t] \leq 2 \exp(-\frac{1}{2} \min(\frac{t^2}{\nu^2}, \frac{t}{\alpha}))$.  

\section{Proofs for \cref{sec:lb_ridge}: Lower Bounds for Prediction with Ridge Regression}

\label{sec:app_lb_ridge}
Here we provide lower bounds on the prediction risk of the ridge regression estimator. To do so, we show that under Gaussian design and independent Gaussian noise $\bepsilon$ the ridge regression estimator can perform poorly.

Recall we define the ridge estimator as $\hblambda = \arg \min_{\bbeta} \frac{1}{2} \left( \normt{\by-\bX \bbeta}^2 + \lambda \normt{\bbeta}^2 \right)$ which implies $\hblambda = (\bX^\top \bX + \lambda \bI_p)^{-1} \bX^\top \by$. For convenience we further define $\hat{\bSigma} = \frac{\bX^\top \bX}{n}$, $\hat{\bSigma}_{\lambda} = \frac{\bX^\top \bX}{n} + \frac{\lambda}{n} \bI_p$ and $\bPi = \bI_p-(\hat{\bSigma}_{\lambda})^{-1} \hat{\bSigma}$.
Note that under \cref{assump:well_spec}, $\hblambda-\bbeta_0 = -\bPi \bbeta_0 + \hat{\bSigma}^{-1}_{\lambda} \bX^\top \bepsilon/n$, which can be thought of as a standard bias-variance decomposition for the ridge estimator.
We begin by stating a standard fact about Wishart matrices we will repeatedly use throughout this section.
\begin{proposition}
    Let $\bx_i \distiid \mN(0, \bI_p)$ for $i \in [n]$. Then the eigendecomposition of the sample covariance $\hbSigma = \frac{1}{n} \sum_{i=1}^n \bx_i  \bx_i^\top = \bV^\top \bD \bV$ satisfies the following properties:
    \begin{itemize}
        \item The orthonormal matrix $\bV$ is uniformly distributed (with respect to the Haar measure) over the orthogonal group $O(p)$.
        \item The matrices $\bV$ and $\bD$ are independent. Moreover, by isotropy, $\bD$ is equivalent in distribution to the random matrix $z \bI_p$ where $z$ is an unordered eigenvalue of $\hbSigma$.
    \end{itemize}
    \label{prop:wishart}
    \end{proposition}
\begin{proof}
    Statements and proofs of these standard facts about Wishart matrices can be found in \citet{wish}.
\end{proof}

\subsection{\cref{thm:ridge_lb}}
\label{sec:app_ridge_lb}

We now provide the proof of our primary lower bound on the prediction risk of the ridge estimator,
\begin{proof}[Proof of \cref{thm:ridge_lb}]
    The first statement follows by using \cref{lem:condx_predrisk} and taking the expectation over $\bX$, 
    \begin{talign}
    \E[\langle \xstar, \hblambda-\bbeta_0 \rangle^2] = \E[(\xstar^\top \bPi \bbeta_0)^2] + \frac{\sigma_{\epsilon}^2}{n} \xstar^\top \E[(\hat{\bSigma}_{\lambda})^{-1} \hat{\bSigma} (\hat{\bSigma}_{\lambda})^{-1}] \xstar = \E[(\xstar^\top \bPi \bbeta_0)^2] +  \sigma_{\epsilon}^2 \frac{\normt{\xstar}^2}{n} \E[\frac{z}{(z+\lambda/n)^2}] 
\end{talign}
The computation of the variance term uses the eigendecomposition of $\bSigma$ and \cref{prop:wishart},
\begin{align}
    \E[(\hat{\bSigma}_{\lambda})^{-1} \hat{\bSigma} (\hat{\bSigma}_{\lambda})^{-1}] = \E[\bV^\top \E[(\bD+\lambda/n \bI_p)^{-2} \bD] \bV] =  \E[\frac{z}{(z+\lambda/n)^2}] \bI_p.
\end{align}
We now lower bound the bias. Again by \cref{prop:wishart} and the eigendecomposition of $\hbSigma$, $\E[\bPi] = \E[\frac{\lambda/n}{z+\lambda/n}] \bI_p$. Using Jensen's inequality,
\begin{align}
    \E[(\bbeta_0^\top \bPi \xstar)^2] \geq (\bbeta_0^\top \E[\bPi] \xstar)^2 = \normt{\xstar}^2 \normt{\bbeta_0}^2 \cos(\xstar, \bbeta_0)^2 \E[(\frac{\lambda/n}{z+\lambda/n})]^2.
\end{align}
The final expectation over the unordered eigenvalue distribution can be controlled using the sharp concentration of Gaussian random matrices. Namely for $n \geq p$,  $\normt{\hbSigma-\bSigma}  \leq 2 \epsilon + \epsilon^2$ for $\epsilon = \sqrt{\frac{p}{n}} + \delta$ with probability at least $1-2\epsilon^{-n \delta^2/2}$ \citep[Theorem 6.1, Example 6.2]{wainwright2017highdim}. Taking $\delta = 1/2 \sqrt{p/n}$ and assuming that $p \geq 20$ we conclude that $\normt{\hat{\bSigma}-\bSigma}  \leq 6 \sqrt{\frac{p}{n}}$ with probability at least $\frac{1}{2}$ -- let $\mE$ denote this event. Note that by the Weyl inequalities, on the event $\mE$, all of the eigenvalues of $\hat{\bSigma}$ are uniformly close to the eigenvalues of $\bSigma$. Hence if $n \geq p$, on $\mE$ we must have that $\hbSigma \preceq 7 \bI_p$, and hence the unordered eigenvalue $z \leq 7$ as well. Thus it follows that $(\E[\frac{1}{\lambda/n+z}])^2 \geq (\E[\frac{1}{\lambda/n+z} \bI[\mE]])^2 \geq (\E[\frac{1}{\lambda/n+7} \bI[\mE]])^2 \geq \frac{1}{4} \frac{1}{(\lambda/n+7)^2}$. Combining the expressions yields the conclusion.
\end{proof}

\subsection{\cref{cor:ridge_opt_lb}}
\label{sec:app_ridge_opt_lb}
We now prove \cref{cor:ridge_opt_lb}.

\begin{proof}[Proof of \cref{cor:ridge_opt_lb}]
The expression for $\lambda_* = \arg \min_{\lambda} \E[\Vert \hblambda-\bbeta_0 \Vert_2^2]$ can be computed using \cref{lem:ridge_lambdamin}. Since, $\arg \min_{\lambda} \E[\Vert \tilde{y}-\tilde{\bx}^\top \hblambda \Vert_2^2 = \E[\Vert \hblambda-\bbeta_0 \Vert_2^2] + \sigma_{\epsilon}^2$, equality of the minimizers follows for both expressions.

Define $\snr = \frac{\normt{\bbeta_0}^2}{\sigma_{\epsilon}^2}$ and $a=\sqrt{\frac{4 C}{ \snr}}$. If, in addition, $n \geq a^2$ and $\lambda \geq \frac{7a n}{\sqrt{n}-a}$, we claim,
\begin{talign}
        \E[\langle \xstar, \hblambda-\bbeta_0 \rangle^2] \geq C \cos(\xstar, \bbeta_0)^2\cdot \normt{\xstar}^2 \cdot\frac{\sigma_{\epsilon}^2}{n}.
\end{talign}
This lower bound follows by simply rearranging the lower bound from \cref{thm:ridge_lb} -- some algebraic manipulation give the conditions that $\frac{\lambda/n}{\lambda/n+7} \geq \frac{a}{\sqrt{n}} \implies \lambda \geq \frac{a}{\sqrt{n}} (\lambda+7n) \implies \lambda (1-\frac{a}{\sqrt{n}}) \geq 7 a \sqrt{n} \implies \lambda \geq \frac{7a \sqrt{n}}{1-\frac{a}{\sqrt{n}}} \implies \lambda \geq \frac{7an}{\sqrt{n}-a}$.

After defining $\lambda_{*} = p/\snr=b$ the previous inequality over $\lambda_*$ to achieve the desired conclusion, can be rearranged to $b(\sqrt{n}-a) \geq 7an \implies n -\frac{b}{7a} \sqrt{n} + \frac{b}{7} \implies n-\frac{b}{7a} \sqrt{n}+\frac{b}{7} \leq 0$. The corresponding quadratic equation in $\sqrt{n}$ has roots $r_{+} = \frac{1}{14} \left(\frac{b}{a}+\frac{\sqrt{b^2-28 a^2 b}}{a}\right), r_{-} = \frac{1}{14} \left(\frac{b}{a}-\frac{\sqrt{b^2-28 a^2 b}}{a}\right)$. In order to ensure both roots are real we must have $b \geq 28a^2 \implies p \geq 120 C$. The condition that $r_{-} \leq \sqrt{n} \leq r_+$ can be equivalently expressed as, 
\begin{align}
    \abs{\sqrt{n}-\frac{1}{14} \frac{b}{a}} \leq \frac{\sqrt{b^2-28a^2b}}{a} \iff \\
    \abs{\sqrt{n}-\frac{1}{14} \frac{p}{\sqrt{4C \snr}}} \leq \sqrt{\frac{p^2}{4C\snr} - 28 \frac{p}{\snr}}.
\end{align}
Defining $C$ such that $\sqrt{n}-\frac{1}{14} \frac{p}{\sqrt{4C \snr}} = 0 \implies C = \frac{p^2}{784 n \snr}$. The remaining condition simplifies as, $\sqrt{\frac{p^2}{4C\snr} - 28 \frac{p}{\snr}} \geq 0 \implies 196 n-28\frac{p}{\snr} \geq 0 \implies n \geq \frac{1}{7} \frac{p}{\snr}$. The condition $p \geq 120 C \implies n \geq \frac{1}{6} \frac{p}{\snr}$. Accordingly, under these conditions,
\begin{talign}
    \E[\langle \xstar, \hblambda-\bbeta_0 \rangle^2] \geq \frac{\cos(\xstar, \bbeta_0)^2}{784} \frac{p^2}{n \snr} \normt{\xstar}^2 \cdot\frac{\sigma_{\epsilon}^2}{n}
\end{talign}
\end{proof} 

We first compute the (conditional on $\bX$) prediction risk of this estimator alongst $\xstar$ as,
\begin{lemma}
    \label{lem:condx_predrisk}
    Let the independent noise distribution be Gaussian, $\bepsilon \sim \mN(0, \bI_n \sigma_{\epsilon}^2)$, and \cref{assump:well_spec} hold. Then,
    \begin{align}
    \E[\langle \xstar, \hblambda-\bbeta_0 \rangle^2 | \bX] = (\xstar^\top \bPi \bbeta_0)^2 + \sigma_{\epsilon}^2 \xstar^\top (\hat{\bSigma}_{\lambda})^{-1} \hat{\bSigma} (\hat{\bSigma}_{\lambda})^{-1} \xstar/n
    \end{align}
\end{lemma}
\begin{proof}
    Using the standard bias-variance decomposition $\hblambda-\bbeta_0 = -\bPi \bbeta_0 + \hat{\bSigma}^{-1}_{\lambda} \bX^\top \bepsilon/n$, squaring and taking the expectation over $\bepsilon$ (which is mean-zero) gives the result.
\end{proof}

We now calculate the optimal choice of the ridge parameter $\lambda$ to minimize the parameter error in the $\ell_2$ distance.
\begin{lemma}
    Under \cref{assump:well_spec}, let $\bx_i \distiid \mN(0, \bI_p)$ with independent noise $\bepsilon \sim \mN(0, \bI_n \sigma_{\epsilon}^2)$. Then,
    \begin{align}
        \E \left[\normt{\hblambda-\bbeta_0}^2 \right] = \normt{\bbeta_0}^2 \E[(\frac{\lambda/n}{z+\lambda/n})^2] + \frac{\sigma_{\epsilon}^2 p}{n} \E[\frac{z}{(z+\lambda/n)^2}]
    \end{align}
    and the optimal $\lambda_* = \arg \min_{\lambda} \E \left[\normt{\hblambda-\bbeta_0}^2 \right]$, is $\lambda_*/p = \frac{\sigma_{\epsilon}^2}{\normt{\bbeta_0}^2}$.
    \label{lem:ridge_lambdamin}
\end{lemma}
\begin{proof}
    We first compute the (expected) mean-squared error. Using \cref{lem:condx_predrisk}, summing over $\xstar=\be_i$, and taking a further expectation over $\bX$ we have that, 
    \begin{align}
        \E \left[\normt{\hblambda-\bbeta_0}^2 \right] = \E \left[\sum_{i=1}^{p} (\be_i^\top \bPi \bbeta_0)^2 \right] + \frac{\sigma_{\epsilon}^2}{n} \E \left [\sum_{i=1}^p \be_i^\top (\hat{\bSigma}_{\lambda})^{-1} \hat{\bSigma} (\hat{\bSigma}_{\lambda})^{-1} \be_i \right]
    \end{align}
    The computation of both the bias and variance terms exploits \cref{prop:wishart} along with the eigendecomposition of $\hbSigma$. For the bias term,
    \begin{align}
        \E[\sum_{i=1}^{p} (\be_i^\top \bPi \bbeta_0)^2] = \E[\bbeta_0^\top \bPi^2 \bbeta_0] = \E[\bbeta_0^\top \bV^\top (\E[\bI_p-2(\bD+\lambda \bI_p)^{-1} \bD + (\bD+\lambda \bI_p)^{-2} \bD^2]) \bV \bbeta_0] = \normt{\bbeta_0}^2 v
    \end{align}
    where $v = \E[(\frac{\lambda/n}{\lambda/n+z})^2]$. Similarly for the variance term,
    \begin{align}
    \frac{\sigma_{\epsilon}^2}{n} \E[\sum_{i=1}^p \be_i^\top (\hat{\bSigma}_{\lambda})^{-1} \hat{\bSigma} (\hat{\bSigma}_{\lambda})^{-1} \be_i] = \frac{\sigma_{\epsilon}^2}{n} \E[\Tr[(\hat{\bSigma}_{\lambda})^{-1} \hat{\bSigma} (\hat{\bSigma}_{\lambda})^{-1}]] = \frac{\sigma_{\epsilon}^2}{n} \E[\Tr[\bV \E[w] \bI_p \bV^\top]] = \frac{\sigma_{\epsilon}^2 p}{n} \E[w]
    \end{align}
    where $\E[w]=\E[\frac{z}{(z+\lambda/n)^2}]$.
    Combining we have that, 
    \begin{align}
        \E \left[\normt{\hblambda-\bbeta_0}^2 \right] = \normt{\bbeta_0}^2 \E[(\frac{\lambda/n}{z+\lambda/n})^2] + \frac{\sigma_{\epsilon}^2 p}{n} \E[\frac{z}{(z+\lambda/n)^2}].
    \end{align}
    In general this expression is a complicated function of $\lambda$, however conveniently, 
    \begin{align}
    \frac{d}{d \lambda} \E \left[ \normt{\hblambda-\bbeta_0}^2 \right] = 2\lambda n \normt{\bbeta_0}^2  \E[\frac{z}{(z+\lambda n)^3}] - 2 n^2 \frac{\sigma_{\epsilon}^2 p}{n} \E[\frac{z}{(\lambda n+z )^3}] \implies \lambda_*/p = \frac{ \sigma_{\epsilon}^2}{\normt{\bbeta_0}^2}.
\end{align}

\end{proof}

\section{Proofs for \cref{sec:lb_lasso}: Lower Bounds for Prediction with the Lasso}
Here we provide lower bounds on the prediction risk of the Lasso estimator.
In order to do so we will exhibit a benign instance of the design matrix for which for the Lasso performs poorly. 

\subsection{\cref{thm:lasso_lb}}

\label{sec:app_lb_lasso}

We begin by stating a more general version of \cref{thm:lasso_lb} and provide its proof

\begin{theorem}  \label{thm:lasso_lb_app}
    Under \cref{assump:well_spec}, fix any $s \geq 0$, and let $\bx_i \distiid \mN(0, \bI_p)$ with independent noise $\bepsilon \sim \mN(0, \bI_n \sigma_{\epsilon}^2)$. Then, if $\hblaslambda$ denotes the solution of the Lasso program,
    with regularization parameter chosen as $\lambda \geq (8+2\sqrt{2}) \sigma_{\epsilon} \sqrt{\log (2ep)/n}$, and $p \geq 20$, there exist 
    universal constants $c_1, c_2, c_3$ such that for all $n \geq c_1 s^2 \log (2ep)$ and for fixed $\xstar \sim \Pstar$ independently of $\bX, \bepsilon$,
    \begin{align}
    \sup_{\bbeta_0 \in \mathbb{B}_0(s)} \E[\langle \xstar, \hblaslambda - \bbeta_0\rangle^2] 
    \geq  \sup_{\substack{\bbeta_0 \in \mathbb{B}_0(s),\\ \normi{\bbeta_0} 
    \leq \lambda}} \E[\langle \xstar, \hblaslambda - \bbeta_0\rangle^2] \geq c_2 \lambda^2 \maxseig{\E[\xstar \xstar^\top]} \geq c_2 \lambda^2 \trimmednorm{\E[\xstar]}{s}^2 
     \end{align}
     where
     the \emph{trimmed norm} $\trimmednorm{\xstar}{s}$ is the sum of the magnitudes of the $s$ largest magnitude entries of $\xstar$ and $\maxseig{\E[\xstar \xstar^\top]}$ is the maximum s-sparse eigenvalue of $\E[\xstar \xstar^\top]$. Moreover, for deterministic $\xstar$,
     \begin{align}
         \sup_{\bbeta_0 \in \mathbb{B}_0(s)} \E[\langle \xstar, \hblaslambda - \bbeta_0\rangle^2] \leq c_3 \lambda^2 \trimmednorm{\E[\xstar]}{s}^2  
     \end{align}
\end{theorem}

\begin{proof}[Proof of \cref{thm:lasso_lb} and \cref{thm:lasso_lb_app}] 

Let $\vstar$ denote the maximum $s$-sparse eigenvector of $\E[\xstar \xstar^\top]$ (which is normalized as have $\normt{\bv} = \trimmednorm{\bv}{s} = 1$) and $\maxseig{\E[\xstar \xstar^\top]}$ its corresponding eigenvalue. We begin by restricting $\bbeta_0$ to have support on these $s$ coordinates of $\vstar$, denoted by $S$; we subsequently choose the magnitude of the elements $\bbeta_0$. Now  under the conditions of the result, we can guarantee support recovery of the Lasso solution, $S_{\hat{\bbeta}_L} \subseteq S_{\bbeta_0} \equiv S$, with probability at least $\frac{1}{2}$ by Proposition \ref{prop:uncond_supp_rec}. Denote this event by $\mS$. 

Thus, for this choice of $\bbeta_0$,
\begin{align}
    & \E[\langle \xstar, \hblaslambda - \bbeta_0\rangle^2] \geq \E[\langle (\xstar)_S, (\hblaslambda - \bbeta_0)_S \rangle^2 \bI[\mS]] = \E[\langle (\xstar)_S, \bI[\mS] (\hblaslambda - \bbeta_0)_S \rangle^2]  \\
    &\geq \langle \E[\bI[\mS] (\hblaslambda - \bbeta_0)_S], \E[\xstar \xstar^\top]_{S} \E[\bI[\mS] (\hblaslambda - \bbeta_0)_S \rangle 
        \label{eq:lb_1}
\end{align}
using Jensen's inequality and independence of $\xstar$ and $\hblaslambda$ in the inequality.

We now focus on characterizing the bias of the Lasso solution $\hblaslambda$ on the coordinates contained in $S$ (in fact using properties of the debiased Lasso estimator). Consider a single coordinate $i \in S$, and without loss of generality assume that $(\xstar)_i > 0$, in which case we choose $(\bbeta_0)_i >0$. We will argue that the magnitude of $(\bbeta_0)_i$ can be chosen so that $\E[(\hblaslambda-\bbeta_0)_i] < c < 0$ for appropriate $c$ under the conditions of the theorem. Note that under our assumptions $\kappa=C_{\max}=C_{\min}=1$ for the following.

Recall, since $\by = \bX \bbeta_0 + \bepsilon$, from the KKT conditions applied to the Lasso objective we have that, 
  \begin{align}
      & \frac{1}{n} \bX^\top (\bX^\top \hblaslambda-\by) + \lambda \bv = 0, \quad \bv \in \partial \left(\Vert \hblaslambda \Vert_1\right) \implies \\
      & \underbrace{(\bI-\hbSigma) (\hblaslambda-\bbeta_0)}_{\bDelta} + \underbrace{\frac{1}{n} \bX^\top \bepsilon}_{\bZ} - \lambda \bv = \hblaslambda -\bbeta_0 
  \end{align}

We can now use this relation to control the coordinate-wise Lasso bias,
\begin{align}
     & \E[\bI[\mS] (\hblaslambda - \bbeta_0)_i] = \E[ (\hblaslambda - \bbeta_0)_i \bI[\mS \cap \{ (\hblaslambda)_i >0 \} ] + \E[ (\hblaslambda - \bbeta_0)_i \bI[\mS \cap \{ (\hblaslambda)_i \leq 0 \} ] = \\
     &  \E[ (\bZ+\bDelta-\lambda \bv)_i \bI[\mS \cap \{ (\hblaslambda)_i >0 \} ] + \E[ (\hblaslambda - \bbeta_0)_i \bI[\mS \cap \{ (\hblaslambda)_i \leq 0 \} ] \leq \\
     & \E[\abs{\bZ_i}+\abs{\bDelta_i}] - \lambda \E[\bI[\mS \cap \{ (\hblaslambda)_i >0 \}] - (\bbeta_0)_i \E[\bI[\mS \cap \{ (\hblaslambda)_i \leq 0 \}] \leq \\
     & \E[\abs{\bZ_i}+\abs{\bDelta_i}] - \min(\lambda, (\bbeta_0)_i) \underbrace{\Pr[\mS]}_{\geq 1/2}.
\end{align}
At this point we fix the magnitude of $(\bbeta_0)_i =  \lambda$ for $i \in S$. 
We can now bound the expectations of our first two terms. For the first term $\bZ_i = \frac{1}{n} \be_i^\top \bX^\top \bepsilon$ where $\bepsilon \sim \mathcal{N}(0, \sigma_{\epsilon}^2 \bI_n)$ and $\bv= \bX \be_i \sim \mathcal{N}(0, \sigma_{\epsilon}^2 \bI_n)$ independently of $\bepsilon$. Thus,
\begin{align}
    \E[\abs{\bZ_i}] \leq \frac{1}{n} \sqrt{\E[(\bv^\top \bepsilon)^2]} = \frac{\sigma}{\sqrt{n}}.
\end{align}
For the second term,
\begin{align}
    \E[\abs{\bDelta_i}] \leq \sqrt{\E[\Vert (\hbSigma - \bI) \be_i \Vert_{\infty}^2} \sqrt{\E[\Vert \hbbeta_L -\bbeta_0 \Vert_1^2]}
\end{align}
From the proof of Lemma \ref{lem:const_conc}, with $\xstar = \be_i$ and $\bOmega=\bI$, we have that $\Pr[\Vert (\hbSigma - \bI) \be_i \Vert_{\infty} \geq t] \leq 2 p \cdot \exp(-\frac{n}{2} \min((\frac{t}{\kappa'})^2,\frac{t}{\kappa'}))$ where $\kappa' = 8$. Note for $n \geq (a/\kappa')^2 \log p$, $a \sqrt{\frac{\log p}{n}} \leq \kappa'$. Defining $A=\Vert (\hbSigma - \bI) \be_i \Vert_{\infty}$, 
\begin{align}
    & \E[A^2] = \int_0^{\infty} 2 t \Pr[A > t] \leq 4 \left[\int_0^{a \sqrt{\log p/n}} t \cdot 1 + \int_{a \sqrt{\log p/n}}^{\kappa'} p \cdot t \exp(-\frac{n}{2} \left(\frac{t}{\kappa'}\right)^2) + \int_{\kappa'}^{\infty} p \cdot t \exp \left(-\frac{n}{2} \frac{t}{\kappa'} \right) \right] \notag  \\
    & \leq 4 \left[\frac{a^2}{2} \frac{\log p}{n} + \frac{\kappa'^2 p^{1-\frac{a^2}{2\kappa'^2}}}{n} + \frac{2 \kappa'^2 e^{-n/2}(2+n) p}{n^2}\right] \leq \left(8 \kappa'^2+\frac{20 \kappa'^2}{p \log p \cdot n} \right) \frac{\log p}{n} \leq 9 \kappa'^2 \frac{\log p}{n}. 
\end{align}
where the last sequence of inequalities follows by choosing $a=2\kappa'$, assuming $n \geq \max\{4 \log p, 2\}$, and then assuming $p \geq 20$. Using Lemma \ref{lem:lasso_consist_g} and \ref{lem:uncond_exp} we have that, 
\begin{align}
    \E[\Vert \hblaslambda -\bbeta_0 \Vert_1^2]  \leq  \left(\frac{49 \lambda s_{\bbeta_0}}{4}  \right)^2 + \left(\frac{49}{8} \frac{(8+2 \sqrt{2})\sigma_{\epsilon}}{ \sqrt{n}} \right)^2 + \left( \frac{\sigma_{\epsilon}^4}{\lambda_*^2} + 2^4 s_{\bbeta_0}^2 \lambda^2 \right) \left(2e^{-c/2 \cdot n} \right)
\end{align}
using our choice of $\abs{(\bbeta_0)_i} = \lambda$ for each of the $s$ non-zero coordinates in $\bbeta_0$ (so $\normo{\bbeta_0} \leq s_{\bbeta_0} \lambda$). Here $\lambda_*$ is the lower bound on $\lambda$ from the Theorem statement. Under the assumption that $n \geq c_1 s_{\bbeta_0}^2 \log (2ep)$ and $p \geq 20$, there exists $c_1$ such that $\left( \frac{\sigma_{\epsilon}^4}{\lambda_*^2} + 2^4 s_{\bbeta_0}^2 \lambda^2 \right) \left(2e^{-c_2/2 \cdot n} \right) \leq (8+2\sqrt{2})^2 \sigma_{\epsilon}^2/n + 2^5 \lambda^2 s_{\bbeta_0}^2$. Once again using $p \geq 20$ and that $\lambda \geq \lambda_*$ we have that,
\begin{align}
    \E[\Vert \hblaslambda -\bbeta_0 \Vert_1^2]  \leq 300 \lambda^2 s_{\bbeta_0}^2.
\end{align}

Assembling, we conclude that,
\begin{align}
    & \E[\bI[\mS] (\hblaslambda - \bbeta_0)_i] \leq 
    \E[\abs{\bZ_i}+\abs{\bDelta_i}] - \min(\lambda, (\bbeta_0)_i) \underbrace{\Pr[\mS]}_{\geq 1/2} \leq \frac{\sigma_{\epsilon}}{\sqrt{n}} + 300 \lambda s_{\bbeta_0} \sqrt{\frac{\log(2ep)}{n}} - \frac{1}{2} \lambda \leq -\frac{2}{5} \lambda \notag.
\end{align}

The last inequality holds using that $\lambda \geq \lambda_*$ and $n \geq c_1 s_{\bbeta_0}^2 \log(2ep)$ for sufficiently large $c_1$.

 
 This allows us to conclude that $(\bv_{\star}^\top \E[\bI[\mS] (\hblaslambda - \bbeta_0)_S])^2 \geq c^2 \lambda^2 \trimmednorm{\bv_{\star}}{s}^2 \geq c^2 \lambda^2$.  Finally if we consider a spectral decomposition of $\E[\xstar \xstar^\top]_{S}$ we can conclude that, $\langle \E[\bI[\mS] (\hblaslambda - \bbeta_0)_S], \E[\xstar \xstar^\top]_{S} \E[\bI[\mS] (\hblaslambda - \bbeta_0)_S \rangle \geq \maxseig{\E[\xstar \xstar^\top]} (\bv_{\star}^\top (\E[\bI[\mS] (\hblaslambda - \bbeta_0)_S]])^2$, which yields the desired conclusion after combining with \cref{eq:lb_1}.
The final inequality in the display, $\maxseig{\E[\xstar \xstar^\top]} \geq  \trimmednorm{\E[\xstar]}{s}^2$ follows by Jensen's inequality and the variational characterization of the $s$-sparse eigenvalues. The claim for fixed deterministic $\xstar$ follows immediately from this result.

To show tightness of the upper bound for deterministic $\xstar$, we first apply the Holder inequality on the top-s norm and its dual (see \cref{prop:dual_k}) to see that, 
\begin{align}
    \E[\langle \xstar, \hblaslambda-\bbeta_0 \rangle^2] \leq \trimmednorm{\xstar}{s}^2 \E \left[ \max \left(\frac{\normo{\hblaslambda-\bbeta_0}}{s_{\bbeta_0}}, \normi{\hblaslambda-\bbeta_0} \right)^2 \right]
\end{align}
Since for $a, b  \geq0$ , $\max(a,b)^2 \leq 2(a^2+b^2)$ it suffices to bound the expectation of each term individually. From the previous computations we recall that $\E[\Vert \hblaslambda -\bbeta_0 \Vert_1^2]  \leq 300 \lambda^2 s_{\bbeta_0}^2$. Finally by appealing to \cref{lem:normibound} and similar computations to before, we have that,
\begin{align}
    & \E[\normi{\hblaslambda-\bbeta_0}^2] \leq 30 \left(\E[(\normi{\bX^\top \bepsilon}/n)^2] + \sqrt{\E[\normi{\bSigma_n-\bI_d}^4]} \sqrt{\E[\normo{\hblaslambda-\bbeta_0}^4]} + (\frac{\lambda}{2})^2 \right) \leq \\
    & O(\lambda_*^2) + O( (\sqrt{\log(2ep)/n} \cdot \lambda s_{\bbeta_0})^2) + O(\lambda^2) \leq O( \lambda^2),
\end{align}
using once again that $\lambda \geq \lambda_*$ and that $n \geq c_1 s_{\bbeta_0}^2 \log(2ep)$ for sufficiently large $c_1$. Recall we define $\bZ_i = \frac{1}{n} \be_i^\top \bX^\top \bepsilon$ where $\bepsilon \sim \mathcal{N}(0, \sigma_{\epsilon}^2 \bI_n)$ and $\bv= \bX \be_i \sim \mathcal{N}(0, \sigma_{\epsilon}^2 \bI_n)$ independently of $\bepsilon$. Hence appealing to \cref{lem:prod_conc} and using a union bound, 
\begin{align}
    \Pr[\max_i \abs{\bZ}_i \geq t] \leq 2p 
    \exp(-\frac{n}{2} \min((t/\kappa)^2, t/\kappa)) \implies \E[(\normi{\bX^\top \bepsilon}/n)^2] \leq O \left( \left(\sigma_{\epsilon} \sqrt{\frac{\log p}{n}} \right)^2 \right)\leq O( \lambda^2)
\end{align}
for $\kappa=8 \sigma_{\epsilon}^2$ by integrating the tail bound using similar computations to before when $n \geq c_1 \log p$ for large-enough constant $c_1$. Combining these results shows that,
\begin{align}
    \E[\langle \xstar, \hblambda-\bbeta_0 \rangle^2] \leq c_3 \trimmednorm{\xstar}{s}^2 \lambda^2
\end{align}
for some large-enough $c_3$.
\end{proof}

\subsection{\cref{cor:lb_lqball} and Supporting Lemmas}
\label{sec:app_lb_lqball}
We now provide a short proof of the supporting corollary.
\begin{proof}[Proof of \cref{cor:lb_lqball}]
This follows from \cref{thm:lasso_lb} since for a fixed $\xstar$ we have that $\E[\xstar]=\xstar$ and $\sup_{ \norm{\xstar}_q = 1} \norm{\xstar}_{(s)}^2 \geq s^{2-2/q}$.
\end{proof}
The construction of this lower bound utilizes a support recovery result which requires the following conditions on the sample design matrix $\bX \in \mathbb{R}^{n \times p}$,

\begin{condition}
    (Lower Eigenvalue on Support). The smallest eigenvalue of the sample covariance sub-matrix indexed by $S$ is bounded below:
        \begin{align}
            \sigma_{\min} \left( \frac{\bX_{S}^\top \bX_{S}}{n} \right) \geq c_{\min} > 0
        \end{align}
        \label{cond:lower_eig}
\end{condition}

\begin{condition}
    (Mutual Incoherence). There exists some $\alpha \in [0,1)$ such that
        \begin{align}
            \max_{j \in S^c} \norm{(\bX_{S}^\top \bX_{S})^{-1} \bX_{S}^\top \bX \be_j}_1 \leq \alpha
        \end{align}
        \label{cond:mut_incoh}
\end{condition}

\begin{condition}
    (Column Normalization). There exists some $C$ such that
        \begin{align}
            \max_{j=1, \hdots, p} \norm{\bX \be_j}_2/\sqrt{n} \leq C
        \end{align}
        \label{cond:col_norm}
\end{condition}
Importantly all of these conditions can be verified w.h.p when $n \gtrsim s_{\bbeta_0} \log p$ for covariates $\bx_i \sim \mathcal{N}(0, \bI_p)$ using standard matrix concentration arguments. To state our first lower bound it is also convenient to define $\Pi_{S^\bot}(\bX) = \bI_n-\bX_{S}(\bX_{S}^\top \bX_{S})^{-1} \bX_{S}^\top$, which is a type of orthogonal projection matrix.

Given these conditions we can state a conditional (on $\bX$) support recovery result,
\begin{proposition}
    \label{prop:cond_supp_rec}
    Let Conditions \eqref{cond:lower_eig}, \eqref{cond:mut_incoh} and \eqref{cond:col_norm} hold for the sample covariance matrix $\bX$, the independent noise distribution be Gaussian, $\bepsilon \sim \mN(0, \bI_n \sigma_{\epsilon}^2)$, and \cref{assump:well_spec} hold (with $s_{\bbeta_0}$-sparse underlying parameter $\bbeta_0$). Then, for any choice of regularization parameter $\lambda = \frac{2 C \sigma}{1-\alpha} \sqrt{\frac{2 \log(p-s_{\bbeta_0})}{n}} + \delta$ for $\delta > 0$, the support of $\hblaslambda$ is strictly contained in the support of $\bbeta_0$:
    \begin{align}
        S_{\hblaslambda} \subseteq S_{\bbeta_0}
    \end{align}
    with probability at least $1-4e^{-n\delta^2/2}$.
\end{proposition}

\begin{proof}
Conditions \eqref{cond:lower_eig} and \eqref{cond:mut_incoh}, and the fact that $\lambda \geq \frac{2}{1-\alpha} \norm{\bX^\top_{S^c} \Pi_{S^\bot}(\bX) \frac{\bepsilon}{n}}_{\infty}$ are sufficient show a support recovery result. Under these conditions, for all $s$-sparse $\bbeta_0$, there is a unique optimal solution to the Lagrangian Lasso program $\hblaslambda$ and the support of $\hblaslambda$, $S_{\hblaslambda}$, is contained within the support $S_{\bbeta_0}$ (no false inclusion property) \citep[Theorem 7.21]{wainwright2017highdim}. We can simplify the condition on the regularization parameter from Proposition \citep[Theorem 7.21]{wainwright2017highdim} using a standard union bound/Gaussian tail bound argument (using Assumption \ref{assump:noise1}) along with the column normalization condition (Condition \eqref{cond:col_norm}) to show that $\lambda = \frac{2 C \sigma}{1-\alpha} \sqrt{\frac{2 \log(p-s_{\bbeta_0})}{n}} + \delta$ satisfies $\lambda \geq \frac{2}{1-\alpha} \norm{\bX^\top_{S^c} \Pi_{S^\bot}(\bX) \frac{\bepsilon}{n}}_{\infty}$ with probability at least $1-4e^{-n\delta^2/2}$ (over the randomness in $\epsilon$) \citep[Corollary 7.22]{wainwright2017highdim}. Combining yields the desired conclusion.
\end{proof}

The aforementioned result holds conditional on $\bX$. However, we can verify that Conditions,
\eqref{cond:lower_eig}, \eqref{cond:mut_incoh}, \eqref{cond:col_norm} hold true w.h.p. even if we sample $\bx_i \sim \mathcal{N}(0, \bI_p)$ (see Lemma \ref{lem:design_conds}). Thus, we can show a Lasso prediction error bound that holds in expectation over all the randomness in the training data $(\bX, \bepsilon)$.

To do so we introduce the following standard result showing Conditions
\eqref{cond:lower_eig}, \eqref{cond:mut_incoh}, \eqref{cond:col_norm} can be verified w.h.p. for i.i.d. covariates from $\mathcal{N}(0, \bI_p)$. 
\begin{lemma}
    \label{lem:design_conds}
    Let $\bx_i \distiid \mN(0, \bI_p)$ for $i \in [n]$. Then there exists a universal constant $c_2$, such that for $n \geq c_2 s_{\bbeta_0} \log p$ and $p \geq 20$, Conditions \ref{cond:lower_eig}, \ref{cond:mut_incoh}, \ref{cond:col_norm} each hold with probability at least $\frac{99}{100}$.
\end{lemma}
\begin{proof}
    The proofs of these follow by standard matrix concentration arguments. Condition \eqref{cond:col_norm} can be verified w.h.p. for $C=1$ (as a function of $n$)  identically to Lemma \ref{lem:column_norm} for $n \gtrsim \log p$. Condition \eqref{cond:mut_incoh} can also be verified w.h.p. for $\alpha=\frac{1}{2}$ for $n \gtrsim s_{\bbeta_0} \log(p-s_{\bbeta_0})$, see for  example \citep[Ch.7, p.221, Exercise 19]{wainwright2017highdim}. While finally, Condition \eqref{cond:lower_eig} can also be verified w.h.p. for $c_{\min}=\frac{1}{2}$ when $n \gtrsim s_{\bbeta_0}$ using standard operator norm bounds for Gaussian ensembles (see for example, \citep[Theorem 6.1, Example 6.3]{wainwright2017highdim}). 
\end{proof}
    Combining Lemma \ref{lem:design_conds} and Proposition \ref{prop:cond_supp_rec} yields the desired conclusion which we formalize below.
\begin{proposition}\label{prop:uncond_supp_rec}
    Under \cref{assump:well_spec},
    suppose $\bx_i \distiid \mN(0, \bI_p)$ with independent noise $\bepsilon \sim \mN(0, \bI_n \sigma_{\epsilon}^2)$. Then, if $\hblaslambda$ denotes the solution of the Lasso program,
    with regularization parameter chosen as $\lambda \geq 8 \sigma_{\epsilon} \sqrt{\log p/n}$, there exists a 
    universal constant $c_1$ such that for all $n \geq c_1 s_{\bbeta_0} \log p$,
    \begin{align}
        S_{\hblaslambda} \subseteq S_{\bbeta_0}
    \end{align}
    with probability at least $\frac{1}{2}$.
\end{proposition}
\begin{proof}
    The proof follows using the independence of $\bepsilon$ and $\bX$, by combining the results of Proposition \ref{prop:cond_supp_rec} and Lemma \ref{lem:design_conds} with a union bound (and taking $n$ sufficiently large).
\end{proof}

We next state a useful supremum norm bound applicable to the Lasso under random design from \citet{geer2014statistical},
\begin{lemma}[Lemma 2.5.1 in \citet{geer2014statistical}]
\label{lem:normibound}
    Under \cref{assump:well_spec}, if $\hblaslambda$ denotes the solution of the Lasso program,
    with regularization parameter chosen as $\lambda$,
    \begin{align}
        \normi{\hblaslambda - \bbeta_0} \leq \normi{\bOmega \bX^\top \bepsilon}/n + \normo{\bOmega} \left(\normi{\bSigma_n -\bI_d} \normo{\hblaslambda-\bbeta_0} + \frac{\lambda}{2} \right)
    \end{align}
    for $\bOmega = \bSigma^{-1}$.
\end{lemma}

Finally, we state a useful (and standard fact) from convex analysis.
\begin{proposition}
    \label{prop:dual_k}
    If $\trimmednorm{\bx}{k}$ denotes the top-$k$ norm, the sum of the magnitudes of the $s$ largest magnitude entries of $\bx$, then its dual norm is $\norm{\bx}_{(k), *} = \max(\normo{\bx}/k, \normi{\bx})$.
\end{proposition}

\section{Proofs for \cref{sec:jm_ub}: Javanmard-Montanari (JM)-style Estimator}
In this section we provide the proof of the prediction risk bounds for the JM-style estimator. 

\subsection{\cref{thm:ub_jm}}
\label{sec:app_ub_jm}
We provide the proof of \cref{thm:ub_jm}.

\begin{proof}[Proof of \cref{thm:ub_jm}]
Recall that we will use $\brateo = (\E_{\bX, \bepsilon}[\Vert \hat{\bbeta}-\bbeta_0 \Vert_1^4])^{1/4}$.
This estimator admits the error decomposition,
\begin{align}
    \yjm - 
    \langle \xstar, \bbeta_0 \rangle = \frac{1}{n}\bw^\top \bX^\top \bepsilon + \langle \xstar-\bSigma_n \bw , \hat{\bbeta}-\bbeta_0 \rangle
\end{align}
and hence,
\begin{align}
    \E_{\bX, \bepsilon}[(\yjm - 
    \langle \xstar, \bbeta_0 \rangle)^2] \leq 2 \left(\E_{\bX, \bepsilon}[(\frac{1}{n}\bw^\top \bX^\top \bepsilon)^2]+\E_{\bX, \bepsilon}[\langle \xstar-\bSigma_n \bw , \hat{\bbeta}-\bbeta_0 \rangle^2] \right)
\end{align}
The first term can be thought of as the variance contribution while the second is the contribution due to bias. 
For the variance term, we begin by evaluating the expectation over $\bepsilon$. Using independence (w.r.t. to $\bX$) and sub-gaussianity of $\bepsilon$,
    \begin{align}
        \E_{\bX, \bepsilon}[(\frac{1}{n}\bw^\top \bX^\top \bepsilon)^2] = \frac{1}{n} \E_{\bX} \E_{\bepsilon}[ (\sum_{i=1}^n \bw^\top \bx_i \epsilon_i)^2 | \bX] = \frac{\sigma_{\epsilon}^2}{n} \E_{\bX} [\bw^\top \bSigma_n \bw]
    \end{align}
    Now using \cref{cor:jm_program_moments} and defining $\kappa_1' = 8 \kappa^2/\Cmin \normt{\xstar}^2$ we have that,
\begin{align}
    \E_{\bX} [\bw^\top \bSigma_n \bw] \leq \xstar^\top \bOmega \xstar + \frac{3 \kappa_1'}{\sqrt{n}}.
\end{align}
using the condition $n \geq 2$.  Turning to the bias term, the Holder and Cauchy-Schwarz inequalities give, 
$\E_{\bX, \bepsilon} [
\langle \xstar-\hbSigma \bw, \hat{\bbeta} - \bbeta_0 \rangle^2] \leq \E_{\bX, \bepsilon} [
\normi{ \xstar-\hbSigma \bw}^2 \normo{\hat{\bbeta} - \bbeta_0}^2] \leq \sqrt{\E_{\bX}[\normi{\xstar-\hbSigma \bw}^4]\E_{\bX, \bepsilon}[\normo{ \hat{\bbeta} - \bbeta_0}^4]}$.

We begin by evaluating the first expectation $\E_{\bX}[\normi{ \xstar-\hbSigma \bw}^4]$ which follows from Corollary \ref{cor:jm_program_moments},
\begin{align}
    & \sqrt{\E_{\bX}[\normi{ \xstar-\hbSigma \bw}^4]} \leq  \lambda_{\bw}^2 + \sqrt{2} \normi{\xstar}^2 (p \vee n)^{-c_3}
\end{align}
for $n \geq a^2 \log (p \vee n)$ and $c_3=a^2/4-\frac{1}{2}$ with $\kappa'_2 = 8 \kappa^2 \sqrt{\cond} \normt{\xstar}$. By definition of the base estimation procedure we can assemble to obtain the desired error is bounded by,
\begin{align}
    & \leq O( \frac{\sigma_{\epsilon}^2 \xstar \bOmega \xstar}{n} + \frac{\sigma_{\epsilon}^2 \kappa_1'}{n^{3/2}} + \brateo^2 ( (\lambda_{\bw}^2+\normi{\xstar}^2 (p \vee n)^{-c_3})) 
\end{align}
where $\lambda_{\bw} = a k^2 \kappa_2' \sqrt{\frac{\log(p \vee n)}{n}}$.

For the second claim note by \cref{cor:jm_program_moments}, that $\bw=0$ and hence we can write the error of the estimator as,
\begin{align}
    \yjm - 
    \langle \xstar, \bbeta_0 \rangle = \langle \xstar, \hbbeta-\bbeta_0 \rangle \implies \E_{\bX, \epsilon} [(\yjm - 
    \langle \xstar, \bbeta_0 \rangle)^2] = \E_{\bX, \bepsilon}[\langle \xstar, \hbbeta-\bbeta_0 \rangle^2].
\end{align}
\end{proof}

We can now instantiate the result of the previous theorem in the setting where the Lasso estimator is used as the base-regression procedure. 

\subsection{\cref{prop:CI}}
\label{sec:app_CI}

We now connect our results to the problem of constructing CIs in sparse linear regression -- namely the results in \citet{cai2017confidence}. We first define formally what it means for a set $S$ to be a $1-\alpha$ CI in this context -- namely that for all $\bbeta_0$, $\lim \inf_{n, p \to \infty} \Pr_{\bbeta_0}[\xstar^\top \bbeta_0 \in S] \geq 1-\alpha$.

\begin{proof}[Proof of \cref{prop:CI}]

Before beginning, we first recall the tail bound in \citet[Theorem 4.2]{bellec2016slope},
    which provides that, 
    \begin{align}
        \Vert \hblaslambda-\bbeta_0 \Vert_q \leq \frac{49}{8} \left (\frac{\log(1/\delta_0)}{s \log(1/\delta(\lambda))} \vee \frac{1}{\phi^2_0} \right) \lambda s^{1/q}
    \end{align}
    with probability at least $1-\delta_0/2$,
    where $\delta(\lambda) = \exp(-(\frac{\lambda \sqrt{n}}{(8+2\sqrt{2}) \sigma}))$ for all design matrices in $\bX \in \mE_{n}(s, 7)$ where $\phi_0^2 = \phi^2_{SRE}(s, 7)$. Note by \cref{thm:design_SRE} we have that under our design assumptions $\bX \in \mE_n(s, 7)$ with probability at least $1-3 \exp(-cn/\kappa^4)$ for $n \gtrsim s \log p$. Hence taking $q=1$ and $\lambda \asymp  \sqrt{\log p/n}$, $s \log(1/\delta(\lambda)) \asymp s \exp(-c \sqrt{\log p}) \asymp \exp(\gamma \log p -c \sqrt{\log p})$ for $0 \leq \gamma < \frac{1}{2}$. Hence, for $\delta_0 \asymp p^{-\gamma/2}$, $\frac{\log(1/\delta_0)}{s \log(1/\delta(\lambda))} \to 0$. Accordingly, for sufficiently large $p$, we \begin{align}
        \Vert \hblaslambda-\bbeta_0 \Vert_1 \leq K_1 s\sqrt{\frac{\log p}{n}} 
    \end{align}
    with probability at least $1- O(\exp(-c n))-O(p^{-\gamma/2})$. Define the set $S_1 = [\xstar^\top \hblaslambda + \normi{\xstar} K s_{\bbeta_0} \sqrt{\frac{\log p}{n}}, \xstar^\top \hblaslambda - \normi{\xstar} K s_{\bbeta_0} \sqrt{\frac{\log p}{n}}]$ for future reference.
    
     In the case of the dense loading regime we have that $\frac{\normi{\xstar}}{\normt{\xstar}} \asymp p^{-\gamma_q/2}$, and take $\lambda_{\bw} = 8 \sqrt{\cond} \kappa^2 \frac{1}{s \sqrt{\log p}}\normt{\xstar}$. This choice of satisfies $\lambda_{\bw} \asymp \underbrace{\frac{p^{\gamma_q/2-\gamma}}{\sqrt{\log p}}}_{\to \infty} \normi{\xstar}$. Hence by the definition of the JM program, for sufficiently large $p$, its minimizer is $\bw = 0$ almost surely as argued in the proof of \cref{thm:ub_jm}  -- in which case $\yjm = \xstar^\top \hblaslambda$ almost surely. Hence in this regime, $S_1 = [\yjm + \normi{\xstar} K_1 s_{\bbeta_0} \sqrt{\frac{\log p}{n}}, \yjm - \normi{\xstar} K_1 s_{\bbeta_0} \sqrt{\frac{\log p}{n}}]$ provides valid coverage by the previous arguments.
     
     To show the second claim consider the set $S_2 = [\yjm + 1.01/\sqrt{n} z_{\alpha/2} \normt{\xstar} \sqrt{\bw^\top \bSigma_n \bw} + \sqrt{n}, \yjm + 1.01/\sqrt{n} z_{\alpha/2} \normt{\xstar} \sqrt{\bw^\top \bSigma_n \bw} + K_2/\sqrt{n}]$, and note
     fom the proof of \cref{thm:ub_jm} we can see 
    \begin{align}
    \abs{\yjm - 
    \langle \xstar, \bbeta_0 \rangle} = \abs{\frac{1}{n}\bw^\top \bX^\top \bepsilon + \langle \xstar-\bSigma_n \bw , \hat{\bbeta}-\bbeta_0 \rangle}
    \end{align}
    using the results in therein that $\abs{\langle \xstar-\bSigma_n \bw , \hat{\bbeta}-\bbeta_0 \rangle} \leq \normo{\hat{\bbeta}-\bbeta_0} \normi{ \xstar-\bSigma_n \bw} \lesssim s \sqrt{\log p/n} \cdot \normt{\xstar} \frac{1}{s \sqrt{\log p}} \leq K_2 \normt{\xstar} \frac{1}{\sqrt{n}}$ with probability at least $1-O(\exp(-cn)-O(p^{-1})$ with the aforementioned choice of $a$ in the regime $n \gtrsim s^2 (\log p)^2$ (which implies $a \gtrsim 1$). 
    Conditionally on $\bX$ we then have that $\frac{1}{n} \bw^\top \bX^\top \bepsilon \sim \mathcal{N}(0, \frac{1}{n} \bw^\top \bSigma_n \bw)$.
    Combining these results with a union bound show thats $\lim \inf_{n, p} \Pr[\xstar^\top \bbeta_0 \in S_2] \to 1-\alpha$ with as $n, p \to \infty$. Finally, since by  \cref{lem:obj_conc} we have that $\sqrt{\bw^\top \bSigma_n \bw} \leq 1.01 \sqrt{\xstar^\top \bOmega \xstar}$ with probability at least $1-\exp(-c n)$, and \cref{cor:jm_program_moments}, we $\E[\bw^\top \bSigma_n \bw] \leq \xstar \bOmega \xstar + O(\frac{1}{\sqrt{n}})$ we can see that in the regime $n \gtrsim s^2 (\log p)^2$ the interval $S_2$ indeed has expected length $O(\frac{\normt{\xstar}}{\sqrt{n}})$ which is optimal in this regime.
\end{proof}

\subsection{\cref{cor:ub_jm_lasso} and Supporting Lemmas}
\label{sec:app_ub_jm_lasso}
We provide the proof of the \cref{cor:ub_jm_lasso}.
\begin{proof}[Proof of \cref{cor:ub_jm_lasso}]
The second expectation $\E_{\bX, \bepsilon}[\Vert \hblaslambda-\bbeta_0 \Vert_1^4]$ can be evaluated using Lemmas \ref{lem:lasso_upper_sg} and \ref{lem:uncond_exp} from which we find,
\begin{align}
    r_{\beta,1}^2 = \sqrt{\E_{\bX, \bepsilon}[\Vert \hblaslambda - \bbeta_0 \Vert_1^4]} \leq  O\left(\frac{\lambda_{\bbeta_0} s_{\bbeta_0}}{C_{\min}}  \right)^2 + O\left(\frac{\sigma_{\epsilon}}{\sqrt{n}} \right)^2  + O(\left(\frac{\sigma_{\epsilon}^{4}}{\lambda_{\bbeta}^{2}} + \Vert \bbeta_0 \Vert_1^{2} \right)\left(e^{-\frac{c}{4 \kappa^4} n}\right))
\end{align}
Assuming $p \geq 20$ and $n \geq c_2 \frac{\kappa^4}{C_{\min}} s \log(2ep)$, there exists sufficiently large $c_2$ such that 
$\left(\frac{\sigma_{\epsilon}^{4}}{\lambda_{\bbeta}^{2}} + \Vert \bbeta_0 \Vert_1^{2} \right)\left(\sqrt{2} e^{-n \frac{c}{4\kappa^4}}\right) \leq O(\frac{\sigma_{\epsilon}^2}{n} + \normo{\bbeta_0}^2 e^{-n c/(4 \kappa^4)}) \leq O(\frac{\sigma_{\epsilon}^2}{n})$ since $\normi{\bbeta_0}/\sigma_{\epsilon} = o(e^{c_1 n})$ for some sufficiently small $c_1$. Thus we have $r_{\beta, 1}^2 \leq O \left( \frac{\lambda_{\bbeta}^2 s_{\bbeta_0}^2}{\Cmin^2} + \frac{\sigma_{\epsilon}^2}{n} \right) = O(\frac{\lambda_{\bbeta}^2 s_{\bbeta_0}^2}{\Cmin^2})$ due to the lower bound on $\lambda_{\bbeta}$. Combining with \cref{thm:ub_jm} gives the result,
\begin{talign}
    O\left( \frac{\sigma_{\epsilon}^2 \xstar \bOmega \xstar}{n} + \left( \frac{\lambda_{\bbeta}^2 s_{\bbeta_0}^2}{\Cmin^2} \right) (\lambda_{\bw}^2 + \normi{\xstar}^2 (p \vee n)^{-c_3}) \right)
\end{talign}
\end{proof}

Here we collect several useful lemmas which follow from standard concentration arguments useful both in the analysis of the upper bound on the JM estimator and in the Lasso lower bound.

To begin we show the convex program defining the JM estimator is feasible with high probability. For convenience we define the event $\mathcal{F}(a)$ to be the event that the convex program defining the JM estimator in \cref{eq:jm_program} with choice of regularization parameter $\lambda_{\bw} = a \sqrt{\log p/n}$ possesses $\bw_0 =\bOmega \bx_\star$ as a feasible point.

\begin{lemma}
    Let Assumption \ref{assump:cov} and \ref{assump:design} hold for the design $\bX$ and assume $n \geq a^2 \log (p \vee n)$ with $\kappa_2' = 8 \kappa^2 \sqrt{\cond} \normt{\xstar}$. If $\xstar \in \mR^p$ then for $\bw_0 = \bOmega \xstar$,
    \begin{align}
        \Pr[\normi{\hbSigma \bw_0 - \xstar } \geq a  \kappa_2' \sqrt{\log (p \vee n)/n}] \leq 2(p \vee n)^{-c_2}
    \end{align}
    for $c_2 = \frac{a^2}{2}-1$. Hence the convex program in \cref{eq:jm_program} with regularization parameter $\lambda_{\bw} = a \kappa_2' \sqrt{\frac{\log (p \vee n)}{n}}$ admits $\bw_0$ as a feasible point with probability at least $1-2(p \vee n)^{-c_2}$.
    \label{lem:const_conc}
\end{lemma}
\begin{proof}
    This follows from a standard concentration argument for sub-exponential random variables. Throughout we will use $\tilde{\bx}_{\ell} = \bOmega^{1/2} \bx_{\ell}$. Consider some $j \in [p]$ and define $z_{\ell}^{j}= \be_j^\top \bOmega^{1/2} \tilde{\bx}_\ell \cdot \tilde{\bx}_\ell^\top \bSigma^{1/2} \xstar  -  \be_j^\top \xstar$ which satisfies $\E[z^{j}_{\ell}]=0$, is independent over $\ell \in [n]$, and for which $\be_j^\top (\hbSigma \bw_0 - \xstar) = \frac{1}{n} \sum_{j=1}^n z_j^{\ell}$.  Since $\be_j^\top \bOmega^{1/2} \tilde{\bx}_\ell  \sim \sG(\kappa \Vert \be_j^\top \bOmega^{1/2} \Vert_2) \sim \sG(\kappa /\sqrt{C_{\min}})$, and $(\xstar)^\top \bSigma^{1/2} \tilde{\bx}_\ell  \sim \sG(\kappa \Vert  \bSigma^{1/2} \xstar \Vert_2) \sim \sG(\kappa \sqrt{C_{\max}} \normt{\xstar})$, $z^j_\ell$ is a mean-zero $\sE(8 \kappa^2 \sqrt{\cond} \normt{\xstar},  8 \kappa^2 \sqrt{\cond} \normt{\xstar})$ r.v. by Lemma \ref{lem:prod_conc}.
    Defining $\kappa_2'=8 \kappa^2 \sqrt{C_{\max}/C_{\min}} \normt{\xstar}$, applying the tail bound for sub-exponential random variables, and taking a union bound over the $p$ coordinates implies that,
    \begin{align}
        \Pr[\normi{\bSigma_n \bw_0 - \xstar} \geq t] \leq \Pr[\norm{\bSigma_n \bw_0 - \xstar}_{\infty} \geq t] \leq 2 p \exp [-\frac{n}{2} \min( (t/\kappa_2')^2, t/\kappa_2'))].
    \end{align}
    Choosing $t = a \kappa_2' \sqrt{\log (p \vee n)/n}$, assuming $n \geq a^2 \log (p \vee n)$, gives the conclusion 
    \begin{align}
      \Pr[\normi{\bSigma_n \bw_0 - \xstar} \geq a  \kappa_2' \sqrt{\log (p \vee n)/n}] \leq 2 (p \vee n)^{-a^2/2+1}
    \end{align}
    and the conclusion follows.
\end{proof}

We can now provide a similar concentration argument to bound the objective of the JM program. 

\begin{lemma}
    Let Assumption \ref{assump:cov} and \ref{assump:design} hold for the design $\bX$. Let $\bw$ be the solution of the convex program in \cref{eq:jm_program} with regularization parameter set as $\lambda_{\bw}$. If $\xstar \in \mR^p$, then,
    \begin{align}
        \Pr[\bw^\top \bSigma_n \bw \geq \xstar \bOmega \xstar + t ] \leq 2 \exp [-n/2 \min( (t/\kappa_1')^2, t/\kappa_1')]
    \end{align}
    \label{lem:obj_conc}
    for $\kappa_1'^2 = 8 \kappa^2/C_{\min} \normt{\xstar}^2$.
\end{lemma}
\begin{proof}
    The argument once again follows from a  standard concentration argument for sub-exponential random variables.
    Considering, \begin{align}
        & (\xstar \bOmega)^\top \bSigma_n \bOmega \xstar =  [(\xstar \bOmega)^\top \bSigma_n \bOmega \xstar - \xstar \bOmega \xstar]  + \xstar \bOmega \xstar = \frac{1}{n} \sum_{j=1}^{n} (z_j^2 - \xstar \bOmega \xstar) + \xstar \bOmega \xstar
    \end{align}
    where $z_j = \xstar^\top \bOmega \bx_j$ is mean-zero with $s_j \sim \sG(\kappa \Vert \bOmega^{1/2} \xstar \Vert_2) \sim \sG(\kappa/\sqrt{C_{\min}} \normt{\xstar})$. Since  $\E[z_j^2] = \xstar \bOmega \xstar$, Lemma \ref{lem:prod_conc} implies $z_j^2-\xstar \bOmega \xstar \sim \sE(8 \kappa^2/C_{\min} \normt{\xstar}^2, 8 \kappa^2/C_{\min} \normt{\xstar}^2)$ and is mean-zero.
    The sub-exponential tail bound gives,
    \begin{align}
        \Pr[\frac{1}{n} \sum_{\ell=j}^{n} z_j^2 \geq \xstar \bOmega \xstar + t] \leq \exp [-n/2 \min( (t/\kappa_1')^2, t/\kappa_1')] \notag
    \end{align}
    where $\kappa_1' = 8 \kappa^2/C_{\min} \normt{\xstar}^2$.
    Hence, since on the event $\mF(a)$, we have that $\bw^\top \bSigma_n \bw \leq  (\xstar \bOmega)^\top \bSigma_n \bOmega \xstar$ (recall $\bw_0 = \bOmega \xstar$ is feasible on $\mF(a)$),
    \begin{align}
        & \Pr[\bw^\top \bSigma_n \bw \geq \xstar \bOmega \xstar +t] \leq \Pr[\{ \bw^\top \bSigma_n \bw \geq \xstar \bOmega \xstar +t \} \cap \mF(a)] +\Pr[ \{ \bw^\top \bSigma_n \bw \geq \xstar \bOmega \xstar +t \} \cap \mF(a)^c] \notag  \\
        & \leq \Pr[\frac{1}{n} \sum_{\ell=j}^{n} z_j^2 \geq \xstar \bOmega \xstar + t] + 0 \leq \exp [-n/2 \min( (t/\kappa_1')^2, t/\kappa_1'))],  \notag
    \end{align}
    since by definition on the event $\mF(a)^c$ the convex program outputs $\bw=\mathbf{0}$ and $\xstar \bOmega \xstar \geq 1/C_{\max} > 0$.
\end{proof}

Finally we can easily convert these tail bounds into moment bounds,

\begin{corollary}
 Let Assumption \ref{assump:cov} and \ref{assump:design} hold for the design $\bX$. Let $\bw$ be the solution of the convex program in \cref{eq:jm_program} with regularization parameter set as $\lambda_{\bw} = a \kappa_2' \sqrt{\frac{\log (p \vee n)}{n}}$. If $\xstar \in \mR^p$, then,
    \begin{align}
        \E[\bw^\top \hbSigma \bw] \leq \xstar \bOmega \xstar + \frac{3 \kappa_1'}{\sqrt{n}}
    \end{align}
    for $\kappa_1' = 8 \kappa^2/C_{\min} \normt{\xstar}^2$ and assuming $n \geq a^2 \log (p \vee n)$,
    \begin{align}
         \sqrt{\E[\normi{\hbSigma \bw - \xstar}^4]}\leq \lambda_{\bw}^2 + \sqrt{2} \normi{\xstar}^2 (p \vee n)^{-c_2}
    \end{align}
    for $c_2 =a^2/4-1/2$ with $\kappa'_2 = 8 \kappa^2 \sqrt{\cond} \normt{\xstar}$. Moreover if $\lambda_{\bw} \geq \normi{\xstar}$ then $\bw = 0$ almost surely.
    \label{cor:jm_program_moments}
\end{corollary}
\begin{proof}
    Using Lemma \ref{lem:obj_conc} we have that, 
    \begin{align}
    & \E[\bw^\top \hbSigma \bw] = \xstar \bOmega \xstar + \int_0^{\infty} \Pr[\bw^\top \bSigma_n \bw \geq \xstar \bOmega \xstar + t ]  \\
    & \leq \xstar \bOmega \xstar + \int_{0}^{\kappa_1'} [\exp [-n/2  (t/ \kappa_1')^2] dt + \int_{\kappa_1'}^{\infty} [\exp [-n/2  (t/\kappa_1')] dt \\
    & \leq \xstar \bOmega \xstar + \frac{ 2 \kappa_1'}{\sqrt{n}} + \frac{2 \kappa_1' e^{-n/2}}{n} \leq \xstar \bOmega \xstar + \frac{3 \kappa'_1}{\sqrt{n}}
    \end{align}
    which holds for $n \geq 2$.
    
    Similarly, directly applying Lemma \ref{lem:const_conc} we obtain,
    \begin{align}
     & \E[\normi{\hbSigma \bw - \xstar }^4] = \E[\normi{\hbSigma \bw - \xstar }^4 \mone[\mF(a)]] + \E[\normi{\hbSigma \bw - \xstar }^4 \mone[\mF^{c}(a)]] \leq \\ & \lambda_{\bw}^4 + \normi{\xstar}^4 \Pr[\mF^c(a)] \leq \lambda_{\bw}^4 + 2 \normi{\xstar}^4 (p \vee n)^{-c_2}
    \end{align}
    $c_2 =a^2/2-1$, since the convex program outputs $\bw = 0$ on the event $\mathcal{F}^c(a)$. The first conclusion follows using subadditivity of $\sqrt{\cdot}$.
    
    For the second statement note the convex program in \cref{eq:jm_program} always admits $\bw=0$ as a feasible point under the condition $\lambda_{\bw} \geq \normi{\xstar}$, in which case $\bw = 0$ is a global minima of the objective since $\hbSigma$ is p.s.d.
\end{proof}

\section{Proofs for \cref{sec:om_ub}: Orthogonal Moment Estimators}
\label{sec:app_ub_om}
We begin by providing the consistency proofs for the orthogonal moment estimators introduced in \cref{sec:om_ub}. However, first we make a remark which relates the assumptions on the design we make to the properties of the noise variable $\eta$.

\begin{remark}
    Under the random design assumption on $\bx$, if we consider $\bx' = [t, \bz] = (\bU^{-1})^\top \bx$, then by Assumption \ref{assump:design},
    $\bg_0 = \arg \min_{g} \E_{\bX}[(t-\bz^\top \bg_0)^2]$ can be thought of as the best linear approximator interpreted in the regression framework. Hence it can also be related to the precision matrix and residual variance as:
    \begin{align}
        \bOmega_{t, \cdot} = \frac{(1, -\bg_0)}{\sigma_{\eta}^2}.
    \end{align}
    In this setting, we have that $\E[\eta^2] = \bSigma_{tt}-\bg_0^\top \bSigma_{\bz \bz} \bg_0 \geq 0$. Moreover from the variational characterization of the minimum eigenvalue we also have that $\E[\eta^2] \geq \Cmin/\normt{\xstar}$. Thus $\norm{\bg_0}_2^2 \leq \frac{\bSigma_{tt}}{C_{\min}}\leq \cond/\normt{\xstar}$ and $\E[\eta^2] \leq \bSigma_{tt} \leq \Cmax /\normt{\xstar}$. Moreover, the treatment noise $\eta$ is also a sub-Gaussian random variable, since $\eta = t-\bz^\top \bg_0 = (1, -\bg_0)^\top \bx'$. Recall by Assumption \ref{assump:cov} that $\E[(\bx^\top \bv)^p] \leq \kappa^{2p} \Vert \bSigma^{1/2} \bv \Vert_2^{2p}$ while $\eta = (1, -\bg_0)^\top \bx'$. Thus we have that $\E[\eta^{2p}] = \kappa^{2p} \Cmax^p (1+\normt{\bg_0}^2)^p/\normt{\xstar}^{2p} \leq O((\kappa
    ^2 \cond \Cmax/\normt{\xstar}^2)^p)$. Similarly $\E[(\bz^\top \bg_0)^{2p}] \leq (\kappa^2 \normt{\bg_0}^2 \Cmax/\normt{\xstar}^2)^p$. 
    \label{rmk:precision_rows}
\end{remark}

\subsection{\cref{thm:oml_ub}}
\label{sec:app_oml_ub}
We now present the Proof of \cref{thm:oml_ub}.

\begin{proof}[Proof of \cref{thm:oml_ub}]
To begin we rescale the $\xstar$ such that is has unit-norm (and restore the scaling in the final statement of the proof). In order to calculate the mean-squared error of our prediction $\E[(\yom-\xstar^\top \bbeta_0)^2]$, it is convenient to organize the calculation in an error expansion in terms of the moment function $m$. For convenience we define the following (held-out) prediction errors $\bDelta_{f}(\bz_i) = \bz_i^\top (\bbtf-\bbf_0)$, and $\bDelta_{g}(\bz_i) = \btg(\bz_i)-\bg_0(\bz_i)$ of $\bbf$ and $\bg( \cdot )$ which are trained on first-stage data but evaluated against the second-stage data. Note that as assumed in the Theorem, $\bg_0(\bz) = \bz^\top \bg_0$.
Also note the moment equations only depend on $\bbf$ and $\bg( \cdot )$ implicitly through the evaluations $\bz^\top \bbf$ and $\bg(\bz)$, so derivatives of the moment expressions with respect to $\bz^\top \bbf$ and $\bg(\bz)$, refer to derivatives with respect to scalar. Recall the sums of the empirical moment equation here only range over the second fold of data, while $\bbtf$ and $\btg$ are fit on the first fold. The empirical moment equations can be expanded (exactly) as,
    \begin{align}
        \underbrace{\frac{1}{n/2} \sum_{i=1}^{n/2} \nabla_{\theta} m(t_i, y_i, \theta_0, \bz_i^\top\bbtf, \btg(\bz_i))}_{J} (\theta_0-\yom) = \frac{1}{n/2} \sum_{i=1}^{n/2} m(t_i, y_i, \theta_0, \bz_i^\top \bbtf, \btg(\bz_i))
    \end{align}
    since by definition $\frac{1}{n/2} \sum_{i=1}^{n/2} m(t_i, y_i, \yom, \bz_i^\top \bbtf, \btg(\bz_i)) = 0$. Then we further have that, 
    \begin{align}
         & \frac{1}{n/2} \sum_{i=1}^{n/2} m(t_i, y_i, \theta_0, \bz_i^\top \bbtf, \btg(\bz_i)) = \underbrace{\frac{1}{n/2} \sum_{i=1}^{n/2} \nabla m(t_i, y_i, \theta_0, \bz_i^\top \bbf_0, \bg_0(\bz_i))}_{A} + \\ 
         & \underbrace{\frac{1}{n/2} \sum_{i=1}^{n/2} \nabla_{\bz^\top \bbf} m(t_i, y_i, \theta_0, \bz_i^\top \bbf_0, \bg_0(\bz_i))^\top(\bDelta_f)}_{B_1} + \\
         & \underbrace{\frac{1}{n/2} \sum_{i=1}^{n/2} \nabla_{\bg(\bz)} m(t_i, y_i, \theta_0, \bz_i^\top \bbf_0, \bg_0(\bz_i))^\top(\bDelta_g)}_{B_2} + \\
         & \underbrace{\frac{1}{n/2} \sum_{i=1}^{n/2} \nabla_{\bz^\top \bbf, \bg(\bz)} m(t_i, y_i, \theta_0, \bz_i^\top \bbf_0, \bg_0(\bz_i)) [\bDelta_f, \bDelta_g]}_{C} 
    \end{align}
We first turn to controlling the moments of $A, B_1, B_2, C$. We use as convenient shorthand $\zeta = \kappa^2 \Cmax$. Similarly we also use $\fratet = (\E[\bDelta_f(\bz)^4])^{1/4}$.
\begin{enumerate}
    \item For $A=\frac{1}{n/2} \sum_{i=1}^{n/2} \eta_i \epsilon_i$, note that $\E[m(t_i, y_i, \theta_0, \bz_i^\top \bbf_0, \bg_0(\bz_i)) | \bz_i ]=0$ so it follows that, 
    \begin{align}
        \E[A^2] = O(\frac{1}{n} \E[\eta^2 \epsilon^2]) = \frac{1}{n} \sigma_{\epsilon}^2 \sigma_{\eta}^2
    \end{align}
    \item For $B_1=\frac{1}{n/2} \sum_{i=1
}^{n/2} \bDelta_f(\bz_i) \eta_i$. Note $\E[\nabla_{\bz^\top \bbf} m(t_i, y_i, \theta_0, \bz_i^\top \bbf_0, \bg_0(\bz_i)) | \bz_i] = 0$ since $\E[\eta_i | \bz_i]=0$. So we have using sub-gaussianity of the random vector $\bz$,  sub-gaussianity of $\eta$ and independence that, 
    \begin{align}
        \E[B_1^2] = O(\frac{1}{n} \E[(\bDelta_f(\bz))^2 \eta^2]) \leq O(\frac{1}{n} \fratet^2 \sigma_{\eta}^2)
    \end{align}
    \item For $B_2 = \sum_{i=1}^{n} \bDelta_g(\bz_i) \epsilon_i$. Note $\E[\nabla_{\bg(\bz)} m(t_i, y_i, \theta_0, \bz_i^\top \bbf_0, \bg_0(\bz_i)) | \bz_i] = 0$ using independence of $\epsilon_i$ and the fact $\E[\epsilon_i]=0$. Once again using independence,
    \begin{align}
        \E[B_2^2] = \frac{1}{n} \E[\epsilon^2 (\bDelta_g(\bz))^2] \leq O(\frac{1}{n} \sigma_{\epsilon}^2 r_g^2)
    \end{align}
    \item For $C=\frac{1}{n}  \sum_{i=1}^n \bDelta_g(\bz_i) \bDelta_f(\bz_i)$. Note that in general for the remainder term $\E[\nabla_{\bz^\top \bbf, \bg(\bz)} m(t_i, y_i, \theta_0, \bz_i^\top \bbf_0, \bg_0(\bz_i)) | \bz_i] \neq 0$; however in some cases we can exploit unless we can exploit \textit{unconditional} orthogonality: $\E[\nabla_{\bz^\top \bbf, \bg(\bz)} m(t_i, y_i, \theta_0, \bz_i^\top \bbf_0, \bg_0(\bz_i))] = 0$ to obtain an improved rate although this is not mentioned in the main text. 
    \begin{itemize}
    \item In the absence of unconditional orthogonality, we have by the Cauchy-Schwarz inequality that,
    \begin{align}
        \E[C^2] \leq O(\sqrt{\E[(\bDelta_g(\bz))^4]} \sqrt{\E[ (\bDelta_f(\bz))^4]}) \leq O(\fratet^2 \gratet^2)
    \end{align}
    \item In the presence of unconditional orthogonality we have that,
    \begin{align}
        \E[C^2] = \frac{1}{n} \fratet^2 \gratet^2
    \end{align}
    as before using Cauchy-Schwarz but cancelling the cross-terms.
    \end{itemize}
\end{enumerate}

Now we can amalgamate our results. Before doing so, note that $\fratet \leq \zeta \bratet$ since in the description of the algorithm the estimator $\bf$ is defined by rotating an estimate of $\bbeta_0$ in the base regression procedure (and consistency of the (held-out) prediction error is preserved under orthogonal rotations).

First define the event $\mJ = \{ J \leq \frac{1}{4} \sigma_\eta^2 \}$. For the orthogonal estimator defined in the algorithm, on the event $\mJ$, the estimator will output the estimate from the first-stage base regression using $\yom = \xstar^\top \bbeta_0$. So introducing the indicator of this event, and using Cauchy-Schwarz, we have that, 
\begin{align}
    & \E[(\yom-\theta_0)^2] =   \left [\E[(\yom-\theta_0)^2 \mathbbm{1}(\mJ)] + \sqrt{\E[\norm{\bDelta_{\beta}(\xstar)}_2^4]}\sqrt{\Pr[\mJ]} \right] \\ 
    & \leq \left[ O(\frac{\E[A^2+B_1^2+B_2^2+C^2]}{(\sigma_{\eta}^2)^2}) + O(\bratet^2) \sqrt{O(((\frac{\xi}{\sigma_{\eta}^2})^4 + \frac{ \xi^2}{(\sigma_{\eta}^2)^4} \grate^4)  \cdot \frac{1}{n^2})} \right] \\
    & 
    \leq \normt{\xstar}^2 \Big [ O(\frac{\sigma_\eta^2 \sigma_\epsilon^2 + \zeta \bratet^2 \sigma_\eta^2 + \grate^2 \sigma_{\epsilon}^2}{(\sigma_{\eta}^2)^2 n}) + O((((\frac{\xi}{\sigma_{\eta}^2})^2 + \frac{ \xi}{(\sigma_{\eta}^2)^2} \gratet^2)  \cdot \frac{1}{n}) \cdot \bratet^2) \ + \\
    & \begin{cases}
    & O(\frac{\zeta^2 \bratet^2 \gratet^2}{(\sigma_{\eta}^2)^2 n}) \quad \text{with unconditional orthogonality} \\ 
    & O(\frac{\zeta^2 \bratet^2 \gratet^2}{(\sigma_{\eta}^2)^2} \quad \text{without unconditional orthogonality}
    \end{cases} 
\end{align}
where $\Pr[\mJ]$ is computed using \cref{lem:oml_threshold}. If we consider the case without unconditional orthogonality, and assume since $\Cmax \geq \sigma_{\eta}^2 \geq \Cmin$, the above results simplifies (ignoring conditioning-dependent factors) to the theorem statement,
\begin{align}
\normt{\xstar}^2 \left[ O(\frac{ \sigma_\epsilon^2}{\sigma_{\eta}^2 n}) + O(\frac{\bratet^2 \gratet^2}{(\sigma_{\eta}^2})^2)  + O(\frac{\bratet^2 \sigma_\eta^2 +   \grate^2 \sigma_{\epsilon}^2}{(\sigma_{\eta}^2)^2 n}) \right]
\end{align}
\end{proof}

\begin{lemma}\label{lem:oml_threshold}
 Let Assumptions \ref{assump:cov}, \ref{assump:design}, and \ref{assump:noise2} hold and suppose $\bg_0(\bz) = \bz^\top \bg_0$ in \cref{eq:oml_model}. Defining $J = \frac{1}{n} \sum_{i=1}^{n} J_i = \frac{1}{n} \sum_{i=1}^n t_i(t_i-\btg(\bz_i))$ as in the description of first-order OM estimator with $\tau \leq \frac{1}{4} \E[\eta^2]$, then,
 \begin{align}
     \Pr[\frac{1}{n} \sum_{i=1}^n J_i \leq \tau] \leq O(((\frac{\xi}{\sigma_{\eta}^2})^4 + \frac{ \xi^2}{(\sigma_{\eta}^2)^4} \gratet^4)  \cdot \frac{1}{n^2})
 \end{align}
 where $\xi = \cond \Cmax \kappa^2$ and $\zeta = \kappa^2 \Cmax$ and $\gratet = (\E[\normt{\bDelta_g(\bz)}^4])^{1/4}$.
\end{lemma}
\begin{proof}
To begin we rescale the $\xstar$ such that is has unit-norm (and restore the scaling in the final statement of the proof).
We begin by establishing concentration of the $J$ term which justifies the thresholding step in the estimator using a 4th-moment Markov inequality. We have that $J = \frac{1}{n} \sum_{i=1}^n \nabla_{\theta} m(t_i, y_i, \theta_0, \bz_i^\top \bbtf, \btg(\bz_i)) = \frac{1}{n} \sum_{i=1}^{n} t_i (t_i-\btg(\bz_i)) = \frac{1}{n} \sum_{i=1}^n J_i$. Note that we assume $t_i = \bz_i^\top \bg_0+\eta_i$. Then, for an individual term we have that,
    \begin{align}
    & J_i = (\bz_i^\top \bg_0 + \eta_i)(\bDelta_g(\bz_i) + \eta_i) = \underbrace{\eta_i^2  + \bz_i^\top \bg_0 \eta_i}_{a_i} + \underbrace{\bz_i^\top \bg_0 (\bDelta_g(\bz_i))}_{b_i}
    \end{align}
Recall by Remark \ref{rmk:precision_rows}, that $\eta = (1, -\bg_0)^\top \bx'$, and that $\Vert \bg_0 \Vert_2^2 = O(\cond)$. Using sub-gaussianity of $\bx'$ we have that  $\eta_i \sim \sE( 8\cond \kappa^2 \Cmax, 8\cond \kappa^2 \Cmax)$ by Lemma \ref{lem:prod_conc}. Similarly, $\bz_i^\top \bg_0 \eta_i \sim \sE(8 \cond \kappa^2 \Cmax, 8 \cond \kappa^2 \Cmax)$ since $\bz_i^\top \bg_0 \sim \sG(\Cmax \kappa^2 \cond)$. We introduce $\xi = \cond \Cmax \kappa^2$ and $\zeta = \kappa^2 \Cmax$.

Analyzing each term, we have that,
\begin{itemize}
\item For the first terms, $\E[\eta_i^2] = \E[\eta^2]$. Similarly for the second term, note $\E[b_i]=0$ since $\eta_i$ is conditionally (on $\bz$) mean-zero. Hence we have that each $a_i$ is mean-zero and $a_i \sim \sE(16 \xi, 16 \xi)$.
\item For the final term, note $\E[(\bz_i^\top \bg_0(\bDelta(\bz_i)))^4] \leq
O(\xi^2  r_g^4)$ by Cauchy-Schwarz.
\end{itemize}

Since, $J = \frac{1}{n} \sum_{i=1}^{n} a_i + b_i + c_i$, if $\abs{\frac{1}{n} \sum_{i=1}^{n} b_i+c_i} \leq \epsilon'$ and $\frac{1}{n} \sum_{i=1}^{n} a_i  \geq \epsilon'+\tau$ then $\frac{1}{n} \sum_{i=1}^{n} J_i > \tau$. So a union bound gives,
\begin{align}
    \Pr[\frac{1}{n} \sum_{i=1}^{n} J_i \leq \tau] \leq  \Pr[\frac{1}{n} \sum_{i=1}^{n} a_i  < +\epsilon'+\tau]  + \Pr[\abs{ \frac{1}{n} \sum_{i=1}^{n} b_i+c_i } \geq \epsilon' ]
\end{align}
    
Using a sub-exponential tail bound for the first term and the 4th-moment Marcinkiewicz–Zygmund inequality for the second we obtain,
    \begin{itemize}
        \item For the first term \begin{align} 
        \Pr[\frac{1}{n} \sum_{i=1}^{n} a_i - \E[\eta^2] \leq -(\underbrace{-\epsilon'- \tau +\E[\eta^2]}_t)] \leq O(\exp(-c n \min(\frac{t^2}{\xi^2}, \frac{t}{\xi}))
        \end{align}
        for some universal constant $c$ (that may change line to line).
        \item For the second term 
        \begin{align}  \Pr[ \abs{\frac{1}{n} \sum_{i=1}^{n} b_i} \geq +\epsilon'] \leq 
        O(\frac{\xi^2 \grate^4}{(\epsilon')^4 n^2})
        \end{align}
    \end{itemize}
    Taking $\epsilon' = \frac{1}{8} \sigma_{\eta}^2$ and $\tau \leq \frac{1}{4} \sigma_{\eta}^2$ it follows that $t \geq \frac{1}{2} \sigma_{\eta}^2$. Hence the second term can be simplified to $O(\exp(-c n \min(\frac{t^2}{\xi^2}, \frac{t}{\xi})) = O(\max (\frac{\xi^2}{\sigma_{\eta}^2}, \frac{\xi}{\sigma_{\eta}})^2 \frac{1}{n^2})$ Hence the desired bound becomes, $\Pr[\frac{1}{n} \sum_{i=1}^n J_i \leq \epsilon] \leq O(\frac{\xi^2 \grate^4}{(\sigma_{\eta}^2)^4 n^2})+ O(\exp(-c n \min(\frac{t^2}{\xi^2}, \frac{t}{\xi})) = O(((\frac{\xi}{\sigma_{\eta}^2})^4 + \frac{ \xi^2}{(\sigma_{\eta}^2)^4} \grate^4)  \cdot \frac{1}{n^2})$. 
\end{proof}
    
\subsection{\cref{cor:oml_ub_ridge,,cor:oml_ub_lasso}}
\label{sec:app_oml_ub_ridge}

We conclude the section by presenting the proofs of \cref{cor:oml_ub_ridge} and \cref{cor:oml_ub_lasso} which instantiate the OM estimators when both first-stage regressions are estimated with the Lasso.

First we prove \cref{cor:oml_ub_ridge}.
\begin{proof}[Proof of \cref{cor:oml_ub_ridge}]
    It suffices to compute $\bratet$ and $\gratet$. By using \cref{lem:ridge_total_ub},
    \begin{talign}
         & \bratet \leq O \left( \frac{\sigma_{\epsilon}^2 p}{n} \right)
    \end{talign}
    by utilizing condition on $\lambda_{\beta}$ in the theorem statement and that $\normi{\bbeta_0}/\sigma_{\epsilon} \leq O(1)$ and $n \geq \Omega(\kappa^4 \cond^2 p)$.
    Similarly, for the case of $\gratet$ in the case the estimator is parametric Lasso estimator it follows that $\gratet = (\E[(\bz^\top(\bg_0-\bg))^4])^{1/4} \leq O(\sqrt{\zeta}\E[(\Vert \bg_0 - \bg \Vert_2^4])^{1/4})$ where $\zeta = \kappa^2 \Cmax$. Similar to above we obtain that, 
    \begin{talign}
        \gratet \leq O(\frac{\sigma_{\eta}^2 p}{n})
    \end{talign}
    since we can verify that the conditions of \cref{lem:ridge_total_ub} also hold when $t$ is regressed against $\bz$ under the hypotheses of the result. In particular, note since the regression for $\bg$ is performed between $t$ and $\bz$ (which up to an orthogonal rotation is a subvector of the original covariate $\bx$ itself), the minimum eigenvalue for this regression is lower-bounded by the minimum eigenvalue of $\bX$. Moreover by \cref{rmk:precision_rows}, $\normt{\bg_0} \leq \sqrt{\cond}$.
\end{proof}
    
\begin{proof}[Proof of \cref{cor:oml_ub_lasso}]
    It suffices to compute $\bratet$ and $\gratet$. The computation for $\bratet$ is similar to the one for $\brateo$. By combining \cref{lem:lasso_upper_sg} and \cref{lem:uncond_exp}, and assuming $p \geq 20$ and $n \geq \frac{c_2 \kappa^4}{C_{\min}} s \log(2ep)$, there exists sufficiently large $c$ such that,
    \begin{talign}
         & \bratet \leq O \left( \left(\frac{ \lambda_{\beta} \sqrt{s_{\beta}}}{\Cmin^2}  \right) + O(\left( \frac{\sigma_{\epsilon}}{\sqrt{ns_{\beta}}} \right)  + (\frac{1}{n \lambda_{\beta} \sqrt{s_{\beta}}}) \right) + O \left(\frac{\sigma_{\epsilon}}{\sqrt{n}}  + \normo{\bbeta_0} e^{-n c/(8 \kappa^4)} \right) = \\
         & O \left( \frac{\lambda_{\beta} \sqrt{s_{\beta_0}}} {\Cmin^2}  \right)  + O(\normo{\bbeta_0} e^{-n c/(8 \kappa^4)}) \leq O(\frac{\lambda_{\beta} \sqrt{s_{\bbeta_0}}}{\Cmin^2})
    \end{talign}
    using the lower bound on $\lambda_{\beta}$ in the theorem statement and that $\normi{\bbeta_0}/\sigma_{\epsilon} = o(e^{c_1 n})$ for some sufficiently small $c_1$.
    Similarly, for the case of $\gratet$ in the case the estimator is parametric Lasso estimator it follows that $\gratet = (\E[(\bz^\top(\bg_0-\bg))^4])^{1/4} \leq O(\sqrt{\zeta}\E[(\Vert \bg_0 - \bg \Vert_2^4])^{1/4})$ where $\zeta = \kappa^2 \Cmax$. Similar to above we obtain that, 
    \begin{talign}
        \gratet \leq O(\frac{\lambda_g \sqrt{s_{\bg_0}}}{\Cmin^2}) + O(\normo{\bg_0}  e^{-nc/(8 \kappa^4)}) \leq O(\frac{\lambda_{\bg} \sqrt{s_{\bg_0}}}{\Cmin^2})
    \end{talign}
    since we can verify that the conditions of \cref{lem:lasso_upper_sg} and \cref{lem:uncond_exp} also hold when $t$ is regressed against $\bz$ under the hypotheses of the result. In particular, note since the regression for $\bg$ is performed between $t$ and $\bz$ (which up to an orthogonal rotation is a subvector of the original covariate $\bx$ itself), the strong-restricted eigenvalue for this regression is lower-bounded by the strong-restricted eigenvalue of $\bX$. Moreover by \cref{rmk:precision_rows}, $\normo{\bg_0} \leq s_{\bg_0} \normt{\bg_0} \leq s_{\bg_0} \sqrt{\cond}$.
\end{proof}
     
\section{Auxiliary Lemmas}\label{sec:aux_appendix}
We now introduce a standard concentration result we will repeatedly use throughout, 
\begin{lemma}
    Let $x, y$ be mean-zero random variables that are both sub-Gaussian with parameters $\kappa_1$ and $\kappa_2$ respectively. Then $z = xy-\E[xy] \sim \sE(8 \kappa_1 \kappa_2, 8 \kappa_1 \kappa_2)$.
    \label{lem:prod_conc}
\end{lemma}
\begin{proof}
    Using the dominated convergence theorem,
    \begin{align}
        & \E[e^{\lambda z}] = 1+\sum_{k=2}^{\infty} \frac{\lambda^k \E[\left(xy-\E[xy]\right)^k]}{k!} \\
        & \leq 1 + \sum_{k=2}^{\infty} \frac{\lambda^k 2^{k-1} (\E[\abs{xy}^k]+\E[\abs{xy}]^k)}{k!} \\
        & \leq 1 + \sum_{k=2}^{\infty} \frac{\lambda^k 2^{k} \sqrt{\E[x^{2k}] \E[y^{2k}]}}{k!} \\
        & \leq 1 + \sum_{k=2}^{\infty} \frac{\lambda^k 2^{k} (2\kappa_1 \kappa_2)^k (2k) \Gamma(k)}{k!} = 1 + 2 (4\lambda \kappa_1 \kappa_2)^2 \sum_{k=0}^{\infty} (4 \lambda \kappa_1 \kappa_2)^k \\
        & \leq 1+4(4\lambda \kappa_1 \kappa_2)^2 = 1+64 \lambda^2 \kappa_1^2 \kappa_2^2 \quad \text{ for } \abs{\lambda} \leq \frac{1}{8\kappa_1 \kappa_2} \\
        & \leq e^{ (\lambda \cdot 8\kappa_1 \kappa_2)^2} \leq e^{ (\lambda \cdot 8\kappa_1 \kappa_2)^2/2}
    \end{align}
\end{proof}
where we have used the fact a sub-Gaussian random variable $x$ with parameter $\kappa$ satisfies $\E[\abs{x}^k] \leq (2\kappa^2)^{k/2} k \Gamma(k/2)$ (which itself follows from integrating the sub-gaussian tail bound), along with the Cauchy-Schwarz and Jensen inequalities. 

\subsection{Random Design Matrices and Lasso Consistency}
Here we collect several useful results we use to show consistency of the Lasso estimator in the random design setting. 

Note Assumption \ref{assump:cov} ensures the population covariance for the design $\bX$ satisfies $\bSigma_{ii} \leq 1/2$, and a standard sub-exponential concentration argument establishes the result for a random design matrix under Assumption \ref{assump:design}. Accordingly, we introduce,
\begin{definition}
  The design matrix $\bX \in \mathbb{R}^{n \times p}$ if satisfies the $1$-column normalization condition if \begin{align}
      \max_{i \in [p]} \norm{\bX \be_j}_2^2/n =  \hat{\bSigma}_{ii} \leq 1 \quad 
  \end{align}
  \label{def:column_norm}
\end{definition}
and we have that, 
\begin{lemma}\label{lem:column_norm}
    Let $\kappa' = 8 \sqrt{2} \kappa$. If Assumptions \ref{assump:cov} and \ref{assump:design} hold, then
    \begin{align}
    \Pr[ \max_{i \in [p]} [(\hbSigma)_{ii} -\bSigma_{ii}] \geq  t] \leq p \exp(-\frac{n}{2} \min(\frac{t^2}{\kappa'^2}, \frac{t}{\kappa'} ))
    \end{align}
    and if $n \geq 2 a \max(\kappa'^2, \kappa') \log p$, then with probability at least $1-p^{-a}$
    \begin{align}
        \max_{i \in [p]} (\hbSigma)_{ii} \leq 1.
    \end{align}
\end{lemma}
\begin{proof}
    Note that $\bx_i = \bx^\top \be_i$ satisfies $\E[\exp(\lambda \bx_i)] \leq \exp(\lambda^2 \kappa^2 \bSigma_{ii}/2)$. For fixed $i$ we have that $(\hbSigma)_{ii} = \frac{1}{n} \sum_{i=1}^{n}(\bx_{i}^2-\bSigma_{ii})$. Since  $\bx_i \sim \sG(\kappa \sqrt{\bSigma_{ii}})$, using Lemma \ref{lem:prod_conc} along with a sub-exponential tail bound we have that, 
    \begin{align}
        \Pr[ (\hbSigma)_{ii} \geq \bSigma_{ii} + t] \leq \exp(-\frac{n}{2} \min(\frac{t^2}{\kappa'^2}, \frac{t}{\kappa'} ))
    \end{align}
    defining $\kappa' = 8 \kappa \sqrt{\bSigma_{ii}} \leq 4\sqrt{2} \kappa$.
    Since $\bSigma_{ii} \leq \frac{1}{2}$ using a union bound over the $p$ coordinates we have that $\max_{i \in [p]} (\hbSigma)_{ii} \geq 1$,
    with probability less than $p \exp(-\frac{n}{2} \min(\frac{t^2}{\kappa'^2}, \frac{t}{\kappa'}))$. If $t=\frac{1}{2}$ and $n \geq 2 a \max(\frac{\kappa'^2}{t^2}, \frac{\kappa'}{t}) \log p$ the stated conclusion holds.
\end{proof}
Similarly, although the sample covariance will not be invertible for $p > n$ we require it to be nonsingular along a restricted set of directions. To this end we
introduce the strong restricted eigenvalue condition (or SRE condition) defined in \citep[Equation 4.2]{bellec2016slope} which is most convenient for our purposes.
\begin{definition}
  Given a symmetric covariance matrix $\bQ \in \mathbb{R}^{p \times p}$ satisfying $\max_{i \in [p]} \bQ_{ii} \leq 1$, an integer $s$, and parameter $L$, the strong restricted eigenvalue of $Q$ is,
  \begin{align}
    \phi^2_{SRE}(Q, s, L) \equiv \min_{\theta} \left \{ \frac{ \langle \theta, \bQ \theta \rangle} {\norm{\theta_S}_2^2} : \theta \in \mathbb{R}^p, \norm{\theta}_1 \leq (1+L) \sqrt{s} \norm{\theta}_2 \right \}.
  \end{align}
  \label{def:SRE}
\end{definition}

In general the cone to which $\theta$ belongs in Definition \ref{def:SRE} is more constraining then the cone associated with the standard restricted eigenvalue condition of \citet{bickel2009simultaneous}. Interestingly, due to the inclusion of the 1-column normalization constraint in Definition \ref{def:SRE}, up to absolute constants, the SRE condition is equivalent to the standard RE condition (with the 1-column normalization constraint also included in its definition) \citep[Proposition 8.1]{bellec2016slope}. 

Importantly, using further equivalence with $s$-sparse eigenvalue condition, \citep[Theorem 8.3]{bellec2016slope} establishes the SRE condition holds with high probability under the sub-gaussian design assumption.

\begin{theorem}\citet[Theorem 8.3]{bellec2016slope}.
Let Assumptions \ref{assump:cov} and \ref{assump:design} hold. Then there exist absolute constants $c_1, c_2 > 0$ such that for $L \geq 0$, if $n \geq \frac{c_1 \kappa^4 (2+L)^2}{C_{\min}} s \log(2ep/s)$, then with probability at least $1-3 \exp(-c_2 n/\kappa^4)$, we have that
\begin{align}
    \max_{i \in [p]} (\hbSigma)_{ii} \leq 1 
\end{align}
and 
\begin{align}
    \phi^2_{SRE}(\hbSigma, s, L) \geq \frac{C_{\min}}{2}
\end{align}
\label{thm:design_SRE}
\end{theorem}
This result follows from \citep[Theorem 8.3]{bellec2016slope}, the stated implication therein that the weighted restricted eigenvalue condition implies the strong restricted eigenvalue condition with adjusted constants, along with the fact that $\phi^2_{SRE} (\bSigma, s, L) \geq C_{\min}$.

We define the sequence of sets,
\begin{align}
    \mE_n(s, L) = \{ \bX \in \mathbb{R}^{n \times p} : \phi^2_{SRE}(\hbSigma, s, L) \geq \frac{C_{\min}}{2}, \max_{ i \in [p]} \hat{\bSigma}_{ii} \leq 1, \hat{\bSigma} = \bX^\top \bX/n \} 
\end{align}
characterizing the class of design matrices satisfying both Definitions \ref{def:column_norm} and \ref{def:SRE}. 

There are many classical results on $\ell_1/\ell_2$ consistency of the Lasso program,
\begin{align}
    \hbbeta_L = \argmin_{\beta \in \mR^p} \frac{1}{2} \Vert \by-\bX\bbeta \Vert_2^2 + \lambda \Vert \bbeta \Vert_1
\end{align}
for sparse regression (see for example \citep[Ch. 6]{van2014statistical}) when the model is specified as $\by=\bX \bbeta_0+\bepsilon$ for $\epsilon_i$ i.i.d. that are sub-Gaussian with variance parameter $\sigma^2$.
Such classical results have the confidence level of the non-asymptotic error tied directed directly to the tuning parameter. However, recently \citep{bellec2016slope}, through a more refined analysis, has obtained optimal rates for the Lasso estimator over varying confidence levels for a \textit{fixed} regularization parameter. These results allow us to provide clean upper bounds on the Lasso parameter error in expectation.

\begin{lemma}
    Let $s \in [p]$, assume that the deterministic design matrix $\bX \in \mE_{n}(s, 7)$, and let Assumption \ref{assump:noise1} hold with $\epsilon_i \sim \mN(0, \sigma^2)$. If $\hblaslambda$ denotes the Lasso estimator with $\lambda \geq (8+2 \sqrt{2}) \sigma \sqrt{\frac{\log(2ep/s)}{n}}$, $1 \leq q \leq 2$, and $\norm{\bbeta_0}_0 \leq s$ then letting $\phi_0^2 = \phi^2_{SRE}(s, 7)$,
    \begin{align}
        \E[\Vert \hblaslambda - \bbeta_0 \Vert_q^k] \leq  \left(\frac{49 \lambda s^{1/q}}{8 \phi_0^2}  \right)^k + \left(\frac{49}{8} \frac{(8+2 \sqrt{2})\sigma}{s^{1-1/q} \sqrt{n}} \right)^k \frac{k(k-1)}{2}
    \end{align}
    \label{lem:lasso_consist_g}
\end{lemma}
\begin{proof}
    The proof follows easily by integrating the tail bound in \citet[Theorem 4.2]{bellec2016slope},
    which provides that, 
    \begin{align}
        \Vert \hblaslambda-\bbeta_0 \Vert_q \leq \frac{49}{8} \left (\frac{\log(1/\delta_0)}{s \log(1/\delta(\lambda))} \vee \frac{1}{\phi^2_0} \right) \lambda s^{1/q}
    \end{align}
    with probability at least $1-\delta_0/2$,
    where $\delta(\lambda) = \exp(-(\frac{\lambda \sqrt{n}}{(8+2\sqrt{2}) \sigma}))$, which satisfies $\delta(\lambda) \leq \frac{s}{2ep}$. Now, define $\delta_0^*$ as the smallest $\delta_0 \in (0,1)$ for which $\frac{1}{\phi_0^2} = \frac{\log(1/\delta_0)}{s \log(1/\delta(\lambda))}$, in which case $\delta_0^* = (\delta(\lambda))^{\frac{s}{\phi_0^2}}$. 
    
    Then, $Z_q = \frac{8 s \log(1/\delta(\lambda))}{49 \lambda s^{1/q}} \leq \log(1/\delta_0)$ with probability at least $1-\delta_0/2$, for all $\delta_0 \in (0, \delta_0^*]$. Equivalently, $\Pr[Z_q > t] \leq \frac{e^{-t}}{2}$ for all $t \geq T= \log(1/\delta_0^*)=\frac{s}{\phi_0^2} \log(1/\delta(\lambda))$.
   Thus,
    \begin{align}
        & \E[Z_q^k] = \int_{0}^{\infty} kt^{k-1} \Pr[Z_q > t] dt = \int_{0}^{T} kt^{k-1} + \int_{T}^{\infty} kt^{k-1}\frac{e^{-t}}{2} \leq \\
        & T^k + \int_{0}^{\infty} kt^{k-1}\frac{e^{-t}}{2} \leq T^k + \frac{k(k-1)}{2}.
    \end{align}
    which implies the conclusion,
    \begin{align}
        & \Vert \hblaslambda - \bbeta_0 \Vert_q^k \leq \left(\frac{49}{8} \frac{T \lambda s^{1/q}}{s \log(1/\delta(\lambda))} \right)^k + \left(\frac{49}{8} \frac{\lambda s^{1/q}}{s \log(1/\delta(\lambda))} \right)^k \frac{k(k-1)}{2} \leq \\
        & \left(\frac{49 \lambda s^{1/q}}{8 \phi_0^2}  \right)^k + \left(\frac{49}{8} \frac{(8+2 \sqrt{2})\sigma}{s^{1-1/q} \sqrt{n}} \right)^k \frac{k(k-1)}{2}
    \end{align}
    where $\lambda \geq (8+2 \sqrt{2}) \sigma \sqrt{\frac{\log(2ep/s)}{n}}$.
\end{proof}

Although the main results of \citet{bellec2016slope} are stated for Gaussian noise distributions, \citet[Theorem 9.1]{bellec2016slope} also provides a complementary high-probability upper bound for the empirical process $\frac{1}{n} \bepsilon^\top \bX \bu $ when $\bepsilon$ is sub-gaussian:
\begin{lemma}\citet[Theorem 9.1]{bellec2016slope}
    Let $\delta_0 \in (0,1)$, and let \cref{assump:noise1} hold (with variance parameter renamed to $\sigma^2$) and assume the deterministic design matrix $\bX \in \mR^{n \times p}$ satisfies $\max_{i \in [p]} \norm{\bX \be_i}_2/\sqrt{n} \leq 1$. Then with probability at least $1-\delta_0$, for all $u \in \mR^p$,
    \begin{align}
        \frac{1}{n} \bepsilon^\top \bX \bu \leq 40 \sigma \max \left( \sum_{j=1}^p u_j^{\sharp} \sqrt{\frac{\log(2p/j)}{n}}, \frac{\norm{\bX \bu}_2}{\sqrt{n}} \frac{\sqrt{\pi/2} + \sqrt{2 \log(1/\delta_0)}}{\sqrt{n}} \right)
    \end{align}
    \label{lem:sg_noise}
\end{lemma}

The upper bound contains an additional, additive $\frac{\sqrt{\pi/2}}{\sqrt{n}}$ correction along with a change in absolute constants with respect to \citet[Theorem 4.1]{bellec2016slope}. Hence we trace through the proof of \citet[Theorem 4.2]{bellec2016slope} to derive a corresponding statement of \citet[Theorem 4.2]{bellec2016slope} for sub-gaussian distributions.

\begin{lemma}
    Let $s \in [p]$, $\gamma \in (0,1)$ and $\tau \in (0, 1-\gamma]$ and assume the SRE(s, $c_0$) condition holds $c_0(\gamma, \tau) = \frac{1+\gamma+\tau}{1-\gamma-\tau}$. Let $\lambda \geq \frac{40 \sigma}{\gamma} \sqrt{\frac{\log(2ep/s)}{n}}$. Then on the event in Lemma \ref{lem:sg_noise}, for $1 \leq q \leq 2$,
    \begin{align}
       \Vert \hblaslambda -\bbeta_0 \Vert_q \leq \left(\frac{C_{\gamma,\tau}(s, \lambda, \delta_0)}{\tau}\lambda s + \frac{\pi(1+\tau+\gamma)^2}{\gamma^2 \tau n \lambda}\right)^{2/q-1} \left(3(\frac{ C_{\gamma,0}(s, \lambda, \delta_0)}{1+\gamma} \lambda \sqrt{s} + \frac{\pi(1+\gamma)}{\gamma^2 \lambda \sqrt{2s} n}) \right)^{2-2/q}
    \end{align}
    where $C_{\gamma, \tau} = (1+\gamma+\tau)^2 \left(\frac{\log(1/\delta_0)}{s \log(1/\delta(\lambda))} \vee \frac{1}{\phi_0^2(s, c_0(\gamma, \tau))} \right)$.
\end{lemma}
\begin{proof}
    The argument simply requires tracing through the proof of \citet[Theorem 4.2]{bellec2016slope} to accommodate the additional $O(\frac{1}{\sqrt{n}})$ term (and is nearly identical to \citet[Theorem 4.2]{bellec2016slope}), so we only highlight the important modifications.
    
    Following the proof of \citet[Theorem 4.2]{bellec2016slope} we have, 
    \begin{align}
        2 \tau \lambda \Vert \hblaslambda-\bbeta_0 \Vert_1 + 2 \Vert \bX(\hblaslambda-\bbeta_0)\Vert_2^2/n \leq \Delta^* \label{eq:basic_ineq}
    \end{align}
    where $\Delta^* = 2 \tau \lambda \Vert \hblaslambda-\bbeta_0 \Vert_1 + \frac{2}{n} \bepsilon^\top \bX (\hblaslambda-\bbeta_0) + 2 \lambda \Vert \bbeta_0 \Vert_1 - 2\lambda \Vert \hblaslambda \Vert_1$
    Letting $\bu = \hblaslambda-\bbeta_0$, we obtain
    \begin{align}
        \Delta^* \leq 2 \lambda \left( (1+\tau) \sqrt{s} \Vert \bu \Vert_2-(1-\tau) \sum_{j=s+1}^{p} u_j^{\sharp} \right) + 2 \max(F(\bu), G(\bu))
    \end{align}
    where $F(\bu) = \gamma \lambda \left( \sqrt{s} \Vert \bu \Vert_2 + \sum_{j=s+1}^p u_j^{\sharp} \right)$ and $G(\bu) = 40 \sigma ( \frac{\norm{\bX \bu}_2}{\sqrt{n}} \frac{\sqrt{\pi/2} + \sqrt{2 \log(1/\delta_0)}}{\sqrt{n}})$. By definition of $\delta(\lambda) = \exp(-(\frac{\gamma \lambda \sqrt{n}}{40 \sigma})^2)$ we have equivalently that, $G(\bu) = \left( \lambda \sqrt{s} \gamma \sqrt{\log(1/\delta_0)/(s \log(1/\delta(\lambda)))} +\frac{40 \sqrt{\pi/2} \sigma}{\sqrt{n}} \right)\Vert \bX \bu \Vert_2/\sqrt{n}$.
    
    We now consider two cases
    \begin{enumerate}
        \item $G(\bu) > F(\bu)$. Then, 
        \begin{align}
            \Vert \bu \Vert_2 \leq \left( \sqrt{\frac{\log(1/\delta_0)}{s \log(1/\delta(\lambda))}} + \frac{40 \sqrt{\pi/2} \sigma}{ \lambda \sqrt{s} \gamma \sqrt{n}} \right) \Vert \bX \bu \Vert_2/\sqrt{n} \label{eq:case1_begin}
        \end{align}
        Thus, 
        \begin{align}
            & \Delta^* \leq 2 \lambda(1+\tau) \sqrt{s} \Vert \bu \Vert_2 + 2 G(\bu)  \\
            & 2 \lambda \sqrt{s} (1+\tau+\gamma) \left( \sqrt{\frac{\log(1/\delta_0)}{s \log(1/\delta(\lambda))}} + \frac{40 \sqrt{\pi/2} \sigma}{\lambda \sqrt{s} \gamma \sqrt{n}} \right) \Vert \bX \bu \Vert_2/\sqrt{n} \leq \\
            & 2 \lambda^2 s (1+\tau+\gamma)^2 \left(\frac{\log(1/\delta_0)}{s\log(1/\delta(\lambda))}+\frac{800 \pi \sigma^2}{\lambda^2 s \gamma^2 n} \right) + \Vert \bX \bu \Vert_2^2/n = \\
            & 2 \lambda^2 s (1+\tau+\gamma)^2 (\frac{\log(1/\delta_0)}{s\log(1/\delta(\lambda))}) + \Vert \bX \bu \Vert_2^2/n + \frac{1600 \pi \sigma^2(1+\tau+\gamma)^2}{\gamma^2 n} \label{eq:case1_end}
        \end{align}
        \item $G(\bu) \leq F(\bu)$. In this case,
        \begin{align}
            \Delta^* \leq 2 \lambda \left((1+\gamma+\tau) \sqrt{s} \Vert \bu \Vert_2-(1-\gamma-\tau) \sum_{j=s+1}^{p} u_j^{\sharp} \right) = \Delta \label{eq:case2_begin}
        \end{align}
        Since $\Delta > 0$, $\bu$ belongs to the $SRE(s, c_0)$ cone and hence $\phi_0(s, c_0) \Vert \bu \Vert_2 \leq \Vert \bX \bu \Vert_2$. So,
        \begin{align}
            \Delta^* \leq \Delta \leq \frac{2(1+\gamma+\tau) \lambda \sqrt{s}}{\phi_0(s, c_0)} \Vert \bX \bu \Vert/\sqrt{n} \leq
            \left(\frac{(1+\gamma+\tau)\lambda \sqrt{s}}{\phi_0(s, c_0)}) \right)^2+\Vert \bX \bu \Vert^2/n \label{eq:case2_end}
        \end{align}
    \end{enumerate}
    
    Assembling the two cases we conclude that, 
    \begin{align}
        2 \tau \Vert \hblaslambda - \bbeta_0 \Vert_1 \leq 2 C_{\gamma, \tau}(s, \lambda, \delta_0) \lambda s + \frac{1600\pi \sigma^2 (1+\tau+\gamma)^2}{\gamma^2 n \lambda}
    \end{align}
    where $ C_{\gamma, \tau}(s, \lambda, \delta_0) = (1+\gamma+\tau)^2 \left( \frac{\log(1/\delta_0)}{s \log(1/\delta(\lambda))}\vee \frac{1}{\phi_0^2(s, c_0(\gamma, \tau))}\right)$.
    
    Turning to upper bounding $\bu$ in the $\ell_2$ norm, we specialize to $\tau=0$ and consider cases 1 and 2 from before. 
    \begin{enumerate}
        \item $G(\bu) > F(\bu)$, then using Equations \ref{eq:basic_ineq} and \ref{eq:case1_end} we have,
        \begin{align}
            \Vert \bX \bu \Vert_2^2/n \leq 2 \lambda^2 s (1+\gamma)^2 (\frac{\log(1/\delta_0)}{s\log(1/\delta(\lambda))})  + \frac{1600 \pi \sigma^2(1+\gamma)^2}{\gamma^2 n}
        \end{align}
        Combining the previous display with \cref{eq:case1_begin} we have,
        \begin{align}
            & \norm{\bu}_2 \leq \left( \sqrt{2 \lambda^2 s (1+\gamma)^2 (\frac{\log(1/\delta_0)}{s\log(1/\delta(\lambda))})} + \sqrt{\frac{1600 \pi \sigma^2(1+\gamma)^2}{\gamma^2 n}} \right) \left(\sqrt{\frac{\log(1/\delta_0)}{s \log(1/\delta(\lambda))}} + \frac{40 \sqrt{\pi/2} \sigma}{ \lambda \sqrt{s} \gamma \sqrt{n}}\right)  \\
            & = \sqrt{2 s}(1+\gamma) \lambda  \left(\frac{\log(1/\delta_0)}{s \log(1/\delta(\lambda))}\right) + \frac{1600 \pi (1+\gamma) \sigma^2}{\gamma^2 \lambda \sqrt{2s} n}  + \sqrt{\sqrt{2s}(1+\gamma) \lambda \frac{\log(1/\delta_0)}{s \log(1/\delta(\lambda))} \cdot \frac{1600\pi \sigma^2(1+\gamma)}{\gamma^2 \lambda \sqrt{2s} n}} \\
            & \leq \frac{3}{2} \left( \sqrt{2 s}(1+\gamma) \lambda  \left(\frac{\log(1/\delta_0)}{s \log(1/\delta(\lambda))}\right) + \frac{1600 \pi \sigma^2 (1+\gamma)}{\gamma^2 \lambda \sqrt{2s} n} \right)
        \end{align}
        using subadditivity of $\sqrt{\cdot}$.
        \item $G(\bu) \leq F(\bu)$. Equations \ref{eq:basic_ineq} and \ref{eq:case2_begin} implies that $\Delta \geq \Delta^* \geq 0$ a.s. Hence $\bu$ is contained in $SRE(s, \frac{1+\gamma}{1-\gamma})$, and
        \begin{align}
            \Vert \bu \Vert_2 \leq \frac{\Vert \bX \bu \Vert_2}{n \phi_0(s, \frac{1+\gamma}{1-\gamma})} \leq \frac{(1+\gamma) \lambda \sqrt{s}}{\phi_0^2(s, \frac{1+\gamma}{1-\gamma})}
        \end{align}
        using \cref{eq:basic_ineq} and \eqref{eq:case2_end}, and recalling we set $\tau=0$. Assembling these two cases we conclude,
        \begin{align}
            (1+\gamma) \Vert \hblaslambda -\bbeta_0 \Vert_2 \leq 3 \left( C_{\gamma, 0}(s, \lambda, \delta_0) \lambda\sqrt{s} + \frac{1600\pi \sigma^2 (1+\gamma)^2}{\gamma^2 \lambda \sqrt{2s} n} \right)
        \end{align}
        \end{enumerate}
        
         So using the norm interpolation inequality $\Vert \hblaslambda -\bbeta_0 \Vert_q \leq \Vert \hblaslambda -\bbeta_0 \Vert_1^{2/q-1} \Vert \hblaslambda -\bbeta_0 \Vert_2^{2-2/q}$,
        \begin{align}
            \Vert \hblaslambda -\bbeta_0 \Vert_q \leq \left(\frac{C_{\gamma,\tau}(s, \lambda, \delta_0)}{\tau}\lambda s + \frac{1600 \pi \sigma^2(1+\tau+\gamma)^2}{\gamma^2 \tau n \lambda}\right)^{2/q-1} \left(3(\frac{ C_{\gamma,0}(s, \lambda, \delta_0)}{1+\gamma} \lambda \sqrt{s} + \frac{1600 \pi \sigma^2(1+\gamma)}{\gamma^2 \lambda \sqrt{2s} n}) \right)^{2-2/q}
        \end{align}
\end{proof}

We can now derive a corresponding moment bound for error as before\footnote{for convenience we only state for the $\ell_1$ and $\ell_2$ norms an analagous result to Lemma \ref{lem:lasso_consist_g} can be derived with more computation.},
\begin{lemma}
    Let $s \in [p]$, assume that the deterministic design matrix $\bX \in \mE_{n}(s, 7)$, and let \cref{assump:noise1} hold (with variance parameter renamed to $\sigma^2$). If $\hblaslambda$ denotes the Lasso estimator with $\lambda \geq 80 \sigma \sqrt{\frac{\log(2ep/s)}{n}}$, $1 \leq q \leq 2$, and $\norm{\bbeta_0}_0 \leq s$ then letting $\phi_0^2 = \phi^2_{SRE}(s, 7)$,
    \begin{align}
        & \E[\Vert \hblaslambda - \bbeta_0 \Vert_1^k] \leq   2^{k-1} \left( \left(13 \frac{ \lambda s}{\phi_0^2}  \right)^k + \left(13 \frac{40\sigma}{\sqrt{n}} \right)^k \frac{k(k-1)}{2} + (\frac{250000}{n \lambda})^k \right) \\
        & \E[\Vert \hblaslambda - \bbeta_0 \Vert_2^k] \leq 2^{k-1} \left( \left(5 \frac{ \lambda \sqrt{s}}{\phi_0^2}  \right)^k + \left(13 \frac{40\sigma}{\sqrt{ns}} \right)^k \frac{k(k-1)}{2} + (\frac{25000}{n \lambda \sqrt{s}})^k \right)
    \end{align}
    \label{lem:lasso_upper_sg}
\end{lemma}
\begin{proof}
    We instantiate the result of Lemma \ref{lem:lasso_upper_sg} with $\gamma=1/2$ and $\tau=1/4$ in which case $c_0=7$, $(1+\gamma+\tau)^2=49/16$, $\frac{1+\gamma}{1-\gamma}=3$, $1+\gamma=3/2$. Defining $D(\delta_0, \lambda, s)=\left( \frac{\log(1/\delta_0)}{s \log(1/\delta(\lambda))}\vee \frac{1}{\phi_0^2}\right)$ and $\phi_0^2=\phi_0^2(s, 7)$ we have,
    \begin{align}
            \Vert \hblaslambda -\bbeta_0 \Vert_1 \leq 
            13 D(\delta_0, \lambda, s) \lambda s + \frac{250000 \sigma^2}{ n \lambda}
            \\
            \Vert \hblaslambda -\bbeta_0 \Vert_2 \leq 
            5 D(\delta_0, \lambda, s) \lambda \sqrt{s} + \frac{25000 \sigma^2}{\lambda \sqrt{s} n}
    \end{align}
    with probability $1-\delta_0$ where $\delta(\lambda) = \exp(-(\frac{\lambda \sqrt{n}}{80 \sigma}))$. Now, define $\delta_0^*$ as the smallest $\delta_0 \in (0,1)$ for which $\frac{1}{\phi_0^2} = \frac{\log(1/\delta_0)}{s \log(1/\delta(\lambda))}$, in which case $\delta_0^* = (\delta(\lambda))^{\frac{s}{\phi_0^2}}$. 
    
    Then, $Z_1 =  \frac{(\Vert \hblaslambda -\bbeta_0 \Vert_1 -\frac{250000 \sigma^2}{n \lambda})s \log(1/\delta(\lambda))}{13 \lambda s}\leq 
           \log(1/\delta_0)$ and $Z_2 = \frac{(\Vert \hat{\bbeta}_L -\bbeta_0 \Vert_2 -\frac{25000 \sigma^2}{n \lambda \sqrt{s}})s \log(1/\delta(\lambda))}{5 \lambda \sqrt{s}}$ with probability at least $1-\delta_0$, for all $\delta_0 \in (0, \delta_0^*]$. Equivalently, $\Pr[Z_q > t] \leq e^{-t}$ for all $t \geq T= \log(1/\delta_0^*)=\frac{s}{\phi_0^2} \log(1/\delta(\lambda))$ for $q \in \{1,2 \}$. As before,
    \begin{align}
        & \E[Z_q^k] \leq T^k + k(k-1).
    \end{align}
    Since $\E[\Vert \hblaslambda-\bbeta_0 \Vert_q^k] = \E[(\Vert \hblaslambda-\bbeta_0 \Vert_q-c+c)^k] \leq 2^{k-1} \left( \E[(\Vert \hblaslambda-\bbeta_0 \Vert_q-c)^k]+c^k \right)$, we conclude,
    \begin{align}
        & \E[\Vert \hblaslambda - \bbeta_0 \Vert_1^k] \leq 2^{k-1} \left(\left(13 \frac{T \lambda s}{s \log(1/\delta(\lambda))} \right)^k + \left(13 \frac{\lambda s}{s \log(1/\delta(\lambda))} \right)^k \frac{k(k-1)}{2} + (\frac{250000 \sigma^2}{n \lambda})^k \right) \leq \\
        & 2^{k-1} \left( \left(13 \frac{ \lambda s}{\phi_0^2}  \right)^k + \left(13 \frac{40\sigma}{\sqrt{n}} \right)^k \frac{k(k-1)}{2} + (\frac{250000 \sigma^2}{n \lambda})^k \right)
    \end{align}
    and 
    \begin{align}
        & \E[\Vert \hblaslambda - \bbeta_0 \Vert_2^k] \leq 2^{k-1} \left(\left(5 \frac{T \lambda \sqrt{s}}{s \log(1/\delta(\lambda))} \right)^k + \left(5 \frac{\lambda \sqrt{s}}{s \log(1/\delta(\lambda))} \right)^k \frac{k(k-1)}{2} + (\frac{25000 \sigma^2}{n \lambda \sqrt{s}})^k \right) \leq \\
        & 2^{k-1} \left( \left(5 \frac{ \lambda \sqrt{s}}{\phi_0^2}  \right)^k + \left(5 \frac{40\sigma}{\sqrt{ns}} \right)^k \frac{k(k-1)}{2} + (\frac{25000 \sigma^2}{n \lambda \sqrt{s}})^k \right)
    \end{align}
    where $\lambda \geq 80 \sigma \sqrt{\frac{\log(2ep/s)}{n}}$.
\end{proof}

The aforementioned results establish Lasso consistency (in expectation) conditioned on the event $\bX \in \mE_n(s, 7)$. Generalizing these results to an unconditional statement (on $\bX$) requires the following deterministic lemma to control the norm of the error vector $\norm{\hblaslambda-\bbeta_0}_1$ on the ``bad" events $\bX \notin \mE_n(s, 7)$ where we cannot guarantee a ``fast" rate for the Lasso.
\begin{lemma}
    Let $\hblaslambda$ be the solution of the Lagrangian lasso, then 
    \begin{align}
         \norm{\hblaslambda-\bbeta_0}_1 \leq \frac{1}{2n} \norm{\bepsilon}_2^2/\lambda + 2 \norm{\bbeta_0}_1.
    \end{align}
    \label{lem:bad_bound}
\end{lemma}
\begin{proof}
    By definition we have that,
    \begin{align}
        \frac{1}{2n} \norm{\by-\bX \hblaslambda}_2^2 + \lambda \norm{\hat{\bbeta}_L}_1 \leq  \frac{1}{2n} \norm{\bepsilon}_2^2 + \lambda \norm{\bbeta_0}_1 \implies 
        \norm{\hblaslambda}_1 \leq \frac{1}{2n} \norm{\bepsilon}_2^2/\lambda + \norm{\bbeta_0}_1
    \end{align}
    So by the triangle inequality we obtain that,
    \begin{align}
        \norm{\hblaslambda-\bbeta_0}_1 \leq \frac{1}{2n} \norm{\bepsilon}_2^2/\lambda + 2 \norm{\bbeta_0}_1.
    \end{align}
\end{proof}

With this result in hand we can combine our previous results to provide our final desired consistency result for the Lasso.

\begin{lemma}
    Let Assumptions \ref{assump:well_spec}, \ref{assump:cov}, \ref{assump:design}, \ref{assump:noise1} hold (with variance parameter renamed to $\sigma^2$). Then there exist absolute constants $c_1, c_2 > 0$ such that if $n \geq \frac{c_1(k) \kappa^4}{C_{\min}} s \log(2ep/s)$, and $\hblaslambda$ is a solution of the Lagrangian Lasso then for $q \in {1,2}$
    \begin{align}
        & \E_{\bX, \bepsilon} \left[\Vert \hblaslambda-\bbeta_0 \Vert_q^k \right] \leq \E_{\bX, \bepsilon} \left[\Vert \hblaslambda-\bbeta_0 \Vert_q^k \mone[\bX \in \mE_n(s, 7)] \right] + \left(\frac{\sigma^{2k}}{\lambda^{k}} + 2^{2k} \Vert \bbeta_0 \Vert_1^{k} \right)\left(2e^{-\frac{c_2}{2} n}\right)
    \end{align}
    where the first term can be bounded exactly as the conclusion of either Lemmas \ref{lem:lasso_consist_g} or \ref{lem:lasso_upper_sg} with appropriate choice of regularization parameter $\lambda_{\bbeta}$.
    \label{lem:uncond_exp}
\end{lemma}
\begin{proof}
Consider the event $\{ \bX \notin \mE_n(s, 7) \}$. For $q \in {1,2}$, we can split the desired expectation over the corresponding indicator r.v. giving,
\begin{align}
    \E_{\bX, \bepsilon}[\Vert \hblaslambda-\bbeta_0 \Vert_q^k] = \E_{\bX, \bepsilon} \left[\Vert \hblaslambda-\bbeta_0 \Vert_q^k \mone[\bX \in \mE_n(s, 7)] \right] + \E_{\bX, \bepsilon}\left[\Vert \hblaslambda-\bbeta_0 \Vert_q^k \mone[\bX \notin \mE_n(s, 7)] \right] \label{eq:decomp}
\end{align}
The first term can be bounded using independence of $\bX$ and $\bepsilon$ to integrate over $\bepsilon$ restricted to the set $\{\bX \notin \mE_n(s, 7)\}$ (by applying Lemmas \ref{lem:lasso_consist_g} and \ref{lem:lasso_upper_sg}). The second term can be bounded using Cauchy-Schwarz and Lemma \ref{lem:bad_bound} which provides a coarse bound on the Lasso performance which always holds,
\begin{align}
     \E_{\bX, \bepsilon}\left[\Vert \hblaslambda-\bbeta_0 \Vert_q^k \mone[\bX \notin \mE_n(s, 7)] \right]\leq \sqrt{\E_{\bX, \bepsilon}\left[\Vert \hblaslambda-\bbeta_0 \Vert_q^{2k}\right]} \sqrt{\Pr_{\bX}[\bX \notin \mE_n(s, 7))]}
      \label{eq:CS}
\end{align}
The hypotheses of Theorem \ref{thm:design_SRE} are satisfied, so  $\sqrt{\Pr_{\bX}[\bX \notin \mE_n(s, 7))]} \leq 2 e^{-\frac{c_{2}}{2} n}$. Using Lemma \ref{lem:bad_bound} along with the identity $(a+b)^k \leq 2^{k-1} (a^k+b^k)$ we have that, 
\begin{align}
    \E_{\bX, \bepsilon}[\Vert \hblaslambda-\bbeta_0 \Vert_q^{2k}] \leq \E_{\bX, \bepsilon}[\Vert \hblaslambda-\bbeta_0 \Vert_1^{2k}] \leq 2^{2k-1} \cdot \mathbb{E}_{\bepsilon} \left[\frac{(\sum_{i=1}^{n} \epsilon_i^2/n)^{2k}}{2^{2k} \lambda^{2k}} + 2^{2k} \norm{\bbeta_0}_1^{2k} \right] 
\end{align}
Since the $\epsilon_i \sim \sG(0, \sigma^2)$, $\epsilon_i^2 \sim \sE(8 \sigma^2, 8\sigma^2)$ by Lemma \ref{lem:prod_conc}, so $Z=\sum_{i=1}^{n} \epsilon_i^2/n \sim \sE(8 \sigma^2, 8\sigma^2)$ satisfies the tail bound $\Pr[Z-\E[Z] \geq t] \leq \exp(-n/2 \min(t^2/(8 \sigma^2)^2, t/(8 \sigma^2)))$ since the $\epsilon_i$ are independent. Defining $c=8\sigma^2$, we find by integrating the tail bound,
\begin{align}
    & \E[Z^{k}] = \int_0^{\E[Z]} k t^{k-1} + \int_{\E[Z]}^{c} \exp(-n/2 \cdot t^2/c^2) + \int_{c}^{\infty} \exp(-n/2 \cdot t/c) \leq \\
    & (\sigma^2)^k + \frac{k2^{k/2-1}c^k\Gamma(k/2)}{n^{k/2}} + \frac{k 2^k c^k \Gamma(k)}{n^k} \leq 2(\sigma^2)^k.
\end{align}
since $\E[Z] \leq \sigma^2$, and we choose $n^{2k/2} \geq 2 (2k) 2^{2k/2} (8^{2k}) (2k)!$. Assembling, we have the bound 
\begin{align}
    \E_{\bX, \bepsilon}[\Vert \hblaslambda-\bbeta_0 \Vert_q^{2k}] \leq \frac{\sigma^{4k}}{\lambda^{2k}} + 2^{4k} \Vert \bbeta_0 \Vert_1^{2k} \label{eq:loose_bound}
\end{align}

Inserting the coarse bound in \cref{eq:loose_bound} into \cref{eq:CS} and combining with \cref{eq:decomp} gives the result using subadditivity of $\sqrt{\cdot}$,
\begin{align}
     \E_{\bX, \bepsilon} \left[\norm{\hblaslambda-\bbeta_0}_q^k \right] \leq \E_{\bX, \bepsilon} \left[\Vert \hblaslambda-\bbeta_0 \Vert_q^k \mone[\bX \in \mE_n(s, 7)] \right]+\left( \frac{\sigma^{2k}}{\lambda^{k}} + 2^{2k} \Vert \bbeta_0 \Vert_1^{k} \right)\left(2e^{-\frac{c_2}{2} n}\right)
     \label{eq:expectation_final}
\end{align}
As previously noted the first term in Equation  \eqref{eq:expectation_final} is computed exactly as the final result of either Lemmas \ref{lem:lasso_consist_g} or \ref{lem:lasso_upper_sg}.
\end{proof}

\subsection{Random Design Matrices and Ridge Regression Consistency}
Here we collect several useful results we use to show consistency of the ridge regression estimator in the random design setting. There are several results showing risk bounds for ridge regression in the random design setting, see for example \citet{hsu2012random}. Such results make assumptions which do not match our setting and also do not immediately imply control over the higher moments of the $\ell_2$-error which are also needed in our setting. Accordingly, we use a similar approach to that used for the Lasso estimator to show appropriate non-asymptotic risk bounds (in expectation) for ridge regression.

To begin recall we define the ridge estimator $\hblambda = \arg \min_{\bbeta} \frac{1}{2} \left( \normt{\by-\bX \bbeta}^2 + \lambda \normt{\bbeta}^2 \right)$ which implies $\hblambda = (\bX^\top \bX + \lambda \bI_p)^{-1} \bX^\top \by$. Throughout we also use $\hbSigma = \frac{\bX^\top \bX}{n}$, $\hat{\bSigma}_{\lambda} = \frac{\bX^\top \bX}{n} + \frac{\lambda}{n} \bI_p$ and $\bPi = \bI_p-(\hat{\bSigma}_{\lambda})^{-1} \hbSigma$. Note that under \cref{assump:well_spec}, $\hat{\bbeta}_{\lambda}-\bbeta_0 = -\bPi \bbeta_0 + \hat{\bSigma}^{-1}_{\lambda} \bX^\top \bepsilon/n$, which can be thought of as a standard bias-variance decomposition for the ridge estimator.

We first introduce a standard sub-Gaussian concentration result providing control on the fluctuations of the spectral norm of the design matrix which follows immediately from \citet[Theorem 6.5]{wainwright2017highdim},
\begin{lemma}
Let $\bx_1, \hdots, \bx_n$ be i.i.d. random vectors satisfying \cref{assump:cov,,assump:design} with sample covariance $\hbSigma = \frac{1}{n} \bX^\top \bX$, then there exist universal constants $c_1, c_2, c_3$ such that for $n \geq c_1 \kappa^4 \cond^2 p$,
\begin{align}
     \normt{\hbSigma - \bSigma} \leq \frac{\Cmin}{2} 
\end{align}
with probability at least $1-c_2 e^{-c_3 n/(\kappa^4 \cond^2)}$.
\label{thm:design_ridge_eig}
\end{lemma}

With this result we first provide a conditional (on $\bX$) risk bound for ridge regression. For convenience throughout this section we define the set of design matrices $\mE_n = \{ \bX : \forall \bv \text{ such that } \normt{\bv}=1, \bv^\top \hat{\bSigma} \bv \geq \frac{\Cmin}{2} \}$. 
\begin{lemma}
    Let \cref{assump:cov,,assump:noise1} hold (with variance parameter renamed to $\sigma^2$) and assume a deterministic design matrix $\bX \in \mE_n$ and that $n \geq p$. Then if $\hblambda$ denotes the solution to the ridge regression program, with $\lambda \leq \lambda_* = \arg \min_{\lambda} \left((\frac{\lambda/n}{\Cmin+\lambda/n})^4 \normt{\bbeta_0}^4 + \sigma^4 p^2/n^2 (\frac{\Cmax}{(\Cmin+\lambda/n)^2})^2 \right)$,
\begin{align}
    \left(\E \left[\normt{\hblambda-\bbeta_0}^4\right] \right)^{1/2} \leq O \left( \sigma^2 \frac{\cond}{\Cmin}\frac{p}{n} \right).
\end{align}
\label{lem:ridge_uncond}
\end{lemma}
\begin{proof}
    Recall the standard bias variance decomposition $\hblambda-\bbeta_0 = -\bPi \bbeta_0 + \hat{\bSigma}^{-1}_{\lambda} \bX^\top \bepsilon/n$. So $\normt{\hblambda-\bbeta_0}^4 \leq 64 \left( (\bbeta_0 \bPi^2 \bbeta_0)^2 + (\bepsilon^\top \bX \hat{\bSigma}^{-1}_{\lambda} \cdot \hat{\bSigma}^{-1}_{\lambda} \bX^\top \bepsilon/n^2)^2 \right)$. Using the SVD of $\bX/\sqrt{n} = \bU^\top \bLambda \bV$ we see that $\hbSigma = \bV^\top \bLambda^2 \bV = \bV^\top \bD \bV$. Further, on the event $\mE_n$ we have that $\frac{1}{2} \Cmin \leq d_i \leq \frac{3}{2} \Cmax$ for $i \in [p]$ where $d_i = \bD_{ii}$ by the Weyl inequalities. So on $\mE_n$, $\bbeta_0^\top \bPi^2 \bbeta_0 = \bbeta_0^\top \bV^\top (\diag(\frac{\lambda/n}{d_i+\lambda/n}))^2 \bV \bbeta_0 \leq O((\frac{\lambda/n}{\Cmin+\lambda/n})^2 \normt{\bbeta_0}^2)$. Define $\bS = \bepsilon^\top \bX \hat{\bSigma}^{-1}_{\lambda} \cdot \hat{\bSigma}^{-1}_{\lambda} \bX^\top \bepsilon/n$, we have that $\bS = \bU^\top \diag(\frac{z_i}{(z_i+\lambda/n)^2}) \bU \preceq O(\bU^\top \diag(\frac{\Cmax}{(\Cmin+\lambda/n)^2}) \bU)$ on $\mE_n$, which also has at most rank $p$ since $\bLambda$ has at most $p$ non-zero singular values. Hence applying \cref{lem:trace_square} we find that $\E[(\bepsilon^\top \bS \bepsilon)^2] \leq O(\sigma^4 p^2  (\frac{\Cmax}{(\Cmin+\lambda/n)^2})^2)$. Combining, gives that
    \begin{align}
        \E \left[\normt{\hblambda-\bbeta_0}^4 \right] \leq c_1 \left ((\frac{\lambda/n}{\Cmin+\lambda/n})^4 \normt{\bbeta_0}^4 + \sigma^4 p^2/n^2 (\frac{\Cmax}{(\Cmin+\lambda/n)^2})^2 \right).
    \end{align}
    for some universal constant $c_1$.
    Since by definition $\lambda_*$ minimizes the upper bound in the above expression it is upper bounded by setting $\lambda=0$ in the same expression so,
    \begin{align}
        \E \left[\normt{\hblambda-\bbeta_0}^4 \right] \leq O \left(\sigma^4 p^2/n^2 (\frac{\Cmax}{\Cmin^2})^2 \right).
    \end{align}
    We can further check that the upper bound is decreasing over the interval $[0, \lambda_*]$ and hence the conclusion follows.
    As an aside a short computation shows the optimal choice of $\lambda_*/p = (\cond \Cmax \frac{n}{p}  \frac{\sigma^4}{\normt{\bbeta_0}^4})^{1/3}$.
\end{proof}

We now prove a simple result which provides a crude bound on the error of the ridge regression estimate we deploy when $\bX \notin \mE_n$.
\begin{lemma}
    \label{lem:bad_ridge_bound}
      Let $\hblambda$ be the solution of the ridge regression program $\hblambda= \arg \min_{\bbeta} \normt{\by-\bX \bbeta}^2 + \lambda \normt{\bbeta}^2$, then 
    \begin{align}
         \normt{\hblambda-\bbeta_0}^2 \leq 4 \left( \norm{\bepsilon}_2^2/\lambda +  \norm{\bbeta_0}_2^2 \right).
    \end{align}
\end{lemma}
\begin{proof}
    By definition we have that,
    \begin{align}
        \normt{\by-\bX \hblambda}^2 + \lambda \normt{\hblambda}^2 \leq  \normt{\bepsilon}^2 + \lambda \normt{\bbeta_0}^2 \implies 
        \norm{\hblambda}^2 \leq  \norm{\bepsilon}_2^2/\lambda + \norm{\bbeta_0}^2
    \end{align}
    So we obtain that,
    \begin{align}
        \normt{\hblambda-\bbeta_0}^2 \leq 2(\normt{\hblambda}^2+\normt{\bbeta_0})^2 \leq  4(\normt{\bepsilon}^2/\lambda+\normt{\bbeta_0}^2).
    \end{align}
\end{proof}
Finally, we prove the final result which will provide an unconditional risk bound in expectation for the ridge regression estimator,
\begin{lemma}
\label{lem:ridge_total_ub}
 Let Assumptions \ref{assump:well_spec}, \ref{assump:cov}, \ref{assump:design}, \ref{assump:noise1} hold (with variance parameter renamed to $\sigma^2$). Then there exist universal constants $c_1, c_2, c_3 > 0$ such that if $n \geq c_1 \kappa^4 \cond^2 p$, and $\hblambda$ a solution of the ridge regression program with $c_2 \frac{n^2 \Cmin}{p \cond} e^{-n c_3/\kappa^4 \cond^2} \leq \lambda \leq \lambda_* = \arg \min_{\lambda} \left((\frac{\lambda/n}{\Cmin+\lambda/n})^4 \normt{\bbeta_0}^4 + \sigma^4 p^2 (\frac{\Cmax}{(\Cmin+\lambda/n)^2})^2 \right) =  p (\cond \Cmax \frac{n}{p}  \frac{\sigma^4}{\normt{\bbeta_0}^4})^{1/3}$
    \begin{align}
        & \E_{\bX, \bepsilon} \left[\Vert \hblambda-\bbeta_0 \Vert_2^4\right] \leq \E_{\bX, \bepsilon} \left[\Vert \hblambda-\bbeta_0 \Vert_q^4 \mone[\bX \in \mE_n] \right] + O \left((\frac{n^2 \sigma^4}{\lambda^2} + \normt{\bbeta_0}^4) e^{-\frac{c_3}{\kappa^4 \cond^2}n} \right).
    \end{align}
    Moreover if $\normi{\bbeta_0}=O(1)$ then,
    \begin{align}
        \sqrt{\E_{\bX, \bepsilon} \left[ \normt{\hblambda-\bbeta_0}^4 \right]} \leq O(\frac{\sigma^2 \cond}{\Cmin} \frac{p}{n}).
    \end{align}
    where the $O$ hides universal constants in $\Cmax, \Cmin, \cond, \kappa$ in the final statement.
\end{lemma}
\begin{proof}
    Decomposing as
    \begin{align}
        \E \left[\normt{\hblambda-\bbeta_0}^4 \right] = \E_{\bX, \bepsilon} \left[\Vert \hblambda-\bbeta_0 \Vert_2^4 \mone[\bX \in \mE_n] \right] + \E_{\bX, \bepsilon} \left[\Vert \hblambda-\bbeta_0 \Vert_q^4 \mone[\bX \notin \mE_n] \right]
    \end{align}
    We can bound the second term explicitly using the Cauchy-Schwarz inequality as, 
    \begin{align}
        \E[ \normt{\hblambda-\bbeta_0}^4  \mone[\bX \notin \mE_n]] \leq  \sqrt{\E_{\bX, \bepsilon}\left[\Vert \hblambda-\bbeta_0 \Vert_2^{8}\right]} \sqrt{\Pr_{\bX}[\bX \notin \mE_n]} \leq
        O(\frac{n^2 \sigma^4}{\lambda^2} + \normt{\bbeta_0}^4) e^{-\frac{c_3}{\kappa^4 \cond^2}n}
    \end{align}
    using the crude upper bound from \cref{lem:bad_ridge_bound} to upper bound the first term and \cref{thm:design_ridge_eig} to bound the probability in the second term.
    
    For the second statement note that we can bound the first term using the using the independence of $\bX, \bepsilon$ and \cref{lem:ridge_uncond}, to conclude,
    \begin{align}
        \sqrt{\E_{\bX, \bepsilon} \left[\Vert \hblambda-\bbeta_0 \Vert_2^4 \mone[\bX \in \mE_n] \right]} \leq O(\frac{\sigma^2 \cond}{\Cmin} \frac{p}{n}).
    \end{align}
    With the specific lower bound on $\lambda$ in the theorem statement, when $\normi{\bbeta_0}/\sigma_{\epsilon}=O(1)$ and $n \gtrsim \kappa^4 \cond^2 p$ we have, 
    \begin{align}
        \sqrt{O(\frac{n^2 \sigma^4}{\lambda^2} + \normt{\bbeta_0}^4) e^{-\frac{c_3}{\kappa^4 \cond^2}n}} \leq O(\frac{\sigma^2 \cond}{\Cmin} \frac{p}{n})
    \end{align}
\end{proof}
Finally, we prove a simple matrix expectation upper bound,
\begin{lemma}
    \label{lem:trace_square}
    Let $\bS \in \mR^{n \times n}$ be a (deterministic) p.s.d. matrix with rank at most $p$ satisfying $\normt{\bS} \leq z$, and let $\bepsilon \in \mR^n$ satisfy \cref{assump:noise1}. Then
\begin{align}
    \E \left[(\bepsilon^\top \bS \bepsilon)^2 \right] \leq  O(\sigma^4 z^2 p^2).
\end{align}
\end{lemma}
\begin{proof}
    This follows by a straightforward computation using the sub-Gaussianity of each $\epsilon_i$:
    \begin{align}
        \E\left[(\bepsilon^\top \bS \bepsilon)^2 \right] \leq O(\sum_i S_{ii}^2 \E[\epsilon_i^4] + \sum_{i \neq j} S_{ij}^2 \E[\epsilon_i^2 \epsilon_j^2] + \sum_{i \neq j} S_{ii} S_{jj} \E[\epsilon_i^2 \epsilon_j^2] ) \leq O(\sigma^4 \norm{S}_F^2 + \sigma^4 \Tr[S]^2) \leq O(\sigma^4 p^2 z^2).
    \end{align}
\end{proof}

\section{Experimental Details}\label{sec:expts}
\subsection{Implementation Details}
All algorithms were implemented in Python (with source code to be released to be upon publication). The open-source library scikit-learn was used to fit the Lasso estimator, the cross-validated Lasso estimators, and the random forest regression models used in the synthetic/real data experiments. The convex program for the JM-style estimator was solved using the open-source library CVXPY equipped with the MOSEK solver \citep{cvxpy}. 

Note the debiased estimators presented require either refitting the auxiliary regression for $\bg(\cdot)$ (i.e. the Lasso estimator or a random forest) in the case of the OM estimators, or resolving the convex program in \eq{jm_program} for each new test point $\xstar$. Although this presents a computational overhead in both our synthetic and real-data experiments, such computations are trivially parallelizable across the test points $\xstar$. As such, we used the open-source library Ray to parallelize training of the aforementioned models \citep{moritz2018ray}. All experiments were run on 48-core instances with 256 GB of RAM.

\subsection{Data Preprocessing and Cross-Validation Details}
In all of the experiments (both synthetic and real data) the training covariates (in the design $\bX$) was first centered and scaled to have features with mean zero and unit variance. Subsequently the vector of $\by$ values was also centered by subtracting its mean; that is $\by \to \by - \bar{\by}$. After any given model was fit the mean $\bar{\by}$ was added back to the (y-centered) prediction $\theta$ of the model. On account of this centering, the Lasso estimators were not explicitly fit with an intercept term (we found the performance was unchanged by not performing the demeaning and instead explicitly fitting the intercept for the Lasso baseline). In each case the cross-validated Lasso estimator was fit, the regularization parameter was selected by cross validation over a logarithmically spaced grid containing a 100 values spaced between $10^{-6}$ and $10^{1}$. The cross-validated ridge estimator was fit by using leave-one-out cross-validation to select the regularization parameter over a logarithmically spaced grid containing a 100 values spaced between $10^{-2}$ and $10^{6}$ for the synthetic experiments, while a range of $10^{-6}$ and $10^{1}$ was used for the real data. The $\ell_1$ and $\ell_1/\ell_2$ ratio parameter for the elastic net were also set using cross-validation by letting the $\ell_1$ regularization parameter over a logarithmically spaced grid containing a 100 values spaced between $10^{-6}$ and $10^{1}$, while the $\ell_1/\ell_2$ ratio parameter was allowed to range over $[.1, .5, .7, .9, .95, .99, 1]$. In the case of the real data experiments the random forest regressors (RF) used in the $\bg(\cdot)$ models were fit using a default value of 50 estimators in each RF.

\subsection{JM-style Estimator Details}

Note that $\lambda_{\bw}$ was chosen for the JM-style estimator using the heuristic to search for the smallest $\lambda_{\bw}$ in a set for which the convex program in \eq{jm_program} is feasible. If no such value existed (i.e. all the programs were infeasible) we defaulted to simply predicting using the base Lasso regression in all cases (which is equivalent to using $\bw=0$).

\subsection{OM Estimators Details} \label{sec:expts_oml}

As described in the main text, the OM estimators use 2-fold data-splitting. Such a procedure can be sample-inefficient since only a fraction of the data is used in each stage of the procedure. For the OM methods used in the experiments we instead used a more general $K$-fold cross-fitting as described in \citep{chernozhukov2017double}, with $K=5$ and $K=10$.

The OM methods can be fit exactly as described in the paper with the following modifications. First the original dataset is split into $K$ equally-sized folds we denote as $(\bX_{\mI_1}, \by_{\mI_1}), \hdots, (\bX_{\mI_K}, \by_{\mI_K})$; here the index sets range over the datapoints as $\mI_1 = \{ 1, \hdots, \frac{n}{K} \}, \mI_2 = \{ \frac{n}{K}+1, \hdots, \frac{2n}{K} \}$ etc... We also use $(\bX_{\mI_{-i}}, \by_{\mI_{-i}})$ to describe $K$-leave-one-out subsets of the original folds which contain the union of datapoints in all \textit{but} the $\mI_i$th fold of data.

Then, $K$ sets of first-stage regressions are trained on the $K$-leave-one-out subsets to produce $(\bbf^{-1}, \bg^{-1}), \hdots, (\bbf^{-K}, \bg^{-K})$; explicitly the pair $(\bbf^{-i}, \bg^{-i})$ is fit on $(\bX_{\mI_{-i}}, \by_{\mI_{-i}})$. Finally the empirical moment equations can be solved for $\yom$ by summing over the entire dataset, but evaluating the $(\bbf^{-i}, \bg^{-i})$ model on only the $i$th fold:
\begin{talign}
    \sum_{i \in K} \sum_{j \in \mI_{i}} m(t_j, y_j, \yom, \bz_{j}^\top \bbf^{-i}, \bg^{-i}(\bz_j)) = 0.
\end{talign}
The estimator for the variance $\mu_2$ can also be computed in an analogous fashion, $\sum_{i \in K} \sum_{j \in \mI_{i}} t_j(t_j - \bg^{-i}(\bz_j)$. More details on this procedure can be found in \citet{chernozhukov2017double} and \citet{mackey2017orthogonal}. Note that since $K$ is chosen to be constant, our theoretical guarantees also apply to this estimator up to constant factors.

Also though the thresholding step (with the parameter $\tau$) is used in our theoretical analysis to control against the denominator $\mu_2$ being too small, we found in practice the estimate of $\mu_2$ concentrated quickly and was quite stable. Hence we found explicitly implementing the thresholding step was unnecessary and we did not include this in our implementation.

\subsubsection{OM $q$ moments}
In \cref{sec:om_ub} we focus our analysis on the OM $f$ moments but also introduce the first-order orthogonal $q$ moments, whose practical efficacy we explore in our real data experiments. For completeness we include the details of the algorithm to predict with $q$-moments here. The primary difference with respect to the $f$-moments is with respect to how the $\bq$ or $\bbf$ regression is fit, the $\bg$ regression is handled identically. For simplicity, we present the algorithm in parallel to how the $\bbf$ moments are introduced in the main text (without the $K$-fold cross-fitting), although $K$-fold cross-fitting is used in practice exactly as described above.

After the data reparametrization we have $\bx_i' = [t_i, \bz_i] =  (\bU^{-1})^\top \bx_i$. In the reparametrized basis, the linear model becomes, 
\begin{talign}
    y_i = \theta t_i + \bz_i^\top \bbf_0 + \epsilon_i \quad \quad t_i = \bg_0(\bz_i) + \eta_i
\end{talign}
where $\bq_0(\bz_i) = \theta \bg_0(\bz_i) + \bz_i^\top \bbf_0$. 

\begin{itemize}[leftmargin=.5cm]
    \item The first fold $(\Xo, \yo)$ is used to run two \textit{first-stage} regressions. We estimate
    $\bq_0$ using a linear estimator (such as the Lasso) by directly regressing $\yo$ onto $\zo$ to produce the vector $\btq$. Second we estimate $\bg_0(\cdot)$ by regressing $\tone$ onto $\zo$ to produce a regression model $\btg(\cdot) : \mR^{p-1} \to \mR$.
\item Then, we estimate $\E[\eta_1^2]$ as $\mu_2 = \frac{1}{n/2} \sum_{i=n/2+1}^{n} (t_i-\btg(\bz_i))^2$ where the sum is taken over the second fold of data; crucially $(t_i, \bz)$ are (statistically) independent of $\btg(\cdot)$ in this expression. 

\item If $\mu_2 \leq \tau$ for a threshold $T$ we simply output $\yom = \xstar^\top \hat{\bbeta}$. If $\mu_2 \geq \tau$ we estimate $\theta$ by solving the empirical moment equation:
\begin{talign}
    \sum_{i=n/2+1}^{n} m(t_i, y_i, \yom, \bz_i^\top \btq, \btg(\bz_i)) =  0 \implies \yom = \frac{\frac{1}{n/2} \sum_{i=n/2+1}^{n} (y_i - \bz^\top_i \btq)(t_i-\btg(\bz_i))}{\mu_2} 
\end{talign}
where the sum is taken over the second fold of data and $m$ is defined in \cref{eq:first_f_moment}.
\end{itemize}

\subsubsection{Synthetic Data Experiment Details}

The experiments on synthetic data were conducted as described in the main text in \mysec{experiments}. In each case for the JM-style estimator the base regression was fit using the cross-validated Lasso, while the auxiliary parameter for the regression was chosen to be the smaller of $\sqrt{\log p/n}$ and $0.01 \sqrt{\log p/n}$ for which the convex program in \eq{jm_program} was feasible. The OM $f$ moments were fit as described above using $5$-fold cross-fitting with the Lasso estimator (with either theoretically-calibrated values for the hyperparameters or hyperparameters chosen by cross-validation) used for both the first-stage regressions.

In \cref{sec:excess_bias} all hyperparameters wer set to their theoretically-motivated values: $\lambda_{\bbeta} = \lambda_{\bg} = 4 \sqrt{\log p/n}$ for the Lasso regressions, and, inspired by the feasibility heuristic of \citep{javanmard2014confidence}, we set $\lambda_{\bw}$ to the smallest value between $\sqrt{\log p/n}$ and $.01 \sqrt{\log p/n}$ for which the JM-style program \cref{eq:jm_program} was feasible. The RMSRE in each experiment was computed over 500 test datapoints (i.e., 500 independent $\xstar$'s) generated from the training distribution; each experiment was repeated 20 times, and the average RMSRE is reported.

\subsubsection{Real Data Experiment Details}

For the base regression procedures five-fold CV was used to select hyperparameters for the Lasso and elastic net estimators, while leave-one-out CV was used for ridge regression.

\textbf{OM methods}
The OM $f$ and $q$ moments were implemented as above with $10$-fold cross-fitting. However to exploit the generality of the OM framework in addition to allowing $\btg(\cdot)$ to be estimated via the cross-validated Lasso estimator, we also allowed $\btg( \cdot )$ to be estimated via random forest regression, and a $\bg=0$ baseline. However, note that $\bbtf$ and $\btq$ were \textit{always} fit with the cross-validated Lasso (a linear estimator) since our primary purpose is to investigate the impacts of debiasing linear prediction with the $\bbtf$ and $\btq$ moments.

For each $\xstar$ we fit a cross-validated Lasso estimator, a random forest regressor, and a $\btg=0$ baseline on each of the $K$-leave-one-subsets of data. We adaptively chose between these models in a \textit{data-dependent} fashion by selecting the method that produced the minimal (estimated) variance for $\yom$. We used a plug-in estimate of the asymptotic variance which can be computed as,
\begin{talign}
    \text{q-var(method)} = \frac{\sum_{i \in K} \sum_{j \in \mI_{i}} (t_j-\bg_{\text{method}}^{-i}(\bz_j))^2 }{V}
\end{talign}
and
\begin{talign}
    \text{f-var(method)} = \frac{\sum_{i \in K} \sum_{j \in \mI_{i}} t_i (t_j-\bg_{\text{method}}^{-i}(\bz_j)) }{V}
\end{talign}
where $V_{\text{method}} = \sum_{i \in K} \sum_{j \in \mI_{i}}  (t_j-\bg_{\text{method}}^{-i}(\bz_j))^2 - (\sum_{i \in K} \sum_{j \in \mI_{i}}  (t_j-\bg_{\text{method}}^{-i}(\bz_j))^2$ for each method. These asymptotic variance expressions can be computed from a general formula for the asymptotic variance from \citet[Theorem 1]{mackey2017orthogonal}. Upon selecting the appropriate $\btg(\cdot)$ method for either the $f$ or $q$ moments the algorithm proceeds as previously described with the given choice of $\btg(\cdot)$.

\textbf{JM-style method}
For the real data experiments the $\lambda_{\bw}$ for the JM-style estimator was selected by constructing a logarithmically-spaced grid of $100$ values of $\lambda_{\bw}$ between $10^{-7}$ and $10^2$ and selecting the smallest value of $\lambda_{\bw}$ for which the convex program in \eq{jm_program} was feasible. 

\textbf{Datasets}
All regression datasets, 
in this paper were downloaded from the publicly available UCI dataset repository \citep{Dua:2019}. 
The triazines dataset was randomly split in an 80/20 train-test split and selected since $n_{train} \approx p$ for it. 
The other 4 datasets were selected due to the fact they can be naturally induced to have distributional shift. The Parkinsons and Wine datasets were selected exactly as in \citet{chen2016robust}.  The Parkinsons dataset, where the task is to predict a jitter index, was split into train and test as in \citet{chen2016robust}, by splitting on the "age" feature of patients: $\leq 60 \to$ train and $ > 60 \to$ test. The task for prediction in the Wine dataset, as in \citet{chen2016robust}, is to predict the acidity levels of wine but given training data comprised only of red wines with a test set comprised only of white wines. In the fertility dataset, where the task is to predict the fertility of a sample, we split into train and test by splitting upon the binary feature of whether patients were in the $18-36$ age group ($\to$ train) or not ($\to$ test). Finally, for the Forest Fires
dataset, where the task it to predict the burned area of forest fires that occurred in Portugal during a roughly year-long period, we split into train/test based on the "month" feature of the fire: those occurring before the month of September ($\to$ train) and those after the month of September ($\to$ test).

Note in all the cases the feature that was split upon was not used as a covariate in the prediction task. In \cref{table:2} we include further information these datasets,
\begin{table}[ht!]
\centering
\caption{Information on Real Datasets. 
}
\label{table:2}
\begin{tabular}{l c c c c c c c c} 
\toprule
 Dataset & $n_{train}$ & $n_{test}$ & $p$ & Distrib. Shift? & \\ [0.5ex] 
\midrule
Fertility & 69 & 31 & 8 & Yes &  \\ 
 Forest Fires & 320 & 197 & 10 & Yes & \\
 Parkinson & 1877 & 3998 & 17 & Yes &  \\
 Wine & 4898 & 1599 & 11 & Yes & \\
 Triazines & 139 & 47 & 60 & No &  \\ [1ex]
\bottomrule
\end{tabular}
\end{table}

\end{document}